\documentclass[twoside,11pt]{article}
\usepackage[abbrvbib,preprint]{jmlr2e}
\usepackage{amsmath}
\usepackage{amsfonts}
\usepackage{amssymb}
\usepackage{xcolor}
\usepackage{tikz}
\usetikzlibrary{calc,shadings,patterns,tikzmark,positioning}
\usepackage{graphicx}
\usepackage{diagbox}
\usepackage{subfig}
\usepackage{epsfig}%
\usepackage{color}
\usepackage{multirow}
\usepackage{setspace}
\usepackage{enumitem}
\usepackage{hhline}
\usepackage{algorithm}
\usepackage{algpseudocode}
\usepackage{pbox}
\usepackage{booktabs}
\usepackage{lscape}
\usepackage{array}
\usepackage[flushleft]{threeparttable}
\usepackage{float}
\usepackage{supertabular}
\usepackage[title]{appendix}
\makeatletter
\newenvironment{breakablealgorithm}
{% \begin{breakablealgorithm}
		\begin{center}
			\refstepcounter{algorithm}% New algorithm
			\hrule height.8pt depth0pt \kern2pt% \@fs@pre for \@fs@ruled
			\renewcommand{\caption}[2][\relax]{% Make a new \caption
				{\raggedright\textbf{\ALG@name~\thealgorithm} ##2\par}%
				\ifx\relax##1\relax % #1 is \relax
				\addcontentsline{loa}{algorithm}{\protect\numberline{\thealgorithm}##2}%
				\else % #1 is not \relax
				\addcontentsline{loa}{algorithm}{\protect\numberline{\thealgorithm}##1}%
				\fi
				\kern2pt\hrule\kern2pt
			}
		}{% \end{breakablealgorithm}
		\kern2pt\hrule\relax% \@fs@post for \@fs@ruled
	\end{center}
}
\makeatother

\newcommand\HatchedCell[4][0pt]{%
  \begin{tikzpicture}[overlay,remember picture]%
    \fill[#4] ( $ (pic cs:#2) + (0,1.9ex) $ ) rectangle ( $ (pic cs:#3) + (0pt,-#1*\baselineskip-.8ex) $ );
  \end{tikzpicture}%
}%

\usepackage{lastpage}
\jmlrheading{24}{2023}{1-\pageref{LastPage}}{3/23; Revised 5/22}{9/22}{21-0000}{Forough Fazeli-Asl, Michael Minyi Zhang, and Lizhen Lin}

\ShortHeadings{A Semi-BNP Hypothesis Test Using MMD with Applications in GANs}{Fazeli-Asl, Zhang, and Lin}
\firstpageno{1}

\begin{document}
\title{A Semi-Bayesian Nonparametric Estimator of the Maximum Mean Discrepancy Measure: Applications in Goodness-of-Fit Testing and Generative Adversarial Networks}

\author{\name Forough Fazeli-Asl \email foroughf@hku.hk \\
	\addr Department of Statistics and Actuarial Science\\
	University of Hong Kong\\
	Pok Fu Lam, Hong Kong
	\AND
	\name Michael Minyi Zhang \email mzhang18@hku.hk \\
	\addr Department of Statistics and Actuarial Science\\
	University of Hong Kong\\
	Pok Fu Lam, Hong Kong
    \AND
	\name Lizhen Lin \email lizhen01@umd.edu \\
	\addr Department of Mathematics\\
	University of Maryland\\
	College Park, MD, USA
	}

\editor{}

\maketitle

\begin{abstract}%   <- trailing '%' for backward compatibility of .sty file
A classic inferential statistical problem is the goodness-of-fit (GOF) test. Such a test can be challenging when the hypothesized parametric model has an intractable likelihood and its distributional form is not available. Bayesian methods for GOF can be appealing due to their ability to incorporate expert knowledge through prior distributions. 
However, standard Bayesian methods for this test often require strong distributional assumptions on the data and their relevant parameters. To address this issue, we propose a semi-Bayesian nonparametric (semi-BNP) procedure in the context of the maximum mean discrepancy (MMD) measure that can be applied to the GOF test. Our method introduces a novel Bayesian estimator for the MMD, enabling the development of a measure-based hypothesis test for intractable models. Through extensive experiments, we demonstrate that our proposed test outperforms frequentist MMD-based methods by achieving a lower false rejection and acceptance rate of the null hypothesis. Furthermore, we showcase the versatility of our approach by embedding the proposed estimator within a generative adversarial network (GAN) framework. It facilitates a robust BNP learning approach as another significant application of our method. 
With our BNP procedure, this new GAN approach can enhance sample diversity and improve inferential accuracy compared to traditional techniques.
\end{abstract}


\section{Introduction}
GOF tests are commonly used to evaluate an empirical data set against a hypothesized parametric model. 
%the accuracy of a hypothesized parametric model in representing a dataset. 
However, there are cases when the likelihood of the parametric model is intractable and the explicit form of the model distribution is unavailable, making it challenging to directly assess the model's fit. One such example is the case of generative models, where independent samples can be generated, but the required likelihood function for traditional GOF tests is intractable. In such situations, a potential solution is to use the MMD measure as an alternative approach for conducting GOF tests \citep{Gretton,key2021composite}. The MMD is a metric on the space of probability distributions and is commonly used in hypothesis testing to quantify the difference between the distribution of the data and the hypothesized model. It can be conveniently estimated using available samples generated from desired distributions. The MMD estimator has proven to be effective in various applications, including analyzing large-scale datasets with high-dimensional features and implementing generative models, especially GANs.

Bayesian nonparametric methods, while powerful, have received comparatively little attention, especially regarding their application in estimating the MMD. One of the primary benefits of the Bayesian approach is that expert knowledge can be incorporated into the prior distributions in a diagnostic setting. Moreover, a BNP learning procedure can provide a certain level of regularization to the training process. This is partially a result of placing uncertainty on the sampling distribution of the data, via a Dirichlet process (DP). Therefore, the lack of such methods in MMD estimation proves to be a hindrance for the statistician who wishes to be Bayesian without overly strong assumptions. This paper seeks to fill this crucial gap.
%in statistical inference. 
%\textcolor{red}{\textbf{R1 (2):} Specifically, we propose a modified version of the two-sample MMD kernel-based test between data distribution $F$ and $F_{G_{\boldsymbol{\omega}}}$ in the BNP framework by placing the DP prior only on $F$, not on $F_{G_{\boldsymbol{\omega}}}$, leading us to refer to this test as semi-BNP. Furthermore, we intend to extend this test to be applicable in machine learning tasks and for training GANs.}

In this paper, we propose a BNP estimator that accurately estimates the MMD kernel-based measure between an intractable parametric model and an unknown distribution. To develop the procedure, we place the DP prior solely on the unknown distribution. % , following basic concepts of the BNP inference, 
Therefore, we refer to this procedure as a semi-BNP estimator. Having established our MMD estimator, we demonstrate that we can generalize the bootstrap procedure given in \cite{dellaporta2022robust} beyond posterior parameter inference. First, we apply our estimator in a variety of two-sample hypothesis testing problems. Next, we introduce a robust Bayesian nonparametric learning (BNPL) approach for training GANs based on simulating from the posterior distribution on the parameter space of the generator. Our approach utilizes the aforementioned estimator as a robust discriminator between the generator's distribution and a DP posterior on the empirical data distribution. 
%, inducing a posterior on the parameter space.} 
Specifically, our framework unifies concepts of the MMD measurement and the BNP inference to leverage their respective benefits into a single discriminator. Furthermore, we will %embed our proposed discriminator within a GAN framework and  
investigate the ability of our discriminator to reduce mode collapse and increase the ability of the generator to fool the discriminator more effectively than the frequentist counterpart for GAN training. 

%The construction of the 
The paper is organized as follows: In Section~\ref{sec:previous}, we review previous works and methods related to our proposed technique. We then introduce our novel semi-BNP estimator for the MMD measure between an unknown and intractable parametric distribution in Section~\ref{sec-BNP-estimator}, and provide theoretical properties of our proposed estimator. In Section~\ref{sec-BNP-test}, we utilize our semi-BNP estimator of the MMD measure to create a powerful GOF test based on the relative belief (RB) ratio, which serves as the Bayesian evidence to judge the null hypothesis. Moreover, Section~\ref{sec-gan-train} outlines the incorporation of the semi-BNP estimator as the discriminator in the GAN architecture. This results in a robust BNPL procedure that accurately estimates the generator's parameters for generating realistic samples. The section also discusses the theoretical properties of the proposed discriminator, such as robustness and consistency. We evaluate the novel semi-BNP procedures for hypothesis testing and GAN training through numerical experiments in Section~\ref{sec-experiment}. Lastly, we conclude the paper in Section~\ref{sec-conclusion} and discuss potential future directions. All proofs, algorithms, notations, and additional experiments are given in the supplementary material.

\section{Previous Work}\label{sec:previous}
In this section, we introduce the fundamental components of our BNP estimator of the MMD.
\subsection{Maximum Mean Discrepancy Measure}
Consider the random variables, $\mathbf{X}$ and $\mathbf{Y}$, drawn from Borel distributions $F_1$ and $F_2$ on a topological space $\mathfrak{X}$, respectively. Let $\mathcal{H}_{k}$ be a reproducing kernel Hilbert space (RKHS) indexed with a kernel function $k(\cdot,\cdot)$ that maps pairs of inputs from $\mathfrak{X}$ to real numbers. The function $k(\cdot,\cdot)$ is positive definite, such that for any function $h\in \mathcal{H}_{k}$ and any  $\mathbf{X}\in \mathfrak{X}$, $h(\mathbf{X})=\langle h, k(\mathbf{X},\cdot)\rangle_{\mathcal{H}_k}$, where $\langle \cdot,\cdot\rangle_{\mathcal{H}_k}$ represents the inner product in $\mathcal{H}_{k}$. Consider function $\mu_{F_1}(\cdot) = E_{F_1}\left[k(\mathbf{X},\cdot) \right] \in \mathcal{H}_k$, which is defined as the kernel mean embedding of the distribution $F_1$ in \cite{Gretton}.
Then, 
for given $ \mathbf{X},\mathbf{X}^{\prime}\overset{i.i.d.}{\sim} F_1, \mathbf{Y},\mathbf{Y}^{\prime}\overset{i.i.d.}{\sim} F_{2}$, 
%$ \mathbf{X},\mathbf{X}^{\prime}\sim F, \mathbf{Y},\mathbf{Y}^{\prime}\sim F_G$, let $ \phi $ be a feature map that maps $  \mathbf{X},\mathbf{X}^{\prime}, \mathbf{Y}$ and $ \mathbf{Y}^{\prime}  $ to a space $ \mathcal{F} $ such that $ \phi(\mathbf{X}), \phi(\mathbf{X}^{\prime}), \phi(\mathbf{Y}), \phi(\mathbf{Y}^{\prime})\in\mathcal{F}$. 
the MMD is given by 	\begin{equation}\label{mmd-def1}
\mathrm{MMD}^2(F_1,F_{2})=||\mu_{F_1}-\mu_{F_2}||^{2}_{\mathcal{H}_k}=E_{F_1}[k(\mathbf{X},\mathbf{X}^{\prime})] -2E_{F_1,F_{2}}[k(\mathbf{X},\mathbf{Y})] +E_{F_{2}}[k(\mathbf{Y},\mathbf{Y}^{\prime})],
\end{equation}
%via applying the kernel trick to compute the inner product between $ \mathbf{X} $ and $ \mathbf{Y} $ in an infinite feature space $ \mathcal{F} $ as $ k(\mathbf{X},\mathbf{Y})= \langle \mathbf{X}, \mathbf{Y} \rangle_{\mathcal{F}} $.
where $||\cdot||_{\mathcal{H}_k}$ is the norm function in the RKHS. The MMD is $ 0 $ if and only if $ F=F_{G_{\boldsymbol{\omega}}} $, when $\mathcal{H}_k$ is a \emph{universal} RKHS \citep[Theorem 5]{Gretton}. % The Gaussian kernel, defined as $k_{G_{\sigma}}(\mathbf{X},\mathbf{Y})=\exp(\frac{-||\mathbf{X}-\mathbf{Y}||^{2}}{2\sigma^2}) $ with bandwidth parameter $ \sigma $, is a common universal kernel used in MMD.
In practice, distributions $ F_{1} $ and $ F_{2} $ are not accessible, and then the biased, empirical estimator of \eqref{mmd-def1} is calculated using empirical distributions $F_{1,n}$ and $F_{2,m}$ as
\begin{equation*}%\label{MMD-def}
\mathrm{MMD}^2(F_{1,n},F_{2,m})=\frac{1}{n^{2}}\sum_{i, j=1}^{n}k(\mathbf{X}_{i},\mathbf{X}_{j})-\frac{2}{mn}\sum_{i=1}^{n}\sum_{j=1}^{m}k(\mathbf{X}_{i},\mathbf{Y}_{j})+\frac{1}{m^{2}}\sum_{i,j=1}^{m}k(\mathbf{Y}_{i},\mathbf{Y}_{j}),    
\end{equation*}
where $ \mathbf{X}_1,\ldots,\mathbf{X}_n $ is a sample from $F_1$ and $ \mathbf{Y}_1,\ldots,\mathbf{Y}_m $ is a sample generated from $ F_{2} $. 
%The MMD is a kernel-based measure that can also be applied to data that are not in Euclidean space. This property makes it useful for comparing complicated datasets like images. 

Recently, \cite{key2021composite} proposed a GOF test using the MMD measure when the hypothesized model belongs to a parametric family of intractable models.  
It was proposed to be employed in training generative models such as toggle-switch models and GANs. There are also numerous generative models closely linked to the implementation of MMD in GANs, which can be found in \cite{briol2019statistical}, \cite{niu2023discrepancy}, \cite{oates2022minimum}, and \cite{bharti2023optimally}. These models offer distinct MMD estimators that are specifically designed to further improve the MMD's capability in estimating the generator's parameters.

\subsection{Bayesian Methods: Approximate Bayesian Computation, the Dirichlet Process and Bayesian NonParametric Learning}
Previous work in simulation-based inference has largely focused on applying discrepancy measures from a frequentist nonparametric (FNP) perspective.
A Bayesian perspective on simulation-based inference involves a similar methodology, using approximate Bayesian computation (ABC) to estimate the model parameters via simulation \citep{beaumont2002approximate}. ABC starts by sampling from a prior distribution placed on the parameter space of the generative model. Rather than estimating parameters directly from the posterior distribution, this approach involves comparing summary statistics of simulated data with those of observed data using discrepancy measures. The simulated parameter values corresponding to the accepted summary statistics are retained if the distance falls within a predetermined threshold. 

Identifying informative summary statistics in ABC can be a challenging task, and an inappropriate choice may result in poor posterior inference from the data \citep{robert2011lack, aeschbacher2012novel}. One solution proposed by \cite{park2016k2} is to use the MMD metric between simulated and real data distributions to avoid manually selecting the summary statistics. However, as the threshold approaches zero, ABC tends to approximate the standard Bayesian posterior, which is susceptible to model misspecification and lacks robustness \citep{dellaporta2022robust}. To address these two issues, generalized Bayesian inference (GBI) proposes an alternative method by replacing the likelihood in the posterior distribution with the exponential of a robust loss function. Within the GBI framework, there are two prominent procedures that use the MMD loss. \cite{cherief2020mmd} propose a pseudo-likelihood based on the MMD metric and approximate the posterior using variational inference. \cite{pacchiardi2021generalized} extend this method to a more general Bayesian likelihood-free model using stochastic gradient Monte Carlo Markov Chain (MCMC) to perform posterior inference\footnote{A comprehensive list of other GBI procedures for addressing this issue can be found in \cite{dellaporta2022robust}.}. 

However, \cite{dellaporta2022robust} noted that the performance of GBI is very sensitive to the choice of a learning rate and that there is no general heuristic for selecting this hyperparameter. Additionally, these calculations often require MCMC sampling methods, which can impose a significant computational burden. To address these issues, \cite{dellaporta2022robust} developed an MMD posterior bootstrap procedure following the BNPL strategy developed in \cite{lyddon2018nonparametric,lyddon2019general,fong2019scalable}. 
In this BNPL strategy, a BNP prior is defined on $F$, leading to a BNP posterior on $F$, denoted by $F^{pos}$. The key idea is that any posterior on the generator's parameter space  $\mathcal{W}$ can be derived by mapping $F^{pos}$ through the push-forward measure 
\begin{equation*}
\boldsymbol{\omega}^{\ast}(F^{pos}):=\arg\min\limits_{\boldsymbol{\omega}\in\mathcal{W}} \delta(F^{pos},F_{G_{\boldsymbol{\omega}}}),   
\end{equation*}  
which is visually depicted in \citet[Figure 1]{dellaporta2022robust}.
In particular, \cite{dellaporta2022robust} considered $F^{pos}$ as the DP posterior and $\delta$ as the MMD measure.  %The main framework of this procedure is visually depicted in \citet[Figure 1]{dellaporta2022robust}. For a more comprehensive grasp of this concept, interested readers are encouraged to refer to this figure.   
% The DP is the most commonly used prior in BNP inferences and can be viewed as an infinite-dimensional generalization of the Dirichlet distribution. It is constructed around a fixed probability measure, whose variation is controlled by a positive real number. Since exact simulation from the DP is not feasible, they resorted to a truncated version of the stick-breaking representation of the DP to approximate the MMD measure between the DP posterior and the generator's distribution. This approach enabled them to simulate parameters from the posterior distribution on the generator's parameter space. \cite{dellaporta2022robust} highlighted that, unlike ABC, which relies on rejecting specific samples,  their proposed approach does not require sample rejection. This characteristic is a significant advantage as it circumvents computational challenges associated with ABC.}

The DP, introduced by \cite{Ferguson}, is a commonly used prior %stochastic process that is used as a prior 
in Bayesian nonparametric methods. It  
%which can be viewed as a generalization of the Dirichlet distribution. %that is used to describe the prior knowledge about the distribution of random variables 
%\citep{Ferguson}. 
can be viewed as an infinite-dimensional generalization of the Dirichlet distribution constructed around $ H $ (the base measure), a fixed probability measure, whose variation is controlled by $ a $ (the concentration parameter), a positive real number.
To formally define the DP, consider a space $\mathfrak{X}$  with a $\sigma$-algebra $\mathcal{A}$ of subsets of
$\mathfrak{X}$. For a base measure $G$ on  $(\mathfrak{X},\mathcal{A})$ and $a>0$, a random probability measure $F=\left\{  F(A):A\in\mathcal{A}\right\} $ is called a DP on
$(\mathfrak{X},\mathcal{A})$, denoted by $F^{pri}:=\left(F\sim
{DP}(a,H)\right),$ if for every measurable partition $A_{1},\ldots,A_{k}$ of
$\mathfrak{X} $ with $k\geq2\mathfrak{,}$ the joint distribution of the vector
$\left(  F(A_{1}),\ldots,F(A_{k})\right)$ has the Dirichlet distribution with parameters
$\left(aH(A_{1}),\ldots,aH(A_{k})\right)$. It is assumed that
$H(A_{j})=0$ implies $F(A_{j})=0$ with probability one. %Consequently, for any
%$A\in\mathcal{A}$, $P(A)\sim$ beta$(aG(A),a(1-G(A)))$,
%${E}(P(A))=G(A)\ $and ${Var}(P(A))=G(A)(1-G(A))/(1+a).$

One of the most important properties of the DP is the conjugacy property--when the sample $x=(x_{1},\ldots,x_{n})$ is drawn from $F\sim DP(a, H)$, the posterior distribution of $F$ given $x$, denoted by $F^{pos}$, is also a DP with concentration parameter $a+n$ and base measure
\begin{equation*}%\label{pos base measure}
H^{\ast}=a(a+n)^{-1}H+n(a+n)^{-1}F_{n},
\end{equation*}
where $F_{n}$ denotes the empirical cumulative distribution function (ECDF) of the sample
$x$. Note that, $H^{\ast}$ is a convex combination
of the base measure $H$ and  $F_{n}$. Therefore, $H^{\ast}\rightarrow H$ as
$a\rightarrow\infty$ while $H^{\ast}\rightarrow F_{n}$ as $a\rightarrow0$. On the other hand, as $n\rightarrow\infty$, $H^{\ast}$ converges to the true cumulative distribution function (CDF) that generates the data, according to the Glivenko-Cantelli theorem. A guideline for choosing the hyperparameters $a$ and $H$ for the test of equality distributions will be covered in Section \ref{sec-BNP-test}.

In previous work, there are several BNP GOF tests \citep{al2018prior, al2021bayesian, al2021necessary}, as well as two-sample tests \citep{labadi2014two,al2021two} and a multi-sample test \citep{Al-Labadi}, that are closely connected to the posterior-based distance estimation employed in the BNPL procedure of \cite{dellaporta2022robust}. These methods are developed using different discrepancy measures to compare the distance between DP posteriors, placed on unknown distributions, with the corresponding one between DP priors. However, unlike our proposed method, none of them employ the MMD measure.

%\lizhen{(Forough: do you think we can remove this text since this is quite a nonstandard and non popular representation, which is rarely mentioned other where? \forough{Thanks, I removed the highlighted paragraph.} )}
%Beside the previous definition of the Dirichlet process, \cite{Ferguson} proposed an alternative definition via a series representation: If $F\sim{DP}(a,H)$, then}
%\begin{equation}
%F=\sum_{i=1}^{\infty}L^{-1}(\Gamma_{i}){\delta_{Y_{i}}/}\sum_{i=1}^{\infty
%}{{L^{-1}(\Gamma_{i})}}, \label{series-dp}%
%\end{equation}
%\lizhen{where $\Gamma_{i}=E_{1}+\cdots+E_{i}$ with $E_{i}\overset{i.i.d.}{\sim}\mbox{Exponential}(1)$, $Y_{i}\overset{i.i.d.}{\sim}H$ independent of the $\Gamma_{i}$, $L^{-1}(y)=\inf\{x>0:L(x)\geq y\}$ with $L(x)=a\int_{x}^{\infty }t^{-1}e^{-t}dt,x>0,$ and ${\delta_{Y_{i}}}$ is the Dirac delta measure at $Y_{i}$. It follows from (\ref{series-dp})  that the Dirichlet process is a discrete probability measure even for the cases with an absolutely continuous base measure $H$. Despite this fact, the support of the Dirichlet process could be quite large by imposing a weak topology, namely, the set of all probability measures whose support is contained in the support of the base measure. Recall the complexity of working with \eqref{series-dp}, as no closed forms are available for the inverse of the L\'{e}vy measure $L(x)$. To deal with this difficulty, \cite{wolpert1998simulation} present an approximation for $L(x)$. \cite{bondesson1982simulation}} 

%\cite{bondesson1982simulation} and 
\cite{sethuraman1994constructive} proposed an infinite series representation as an alternative definition for DP. The construction of \cite{sethuraman1994constructive} is known as the stick-breaking representation and is a popularly used method in DP inference. Particularly, for a sequence of identically distributed (i.i.d.) random variables $\left\lbrace \beta_i\right\rbrace_{i\geq1}$ from $\mbox{Beta}(1,a)$, let $w_1=\beta_1$, and $w_i=\beta_i\prod_{j=1}^{i-1}(1-\beta_j)$, for $i\geq2$. Then, the stick-breaking representation is given by $ F_{SB}=\sum_{i=1}^{\infty}w_i\delta_{Y_{i}},$
%\begin{align*}%\label{DP-stick-breaking}
%    F_{SB}=\sum_{i=1}^{\infty}w_i\delta_{Y_{i}},
%\end{align*}
where $\left\lbrace Y_i\right\rbrace_{i\geq1}$ is a sequence of i.i.d. random variables from $H$.
%However, \cite{zarepour2012rapid} mentioned that, unlike the series of \cite{bondesson1982simulation}, the construction of the stick-breaking representation does not include any normalization terms to convert it to a probability measure. \cite{zarepour2012rapid} also remarked that simulating from these infinite series is applicable only if a truncating approach is used for the terms inside the series. They suggested a truncation for the representation of \cite{bondesson1982simulation} and referred to the upper bound proposed by \cite{muliere1998approximating} to approximate samples from stick-breaking representation.
However, \cite{zarepour2012rapid}
addressed some difficulties in using these representations. Meanwhile, \cite{Ishwaran}   proposed a finite representation to facilitate the simulation of the DP. Let
\begin{equation*}%\label{approx of DP}
F^{pri}_{N}=\sum_{i=1}^{N}J_{i,N}\delta_{Y_{i}},
\end{equation*}
where $(J_{1,N},\ldots, J_{N,N})\sim \mbox{Dirichlet}(a/N,\ldots,a/N)$, and $Y_{i}\overset{i.i.d.}{\sim}H$. \cite{Ishwaran} showed that $\lbrace F_{N}\rbrace_{N=1}^{\infty}$ converges in distribution to $F$, where $F_{N}$ and $F$ are random values in the space $M_{1}(\mathbb{R})$ of probability measures on $\mathbb{R}$ endowed with the topology of weak convergence. Thus, to generate $\lbrace J_{i,N}\rbrace_{i=1}^{N}$ put $J_{i,N}=H_{i,N}/\sum_{i=1}^{N}H_{i,N}$, where $\lbrace H_{i,N}\rbrace_{i=1}^{N}$ is a sequence of i.i.d. $\mbox{Gamma}(a/N, 1)$ random variables independent of $\lbrace Y_{i}\rbrace_{i=1}^{N}$. This form of approximation leads to some results in subsequent sections.

To determine the number of DP approximation terms, we apply a random stopping rule, inspired by the method described in \cite{zarepour2012rapid}. This rule, given a specific $\epsilon\in(0,1)$, is defined as:
\begin{align}\label{random-stopping}
N=\inf\left\lbrace j:\, \frac{H_{j,j}}{\sum_{i=1}^{j}H_{i,j}}<\epsilon\right\rbrace.
\end{align}

\section{A Semi-BNP MMD Estimator}\label{sec-BNP-estimator}
This section introduces our semi-BNP estimator for approximating the MMD measure. We consider a scenario where $F_1$ represents a completely unknown distribution, while $F_2$ represents an intractable parametric distribution with a complex generating process. %(and possesses an intractable parametric distribution). %to sample from it, like the distribution of a neural network. 
For a given sample $ \mathbf{Y}_1,\ldots,\mathbf{Y}_m $ from $F_2$  and by assuming $F^{pri}_{1}:=\left(F_1\sim DP(a,H)\right)$ for a non-negative value $a$ and a fixed probability measure $H$, we propose the prior-based MMD estimator as  
\begin{align}\label{BNP-pri-MMD}
\mathrm{MMD}^{2}_{\mathrm{BNP}}(F_{1,N}^{pri},F_{2,m})&= \sum_{\ell,t=1}^{N} J_{\ell,N}J_{t,N}k(\mathbf{V}_{\ell},\mathbf{V}_{t})
-\dfrac{2}{m}\sum_{\ell=1}^{N}\sum_{t=1}^{m} J_{\ell,N}k(\mathbf{V}_{\ell},\mathbf{Y}_{t})+\dfrac{1}{m^2}\sum_{\ell,t=1}^{m} k(\mathbf{Y}_{\ell},\mathbf{Y}_{t}),
\end{align}
where $(J_{1,N},$   
$\ldots,J_{N,N})$ is sampled from $\mbox{Dirichlet}(a/N,\ldots,a/N)$, $\mathbf{V}_{1},\ldots,\mathbf{V}_{N}\overset{i.i.d.}{\sim}H$, and $N$ is the number of terms in the DP approximation $\sum_{\ell=1}^{N}J_{\ell,N}\delta_{\mathbf{V}_{\ell}}$ proposed by \cite{Ishwaran}. %We are proposing this method for the purpose of considering a discriminator of GANs as a two-sample test in training the generative neural network. As we will explain in more detail in the next section, the proposed discriminator will be used to compare the distribution of the real data with the deterministic generator's synthetic data to update the generator's parameter.  
Since we only impose the DP prior on the distribution of the real data, we refer to the approach as a semi-BNP procedure.
\begin{theorem}\label{lem-pri-MMD}
For a non-negative real value $a$ and fixed probability distribution $H$, let $F^{pri}_{1}:=(F_1\sim DP(a,H))$ and $k(\cdot,\cdot)$ be any continuous kernel function with feature space corresponding to a universal RKHS. Assume that $|k(\boldsymbol{z},\boldsymbol{z}^{\prime})|<K$, for any $\boldsymbol{z},\boldsymbol{z}^{\prime}\in \mathbb{R}^d$. Then, \\
\noindent $i.$ $\mathrm{MMD}^{2}_{\mathrm{BNP}}(F_{1,N}^{pri},F_{2,m})\xrightarrow{a.s.}\mathrm{MMD}^{2}(H_N,F_{2,m})$, as $a\rightarrow\infty$,\\
%\noindent $ii.$ 
%$E_{F^{pri}_{1}}(\mathrm{MMD}^{2}_{\mathrm{BNP}}%(F_{1,N}^{pri},F_{2,m}))\rightarrow \mathrm{MMD}^{2}%(H_N,F_{2,m})$, as $a\rightarrow\infty$,\\
%\noindent $ii.$ $E(\mathrm{MMD}^{2}_{\mathrm{BNP}}%(F_{1}^{pri},F_{2}))\rightarrow \mathrm{MMD}^{2}_{m}%(H,F_{2})$ as $a\rightarrow\infty$ and %$N\rightarrow\infty$,\\
%\textcolor{gray}{\noindent $ii.$ $E_{F^{pri}_{1}}(\mathrm{MMD}^{2}_{\mathrm{BNP}}(F_{1,N}^{pri},F_{2,m}))$ converges in probability at rate $O((N+m)^{-1/2})$ to $\mathrm{MMD}^{2}(H,F_{2})$ as $a\rightarrow\infty$, if $k(\boldsymbol{z},\boldsymbol{z}^{\prime})<K$, for any $\boldsymbol{z},\boldsymbol{z}^{\prime}\in \mathbb{R}^d$,}\\
\noindent $ii.$ $E(\mathrm{MMD}^{2}_{\mathrm{BNP}}(F_{1,N}^{pri},F_{2,m}))\rightarrow \mathrm{MMD}^{2}(H,F_{2})$ as $a\rightarrow\infty$, $N\rightarrow\infty$, and $m\rightarrow\infty$,\\
\noindent $iii.$
$E(\mathrm{MMD}^{2}_{\mathrm{BNP}}(F_{1,N}^{pri},F_{2,m}))< \mathrm{MMD}^{2}(H,F_{2})+3K$, for any $N,m\in \mathbb{N}$ and $a\in \mathbb{R}^+$,\\
where ``$\xrightarrow{a.s.}$" denotes the almost surely convergence, $\mathbb{N}$ denotes the natural numbers and $\mathbb{R}^+$ denotes the positive real numbers.
%    \noindent $iii.$ $Var(\mathrm{MMD}^{2}_{\mathrm{BNP}}(F_{1}^{pri},F_{2}))\rightarrow 0 $ as $a\rightarrow\infty$, $N\rightarrow\infty$, and $m\rightarrow\infty$,\\
\end{theorem}

After observing samples $\mathbf{X}_1,\ldots,\mathbf{X}_n$ from $F_1$ and considering  $\mathbf{V}^{\ast}_{1},\ldots,\mathbf{V}^{\ast}_{N}\overset{i.i.d.}{\sim}H^{\ast}$, and $(J^{\ast}_{1,N},\ldots,J^{\ast}_{N,N})\sim \mbox{Dirichlet}(\frac{a+n}{N},\ldots,\frac{a+n}{N})$, we update the prior-based MMD estimator \eqref{BNP-pri-MMD} to the posterior one as
\begin{align}\label{BNP-pos-MMD}
\mathrm{MMD}^{2}_{\mathrm{BNP}}(F_{1,N}^{pos},F_{2,m})&= \sum_{\ell,t=1}^{N} J^{\ast}_{\ell,N}J^{\ast}_{t,N}k(\mathbf{V}^{\ast}_{\ell},\mathbf{V}^{\ast}_{t})
-\dfrac{2}{m}\sum_{\ell=1}^{N}\sum_{t=1}^{m} J^{\ast}_{\ell,N}k(\mathbf{V}^{\ast}_{\ell},\mathbf{Y}_{t})+\dfrac{1}{m^2}\sum_{\ell,t=1}^{m} k(\mathbf{Y}_{\ell},\mathbf{Y}_{t}),
\end{align}
where, $H^{\ast}=a/(a+n)H+n/(a+n)F_{1,n}$, $F_{1,n}$ denotes the empirical distribution of observed data, and $F_{1,N}^{pos}$ refers to the approximation of $F_{1}|\mathbf{X}_{1:n}\sim DP(a+n,H^{\ast})$.
The following Theorem presents asymptotic properties of $\mathrm{MMD}^{2}_{\mathrm{BNP}}(F_{1,N}^{pos},F_{2,m})$.
\begin{theorem}\label{thm-pos}
For a non-negative real value $a$ and fixed probability distribution $H$, let $F^{pri}_{1}:=(F_1\sim DP(a, H))$ and $k(\cdot,\cdot)$ be any continuous kernel function with feature space corresponding to a universal RKHS. Assume that $|k(\boldsymbol{z},\boldsymbol{z}^{\prime})|<K$, for any $\boldsymbol{z},\boldsymbol{z}^{\prime}\in \mathbb{R}^d$. Then, for a given sample $\mathbf{X}_1,\ldots,\mathbf{X}_n$ from distribution $F_1$,\\
\noindent $i.$  as $a\rightarrow\infty$ (informative prior), 
\begin{enumerate}
\item[a.] $\mathrm{MMD}^{2}_{\mathrm{BNP}}(F_{1,N}^{pos},F_{2,m})\xrightarrow{a.s.}\mathrm{MMD}^{2}(H_N,F_{2,m})$,
\item[b.] 
$E(\mathrm{MMD}^{2}_{\mathrm{BNP}}(F_{1,N}^{pos},F_{2,m}))\rightarrow \mathrm{MMD}^{2}(H,F_{2})$,  $N\rightarrow\infty$, and $m\rightarrow\infty$,\\
\end{enumerate}
%\textcolor{gray}{\noindent $ii.$  $\mathrm{MMD}^{2}_{\mathrm{BNP}}(F_{1,N}^{pos},F_{2,m}) $ converges in probability at rates $O((N+m)^{-1/2})$ to $ \mathrm{MMD}^{2}(H,F_{2})$, as $a\rightarrow\infty$,}\\
\noindent $ii.$ as  $n\rightarrow\infty$ (consistency),
\begin{enumerate}
\item[a.] $\mathrm{MMD}^{2}_{\mathrm{BNP}}(F_{1,N}^{pos},F_{2,m})\xrightarrow{a.s.}\mathrm{MMD}^{2}(F_{1,N},F_{2,m})$,
\item[b.] 
$E(\mathrm{MMD}^{2}_{\mathrm{BNP}}(F_{1,N}^{pos},F_{2,m}))\rightarrow \mathrm{MMD}^{2}(F_1,F_{2})$, as $N\rightarrow\infty$, $n\rightarrow\infty$, and $m\rightarrow\infty$.
\end{enumerate}
\end{theorem}

We conclude this section by presenting a corollary that plays a significant role in the two following sections.
\begin{corollary}\label{cor-pos-consis}
Under the assumption of Theorem \ref{thm-pos},\\
\noindent $i.$  as $a\rightarrow\infty$, $N\rightarrow\infty$, $m\rightarrow\infty$, then,
\begin{enumerate}
\item[a.] $E(\mathrm{MMD}^{2}_{\mathrm{BNP}}(F_{1,N}^{pri},F_{2,m}))\rightarrow 0$, if and only if $H=F_2$,
\item[b.] $E(\mathrm{MMD}^{2}_{\mathrm{BNP}}(F_{1,N}^{pos},F_{2,m}))\rightarrow 0$, if and only if $H=F_2$,
\end{enumerate}
\noindent 
$ii.$ for any choice of $a$ and $H$, 
$E(\mathrm{MMD}^{2}_{\mathrm{BNP}}(F_{1,N}^{pos},F_{2,m}))\rightarrow 0$, if and only if $F_1=F_2$, as $N\rightarrow\infty$, and $n\rightarrow\infty$, and $m\rightarrow\infty$.    
\end{corollary}
%\begin{proof}
%The proofs are immediately followed by Theorem \ref{lem-%pri-MMD} and Theorem \ref{thm-pos}.    
%\end{proof}
\section{Constructing a GOF Test with RB Ratio}\label{sec-BNP-test}
In this section, we introduce our novel semi-BNP test, utilizing the proposed estimator discussed in the previous section, to evaluate the hypothesis $\mathcal{H}_0: F_1=F_2$. We put forward an equivalent formulation to test the hypothesis 
\begin{align}\label{H_null}
\mathcal{H}_0: \mathrm{MMD}^{2}(F_{1},F_{2})=0,
\end{align}
using the RB\footnote{A detailed discussion on the RB ratio is provided in the supplementary material.} ratio, introduced by \cite{Evans15}, as the Bayesian evidence.

By relating our problem to RB inference, with $\Psi=\mathrm{MMD}^{2}(F_{1}, F_{2})$ and $\psi_{0}=0$, the RB ratio measures the change in belief regarding the true value of $\psi_{0}$, from \emph{a priori} to \emph{a posteriori}, given a sample $\mathbf{X}_1,\ldots,\mathbf{X}_n$ from $F_1$. It can be expressed by
\begin{align}\label{RB-mmd}
RB_{\mathrm{MMD}^2(F_1,F_2)}(0|\mathbf{X}_{1:n})=\dfrac{\pi_{\mathrm{MMD}^2(F_1,F_2)}(0|\mathbf{X}_{1:n})}{\pi_{\mathrm{MMD}^2(F_1,F_2)}(0)},
\end{align}
where, $\pi_{\mathrm{MMD}^2(F_1,F_2)}(\cdot|\mathbf{X}_{1:n})$\footnote{Note that
the subscript $(F_1, F_2)$ may be omitted whenever it is clear in the context.} and $\pi_{\mathrm{MMD}^2(F_1,F_2)}(\cdot)$ denote the density functions of the estimators given by \eqref{BNP-pos-MMD} and \eqref{BNP-pri-MMD}, respectively.

The density in the denominator of \eqref{RB-mmd} must support $\mathcal{H}_0$ in order to reflect how well the data can support the null hypothesis based on the comparison between the prior and the posterior, utilizing the fundamental concepts of the RB ratio. Here, supporting $\mathcal{H}_0$ by $\pi_{\mathrm{MMD}^2}(\cdot)$ means to place most prior mass on zero. To enforce this term on $\pi_{\mathrm{MMD}^2}(\cdot)$, it is enough to set $H=F_2$ in $DP(a, H)$, which is deduced from the Theorem \ref{lem-pri-MMD}, part (iii). In this case, when $\mathcal{H}_0$ is not true, for a fixed $a$ and $K$ (the upper bound of the kernel $k(\cdot,\cdot)$), the range of $\mathrm{MMD}^{2}_{\mathrm{BNP}}(F_{1,N}^{pri},F_{2,m})$ should, on average, vary within a smaller range than its corresponding posterior version. Specifically, this range should be $(0, 3K)$, compared to $(0, \mathrm{MMD}^{2}_{\mathrm{BNP}}(H^{\ast},F_{2})+3K)$ which can be similarly obtained for the posterior-based MMD estimator. This indicates that $\mathcal{H}_0$ should be rejected, as it is desirable. On the other hand, when $\mathcal{H}_0$ is true, although the prior and posterior-based MMD estimators have approximately the same range of variation $(0, 3K)$, Corollary \ref{cor-pos-consis}(ii)  implies that increasing the sample size leads the posterior to provide stronger evidence in favor of the null hypothesis compared to the prior, resulting in the acceptance of $\mathcal{H}_0$.

With regards to choosing the concentration parameter $a$ in our proposed test, we note that $a$ controls the variation of $F^{pri}$ around $H$, which in turn controls the strength of belief in the truth of $\mathcal{H}_0$. It is recommended to choose $a<n/2$ based on the definition of $H^{\ast}$ in $F^{pos}$ \citep{labadi2014two}. The idea behind using such a value of $a$ is to avoid the excessive effect of the prior $H$ on the test results by considering the chance of sampling from the observed data to be at least twice the chance of generating samples from $H$. Corollaries \ref{cor-pos-consis}(i) also clearly point to this issue in the informative prior case, as both expectations of $\mathrm{MMD}^{2}_{\mathrm{BNP}}(F_{1,N}^{pos},F_{2,m})$ and $\mathrm{MMD}^{2}_{\mathrm{BNP}}(F_{1,N}^{pri},F_{2,m})$ tend to $0$ as $a\rightarrow\infty$, $N\rightarrow\infty$, and $m\rightarrow\infty$.  
Hence, both prior and posterior densities in \eqref{RB-mmd} should be heavily massed and coincide with each other at zero. It causes the value of \eqref{RB-mmd} to become very close to 1, based on which no decision can be made about $\mathcal{H}_0$.

For the proposed test, we will empirically choose $a$ to be less than $n/2$ and then compute \eqref{RB-mmd}. However, some computational methods in the literature have been proposed to elicit $a$ that one may be interested in using \citep{al2022test,al2021two}. Generally, for a given $a$, Corollary \ref{cor-pos-consis}(ii) implies that $\mathrm{MMD}^{2}_{\mathrm{BNP}}(F_{1,N}^{pos},F_{2,m})$ should be more dense than $\mathrm{MMD}^{2}_{\mathrm{BNP}}(F_{1,N}^{pri},F_{2,m}) $ at 0 if and only if $\mathcal{H}_0$ is true.
Hence, the value of \eqref{RB-mmd} presents evidence for or against $\mathcal{H}_0$, if $RB_{\mathrm{MMD}^2}(0|\mathbf{X}_{1:n})>1$ or  $RB_{\mathrm{MMD}^2}(0|\mathbf{X}_{1:n})<1$, respectively. Following \cite{Evans15}, the calibration of \eqref{RB-mmd} is defined as:
\begin{align}\label{str-mmd}
Str_{\mathrm{MMD}^2}(0\,|\,\mathbf{X}_{1:n})&=\Pi_{\mathrm{MMD}^2}\big(RB_{\mathrm{MMD}^2}(mmd^2\,|\,\mathbf{X}_{1:n})\leq RB_{\mathrm{MMD}^2}(0\,|\,\mathbf{X}_{1:n})\,|\,\mathbf{X}_{1:n}\big),
\end{align}
where, $\Pi_{\mathrm{MMD}^2}(\cdot|\mathbf{X}_{1:n})$ is the posterior probability measure corresponding to the density $\pi_{\mathrm{MMD}^2}(\cdot|\mathbf{X}_{1:n})$.
When \eqref{H_null} is false, a small value of \eqref{str-mmd} provides strong evidence against $\psi_{0}$, whereas a large value suggests weak evidence against $\psi_{0}$. Conversely, when \eqref{H_null} is true, a small value of \eqref{str-mmd} indicates weak evidence in favor of $\psi_{0}$, while a large value suggests strong evidence in favor of $\psi_{0}$. Particular attention should be paid here to the computation of \eqref{RB-mmd} and \eqref{str-mmd}. The densities used in \eqref{RB-mmd} do not have explicit forms. Thus, we use their corresponding ECDF based on $\ell$ sample sizes to estimate \eqref{RB-mmd} and \eqref{str-mmd}, respectively, as
\begin{align}
\widehat{RB}_{\mathrm{MMD}^2}(0\,|\,\mathbf{X}_{1:n})&=\frac{\hat{\Pi}_{\mathrm{MMD}^2}(\hat{d}_{i_0/M}|\,\mathbf{X}_{1:n})}{\hat{\Pi}_{\mathrm{MMD}^2}(\hat{d}_{i_0/M})}, \label{rbest}\\
\widehat{Str}_{\mathrm{MMD}^2}(0\,|\,\mathbf{X}_{1:n})&=\sum_{D}\big(\hat{\Pi}_{\mathrm{MMD}^2}(\hat{d}_{(i+1)/M}|\,\mathbf{X}_{1:n})-\hat{\Pi}_{\mathrm{MMD}^2}(\hat{d}_{i/M}|\,\mathbf{X}_{1:n})\big),\label{strest}
\end{align}
where, $    D=\left\lbrace 0\leq i\leq M-1 :\widehat{RB}_{\mathrm{MMD}^2}\big(\hat{d}_{i/M}\,|\,\mathbf{X}_{1:n}\big)\leq \widehat{RB}_{\mathrm{MMD}^2}\big(0\,|\,\mathbf{X}_{1:n}\big)\right\rbrace,$
%\begin{align*}
%    D=\left\lbrace 0\leq i\leq M-1 :\widehat{RB}_{\mathrm{MMD}^2}\big(\hat{d}_{i/M}\,|\,\mathbf{X}_{1:n}\big)\leq \widehat{RB}_{\mathrm{MMD}^2}\big(0\,|\,\mathbf{X}_{1:n}\big)\right\rbrace,
%\end{align*}
in which $M$ is a positive number,  $\hat{d}_{i/M}$ is the estimate of $d_{i/M},$ the $(i/M)$-th prior
quantile of  \eqref{BNP-pri-MMD},
\begin{align*}
\widehat{RB}_{\mathrm{MMD}^2}(\hat{d}_{i/M}\,|\,\mathbf{X}_{1:n})=
\frac{\hat{\Pi}_{\mathrm{MMD}^2}(\hat{d}_{\frac{i+1}{M}}|\,\mathbf{X}_{1:n})-\hat{\Pi}_{\mathrm{MMD}^2}(\hat{d}_{\frac{i}{M}}|\,\mathbf{X}_{1:n})}{\hat{\Pi}_{\mathrm{MMD}^2}(\hat{d}_{\frac{i+1}{M}})-\hat{\Pi}_{\mathrm{MMD}^2}(\hat{d}_{\frac{i}{M}})}
\end{align*}
and $i_{0}$ in \eqref{rbest} is chosen so that $i_{0}/M$ is not too small (typically $i_{0}/M=0.05$). Further details are available in Algorithm 1 in the supplementary material. For fixed $M$, as $N\rightarrow\infty$ and $\ell\rightarrow\infty,$ then $\hat{d}_{i/M}$ converges almost surely to $d_{i/M}$ and (\ref{rbest}) and (\ref{strest}) converge almost surely to \eqref{RB-mmd} and \eqref{str-mmd}, respectively. The following result from \citet[Proposition 6]{al2018prior} gives the consistency of the proposed test.  In the sense that, if $\mathcal{H}_0$ is true, then
\eqref{RB-mmd} and \eqref{str-mmd} converge, respectively, almost surely to $M/i_0(>1)$ and  $1$, as $n\rightarrow\infty$; otherwise, both converge to $0$.

The proposed test is suggested to overcome several limitations present in its frequentist counterparts. In a frequentist test, for a given permissible type I error rate %(the probability of wrongly rejecting $\mathcal{H}_0$ when it is true), 
denoted by $\alpha$, the test rejects $\mathcal{H}_0$ if the value of $\mathrm{MMD}^{2}(F_{1}, F_{2})$ is greater than some threshold $c_{\alpha}$. The corresponding $p$-value for this test can also be computed by $\mathrm{Pr}(\mathrm{MMD}^{2}(F_{1}, F_{2})\geq c_{\alpha}|\mathcal{H}_0)$, which leads the test to reject $\mathcal{H}_0$ if it is less than $\alpha$. However, \citet{li2017mmd} noted that if $\mathrm{MMD}^{2}(F_{1}, F_{2})$ is not significantly larger than $c_{\alpha}$ for some finite samples when $\mathcal{H}_0$ is not true, the null hypothesis $\mathcal{H}_0$ is not rejected. Furthermore, there is a trade-off between the permissible type I error rate $\alpha$ and the probability of failing to reject a false null hypothesis (type II error), denoted by $\beta$, as $\alpha+\beta \leq 1$. Decreasing one error rate inevitably leads to an increase in the other, indicating that we cannot arbitrarily drive to type I error rate to zero. 
Moreover, the $p$-values are uniformly distributed between 0 and 1 under the null hypothesis. In fact, it does not allow evidence for the null, which is one of their weaknesses compared to Bayesian criteria in hypothesis testing problems.

%{Moreover, the $p$-value does not allow researchers to provide evidence for the null hypothesis while also exaggerating the evidence against the null hypothesis \citep{Al-Labadi}. When $\mathcal{H}_0$ should be rejected, the $p$-value in frequentist tests goes to zero as the sample size increases. Although it is known as a desirable feature of the $p$-value, a non-desirable feature is that the $p$-values are %all equally likely and  uniformly distributed between 0 and 1 under the null hypothesis. This distribution holds regardless of the sample size which means that increasing the sample size will not help to gain evidence for the null hypothesis in this case \citep{rouder2009bayesian}.}
\section{Embedding the Semi-BNP Estimator in GAN Learning}\label{sec-gan-train}
In this section, we propose a BNPL procedure that leverages a posterior-based MMD estimator to train GANs. It is inspired by the idea presented in \cite{dellaporta2022robust} to approximate the posterior on the generator's parameters. 
%For a building on graphical representation of this idea, see \citet[Figure 1]{dellaporta2022robust}.} 
%This procedure is inspired by \cite{dellaporta2022robust}, where it is suggested that incorporating the true distribution of the data into Equation \eqref{cost-MMD-basic} leads to an infinite number of parameters in the generating function. Similarly, by introducing the posterior on the data distribution in Equation \eqref{cost-MMD-basic}, we obtain a sample of parameters generated from the posterior on the parameters.  

\subsection{Generative Adversarial Networks}
% In this section, we review previous research that has framed the discriminator of a GAN as a nonparametric two-sample \textcolor{cyan}{comparator} used in hypothesis tests. 
The GAN \citep{Goodfellow} is a machine learning technique used to generate realistic-looking artificial samples. In this context, the discriminator $ D $ can be viewed as a black box that uses a discrepancy measure $ \delta $ to differentiate between the real and fake data. Meanwhile, the generator $ G_{\boldsymbol{\omega}} $ is trained by optimizing a simpler objective function, given by 
\begin{align*}%\label{cost-MMD-basic}
\arg\min\limits_{\boldsymbol{\omega}\in\mathcal{W}} \delta(F,F_{G_{\boldsymbol{\omega}}}),
\end{align*}
where $ F_{G_{\boldsymbol{\omega}}} $ represents the distribution of the generator. In fact, $D$ attempts to continuously train $G_{\boldsymbol{\omega}}$ by computing distance $\delta$ between $F$ and $F_{G_{\boldsymbol{\omega}}}$ until this distance is negligible, making their difference indistinguishable. This technique leads to omitting the neural network from $D$, whose optimization may lead to a vanishing gradient. An effective measure of discrepancy for $\delta$ is the MMD, which is a kernel-based measure that offers several desirable properties such as consistency and robustness in generating samples \citep{Gretton,cherief2022finite}.  %Precisely, 

Numerous frequentist GANs applying the MMD measure to estimate the generator's parameters can be found in the literature. \citep{dziugaite2015training,binkowski2018demystifying,Li}. These models are devised by comparing the generated fake samples with real samples. In addition to the MMD, several other discrepancy measures are commonly used for GANs, including the $f$-divergence measure \citep{nowozin2016f}, the Wasserstein distance \citep{arjovsky2017wasserstein}, and the total variation distance \citep{Lin}.  Nevertheless, the MMD kernel-based measure is remarkably robust against outliers and has the exceptional ability to capture complex relationships and dependencies in the data \citep{sejdinovic2013equivalence, cherief2022finite}. This makes it highly effective in handling model misspecification and detecting subtle differences between distributions. This property is particularly useful for modeling complicated datasets such as images, which are often tackled with GANs. Moreover,  \cite{Al-Labadi} used the energy distance to expand their procedure, which is a member of the larger class of MMD kernel-based measures \citep{sejdinovic2013equivalence}. From here, it is obvious that choosing among a larger class can lead to designing more sensitive discrepancy measures to detect differences. 

Moreover, although a particular case of the test of \citet{Al-Labadi} can be used to compare two distributions, it cannot be considered a convenient discriminator in the minimum distance estimation technique to train GANs. In GANs, the objective is to update the parameter $\boldsymbol{\omega}$ of the deterministic generative neural network $G_{\boldsymbol{\omega}}$. Therefore, treating $F_{G_{\boldsymbol{\omega}}}$ as an unknown distribution on which we place a BNP prior is nonsensical. Consequently, a more suitable distance criterion is required to compare an intractable parametric distribution with an unknown distribution.
\subsection{Architecture}
Various GAN architectures can be found in the literature to model complex high-dimensional distributions. However, we consider the original architecture of GANs proposed by \cite{Goodfellow}, with the difference that here only the generator 
is considered as a neural network and the discriminator $D$ is formed as the semi-BNP estimator. 

Specifically, we follow \cite{Goodfellow} to consider the generator $ G_{\boldsymbol{\omega}} $  as a multi-layer neural network with parameters $\boldsymbol{\omega}$, rectified linear units activation function for each hidden layer, and a sigmoid function for the last layer (output layer). The generator receives a noise vector $\boldsymbol{U}=(U_1,\ldots,U_p)$ as its input nodes, where $p<d$, and each element of $\boldsymbol{U}$ is independently drawn from the same distribution $F_U$. Our BNPL procedure is then expanded based on producing a realistic sample, which is the output of $ G_{\boldsymbol{\omega}} $ in the data space $ \mathbb{R}^d $, based on updating $\boldsymbol{\omega}$ by optimizing the objective function:
\begin{align*}
\arg\min\limits_{\boldsymbol{\omega}\in\mathcal{W}} \mathrm{MMD}^2_{\mathrm{BNP}}(F^{pos}_{N},F_{G_{\boldsymbol{\omega}},m}).
\end{align*}
In fact, our desired BNPL procedure implicitly tries to approximate samples from the posterior distribution on the parameter $\boldsymbol{\omega}$ by minimizing the posterior-based MMD estimator.
%calculates the RB ratio and the generator should attempt to maximize this value. This is equivalent to the generator trying to minimize the posterior-based MMD measure in the numerator of the RB ratio, when the \textit{a priori} term is considered fixed in the denominator of the RB ratio. Hence, the problem should be changed to optimize the objective function $ \arg\min\limits_{\boldsymbol{\omega}} \mathrm{MMD}^2_{\mathrm{BNP}}(F^{pos}_{N},F_{G_{\boldsymbol{\omega}},m})$. 
For any differentiable kernel function 
$k(\cdot,\cdot)$, this optimization is performed by computing the following gradient based on samples from $F|\mathbf{X}_{1:n}\sim DP(a+n,H^{\ast})$, as
\begin{align*}%\label{gradiant}
\frac{\partial \mathrm{MMD}^{2}_{\mathrm{BNP}}(F^{pos}_{N},F_{G_{\boldsymbol{\omega}},m})}{\partial \boldsymbol{\omega}_i}&=\sum_{\ell=1}^{N}\sum_{t=1}^{m}\Bigg\lbrace\frac{\partial}{\partial \mathbf{Y}_t}\Bigg[-\frac{2}{m}\sum_{t=1}^{m} J^{\ast}_{\ell,N}k(\mathbf{V}^{\ast}_{\ell},\mathbf{Y}_{t})\nonumber\\
&~~~~~~~~~~~~~~~~~~~~~~~~~+\frac{1}{Nm^2}\sum_{t,t^{\prime}=1}^{m} k(\mathbf{Y}_{t},\mathbf{Y}_{t^{\prime}})\Bigg]\frac{\partial \mathbf{Y}_{t}}{\partial \boldsymbol{\omega}}\Bigg\rbrace,
\end{align*}
where, $\mathbf{Y}_{t}=G_{\boldsymbol{\omega}}(\boldsymbol{U}_t)$, $\boldsymbol{U}_t=(U_{t1},\ldots,U_{tp})$, and $U_{ti}$'s are generated from a distribution $F_U$, for $t=1,\ldots,m$, and $i=1,\ldots,p$. Then, the backpropagation method is applied 
for calculating partial derivatives $\frac{\partial \mathbf{Y}_{t}}{\partial \boldsymbol{\omega}}$ to update the parameters of $G_{\boldsymbol{\omega}}$.

However, \citet[Equation 8]{Li} remarked that considering the square root of the  MMD measure given by \eqref{mmd-def1} in the cost function of frequentist GANs is more efficient than using \eqref{mmd-def1} to train network $G_{\boldsymbol{\omega}}$. They mentioned that since the gradient of $\sqrt{\mathrm{MMD}^{2}(F_{N},F_{G_{\boldsymbol{\omega}},m})}$ with respect to $\boldsymbol{\omega}$ is the product of $\gamma_1=\frac{1}{2\sqrt{\mathrm{MMD}^{2}(F_{N},F_{G_{\boldsymbol{\omega}},m})}}$ and $\gamma_2=\frac{\partial \mathrm{MMD}^{2}(F_{N},F_{G_{\boldsymbol{\omega}},m})}{\partial\boldsymbol{\omega}}$, then $\gamma_1$ forces the value of the gradient to be relatively large, even if both $ \mathrm{MMD}^{2}(F_{N},F_{G_{\boldsymbol{\omega}},m})$ and $\gamma_2$ are small. This can prevent the vanishing gradient, which improves the learning of the parameters of $G_{\boldsymbol{\omega}}$ in the early layers of this network. We consider this point in order to improve our semi-BNP objective function:
\begin{align}\label{semi-bnp-obj}
\arg\min\limits_{\boldsymbol{\omega}\in\mathcal{W}} \mathrm{MMD}_{\mathrm{BNP}}(F^{pos}_{N},F_{G_{\boldsymbol{\omega}},m}).
\end{align}
Algorithm 2 in the supplementary material provides steps for implementing the training.

Let $\boldsymbol{\omega}^{\ast}$ be the optimized parameter of $G_{\boldsymbol{\omega}}$ that minimizes $\mathrm{MMD}_{\mathrm{BNP}}(F^{pos}_{N},F_{G_{\boldsymbol{\omega}},m})$. 
Since $\mathrm{MMD}_{\mathrm{BNP}}(F^{pos}_{N},F_{G_{\boldsymbol{\omega}},m})$ can be viewed as a semi-BNP estimation of \eqref{mmd-def1}, it becomes imperative to assess the accuracy of this estimation, specifically in terms of how effectively the proposed GAN can generate realistic samples that faithfully represent the true data distribution (generalization error).  Furthermore, it is crucial to take into consideration the generator's performance in dealing with outliers which includes a small proportion of observations that deviate from the clean data distribution $F_0$ (robustness). The next lemma addresses these two concerns.

\begin{lemma}\label{g-error and robustness}
Let $\mathcal{W}$ be the parameter space for $G_{\boldsymbol{\omega}}$ and $\boldsymbol{\omega}^{\ast}\in\mathcal{W}$ be the value that optimizes the objective function \eqref{semi-bnp-obj} and $\boldsymbol{\omega}^{\prime}$ be the true value that minimizes $\mathrm{MMD}(F,F_{G_{\boldsymbol{\omega}}})$. Assume that $F\sim DP(a, H)$ and let $k(\cdot,\cdot)$ be any continuous kernel function with feature space corresponding to a universal RKHS such that $|k(\boldsymbol{z},\boldsymbol{z}^{\prime})|<K$, for any $\boldsymbol{z},\boldsymbol{z}^{\prime}\in \mathbb{R}^d$. For a given sample $\mathbf{X}_1,\ldots,\mathbf{X}_n$ from distribution $F$:\\
\noindent $i.$ Generalization error:
\begin{align*}
E\left(\mathrm{MMD}(F,F_{G_{\boldsymbol{\omega}^{\ast}}})\right)\leq \mathrm{MMD}(F,F_{G_{\boldsymbol{\omega}^{\prime}}})+\dfrac{2K}{\sqrt{n}}+\dfrac{4aK}{a+n}+2\sqrt{\dfrac{(a+n+N)K}{(a+n+1)N}}.
\end{align*}
\noindent $ii.$ Robustness: Suppose there exist outliers in the sample data, which arise from a noise distribution $Q$.  Consider the H\"{u}ber's contamination model \citep{huber1992robust,cherief2022finite}, given by $F=(1-\epsilon)F_0+\epsilon Q$, where $\epsilon\in (0,\frac{1}{2})$ is the contamination rate, and the latent variables $Z_1,\ldots,Z_n\overset{i.i.d.}{\sim}\mathrm{Bernoulli}(\epsilon)$ are such that $\mathbf{X}_{i}\overset{i.i.d.}{\sim}F_0$ if  $Z_i=0$; otherwise, $\mathbf{X}_{i}\overset{i.i.d.}{\sim}Q$.
%\begin{align*}
%  \mathbf{X}_{i}\overset{i.i.d.}{\sim}
%    \begin{cases}
%      F_0 & \text{if $Z_i=0$}\\
%      Q & \text{if $Z_i=1$}.\\
%    \end{cases}       
%\end{align*}
Then,
\begin{align*}
E\left(\mathrm{MMD}(F_0,F_{G_{\boldsymbol{\omega}^{\ast}}})\right)\leq \min\limits_{\boldsymbol{\omega}\in\mathcal{W}}\mathrm{MMD}(F_0,F_{G_{\boldsymbol{\omega}}})+4\epsilon+\dfrac{2K}{\sqrt{n}}+\dfrac{4aK}{a+n}+2\sqrt{\dfrac{(a+n+N)K}{(a+n+1)N}}.
\end{align*}
\end{lemma}

Lemma \ref{g-error and robustness}(ii) demonstrates that despite encountering outlier data, $F_{G_{\boldsymbol{\omega}^{\ast}}}$ and $F_0$ are negligibly different for a sufficiently large sample size. This feature results in the majority of the posterior on the parameter space $\mathcal{W}$ being distributed on value $\boldsymbol{\omega}^{\ast}$, which is a desirable outcome of the proposed method.

Although the preceding statements investigate properties of the estimated parameters by providing upper bounds for the expectation of the MMD estimator, the next lemma presents stochastic bounds for the estimation error in order to assess the posterior consistency.

\begin{lemma}\label{prob-pos-bound}
Building upon the general assumptions stated in Lemma \ref{g-error and robustness}, for a given sample $\mathbf{X}_1,\ldots,\mathbf{X}_n$ from distribution $F$ in the probability space $(\mathfrak{X},\mathcal{A},\mathrm{Pr})$ and any $\epsilon>0$,\\
\noindent $i.$ $\mathrm{Pr}\left(|\mathrm{MMD}(F^{pos}_{N},F_{G_{\boldsymbol{\omega}^{\ast}},m})-\mathrm{MMD}(F,F_{G_{\boldsymbol{\omega}^{\prime}}})|\geq h(n,m,\epsilon)+|\Delta_1|+|\Delta_2|\right)\leq 2\exp{\frac{-\epsilon^2nm}{2K(n+m)}},$\\
$ii.$ $\mathrm{Pr} \left( \mathrm{MMD}(F,F_{G_{\boldsymbol{\omega}^{\ast}}})>\epsilon\right)\leq \dfrac{1}{\epsilon}\left(\mathrm{MMD}(F,F_{G_{\boldsymbol{\omega}^{\prime}}})+\dfrac{2K}{\sqrt{n}}+\dfrac{4aK}{a+n}+2\sqrt{\dfrac{(a+n+N)K}{(a+n+1)N}}.\right)$,\\
where, $h(N,m,K,\epsilon)=2\sqrt{K}(\sqrt{n}+\sqrt{m})/\sqrt{nm}+\epsilon$, $\Delta_1=\mathrm{MMD}(F^{pos}_{N},F_{G_{\boldsymbol{\omega}^{\ast}}})-\mathrm{MMD}(F_n,F_{G_{\boldsymbol{\omega}^{\prime},m}})$, and $\Delta_2=\mathrm{MMD}(F,F_{G_{\boldsymbol{\omega}^{\ast}}})-\mathrm{MMD}(F,F_{G_{\boldsymbol{\omega}^{\prime}}})$.
%\begin{align*}
%    h(N,m,K,\epsilon)&=2\sqrt{K}(\sqrt{n}+\sqrt{m})/\sqrt{nm}+\epsilon,\\
%    \Delta_1&=\mathrm{MMD}(F^{pos}_{N},F_{G_{\boldsymbol{\omega}^{\ast}}})-\mathrm{MMD}(F_n,F_{G_{\boldsymbol{\omega}^{\prime},m}}),\\
%    \Delta_2&=\mathrm{MMD}(F,F_{G_{\boldsymbol{\omega}^{\ast}}})-\mathrm{MMD}(F,F_{G_{\boldsymbol{\omega}^{\prime}}}).
%\end{align*}

\end{lemma}

A direct consequence of Lemma \ref{prob-pos-bound}(ii) is that for a fixed value of $a$, $\mathrm{Pr} ( \mathrm{MMD}(F,F_{G_{\boldsymbol{\omega}^{\ast}}})\geq\epsilon)\rightarrow 0$, as $n\rightarrow\infty$ and $N\rightarrow\infty$, for any $\epsilon>0$, when $\mathrm{MMD}(F,F_{G_{\boldsymbol{\omega}^{\prime}}})=0$ (well-specified case). This implies $F_{G_{\boldsymbol{\omega}^{\ast}}}$ converges in probability to the data distribution $F$ as the sample size increases in well-specified cases.

Note that, choosing the value of $a$ in the test proposed in Section \ref{sec-BNP-test} plays a crucial role in determining the degree of support for the null hypothesis against the alternative. However, in the current context of approximating the posterior on the parameter space, the prior choice for $F$ and determining the strength of belief becomes challenging. Therefore, we opt for $a=0$ as a non-informative prior, as suggested by \cite{dellaporta2022robust}, thanks to its broad ability to characterize uncertainty \citep{terenin2017noninformative}.

The main distinction between our BNPL method and the one proposed by \cite{dellaporta2022robust} lies in the fact that we generalize their BNPL procedure beyond estimating parameters and explicitly consider the terms of the DP posterior approximation and their corresponding weights. \citeauthor{dellaporta2022robust} used the following DP approximation:
\begin{align*}
F^{pos}_{n+N}=\sum_{\ell=1}^{n}\widetilde{J}_{\ell,n}\delta_{\mathbf{X}_{\ell}}+\sum_{t=1}^{N}J_{t,N}\delta_{\mathbf{V}_{t}},
\end{align*}
where $(\widetilde{J}_{1:n,n},J_{1:N,N})\sim \mbox{Dirichlet}(1,\ldots,1,\frac{a}{N},\ldots,\frac{a}{N})$, $(\mathbf{X}_{1:n})\overset{i.i.d.}{\sim}F$, and $(\mathbf{V}_{1:n})\overset{i.i.d.}{\sim}H$. In contrast, we employ $F^{pos}_{N}=\sum_{i=1}^{N}J^{\ast}_{i,N}\delta_{V^{\ast}_{i}}$, with $(J^{\ast}_{1:N,N})\sim \mbox{Dirichlet}(\frac{a+n}{N},\ldots,\frac{a+n}{N})$. Our approach offers an advantage over the approximation used in \cite{dellaporta2022robust} due to its reduced number of terms, significantly reducing both computational and theoretical complexity. Additionally, a further difference is that Dellaporta's bootstrap procedure needs to query the loss function $B$ times to simulate $B$ posterior parameters, whereas our procedure does not require a bootstrap algorithm and we only need to simulate a single parameter. Although their bootstrap procedure is embarrassingly parallelizable, $B$ generally should be a fairly large number and the typical statistical practitioner does not have access to $B$ cores to truly parallelize the additional cost of bootstrap sampling.%, we only employ our cost function to simulate a single posterior parameter.}
\subsection{Kernel Settings}
In our method, we choose to use the standard radial basis function (RBF) kernel as its feature space corresponds to a universal RKHS.  For a comprehensive understanding of RBF functions, refer to Section 4 in the supplementary material.
%To implement the method we need to choose a member of the RBF kernel family as its feature space corresponds to a universal RKHS. 
\citet{dziugaite2015training, Li} and \citet{li2017mmd} used the Gaussian kernel in training MMD-GANs because of its simplicity and good performance. \citet{dziugaite2015training} also evaluated some other RBF kernels such as the Laplacian and rational quadratic kernels to compare the results of the MMD-GANs with those obtained based on using Gaussian kernels. They found the best performance by applying the Gaussian kernel in the MMD cost function. 

Hence, we consider the Gaussian kernel function in our proposed procedure. To choose the bandwidth parameter $\sigma$, we follow the idea of considering a set of fixed values of $\sigma$'s such as $\lbrace \sigma_1,\ldots,\sigma_T \rbrace$, then compute the mixture of Gaussian kernels  $k(\cdot,\cdot)=\sum_{t=1}^{T} k_{G_{\sigma_t}}(\cdot,\cdot)$, to consider in \eqref{BNP-pos-MMD}. For each $\sigma(t)$, $0\leq k_{G_{\sigma_t}}(\cdot,\cdot)\leq 1$; hence, $0\leq k(\cdot,\cdot)\leq T$, which satisfies the theoretical results presented in the paper. As it is mentioned in \cite{Li}, this choice reflects a good performance in training MMD-GANs.

\section{Experimental Investigation}\label{sec-experiment}
In this section, we empirically investigate our proposed methods through comprehensive numerical studies in the following two subsections, which demonstrate the superior performance of our proposed semi-BNP test as a standalone test as well as an embedded discriminator for the semi-BNP GAN. 

\subsection{The Semi-BNP Test}
To comprehensively study test performance evaluation, we consider some major representative examples in two-sample comparison problems. For this, let $\mathbf{y}_1,\ldots, \mathbf{y}_n$ be a sample generated from $ F_2=N(\mathbf{0}_{d},I_{d})$ and $\mathbf{x}_1,\ldots, \mathbf{x}_n$ be a sample generated from each below distributions: $F_1=N(\mathbf{0}_{d},I_{d})$ (No differences), $ F_1=N(\mathbf{0.5}_{d},I_{d})$ (Mean shift), $ F_1=LN(\mathbf{0}_{d},B_{d})$ (Skewness), $F_1=\frac{1}{2}N(-\mathbf{1}_{d},I_{d})+\frac{1}{2}N(\mathbf{1}_{d},I_{d})$ (Mixture), $F_1=N(\mathbf{0}_{d},2I_{d})$ (Variance shift), $ F_1=t_3(\mathbf{0}_{d},I_{d})$ (Heavy tail), and $F_1=LG(\mathbf{0}_{d},I_{d})$ (Kurtosis).

%\noindent \textbf{1. No differences}:  $ F_1=N(\mathbf{0}_{d},I_{d})$,\hspace{1cm}
% \textbf{2. Mean shift}: $ F_1=N(\mathbf{0.5}_{d},I_{d})$,\\
% \noindent\textbf{3. Skewness}: $ F_1=LN(\mathbf{0}_{d},B_{d})$,\hspace{1.6cm}
% \textbf{4. Mixture}: $F_1=\frac{1}{2}N(-\mathbf{1}_{d},I_{d})+\frac{1}{2}N(\mathbf{1}_{d},I_{d})$,\\
% \textbf{5. Variance shift}: $F_1=N(\mathbf{0}_{d},2I_{d})$,\hspace{.9cm}
% \textbf{6. Heavy tail}:  $ F_1=t_3(\mathbf{0}_{d},I_{d})$,\\
% \textbf{7. Kurtosis}: $F_1=LG(\mathbf{0}_{d},I_{d})$,\\
% where $N(\cdot,\cdot)$ denotes a normal distribution, $LN(\cdot,\cdot)$ denotes a lognormal distribution, $t_3(\cdot,\cdot)$ denotes a $t$-distribution with 3 degrees of freedom, and $LG(\cdot,\cdot)$ denotes a logistic distribution, where  $B_{d}$ is an $d\times d$ matrix with $0.25$ on the main diagonal and $0.2$ off the diagonal, $\mathbf{c}_d$ is a $d$-dimensional column vectors of $c$’s, and $I_d$ is a $d\times d$ identical matrix. In all distribution notations, the first component represents the mean vector and the second component represents the covariance matrix.

To implement the test, we set $\ell=1000$, $M=20$, and $\epsilon=10^{-3}$ to be used in Algorithm 1 in the supplementary material. We first considered the mixture of six Gaussian kernels corresponding to the suggested bandwidth parameters $2, 5, 10, 20, 40,$ and $ 80$ by \cite{Li}. We found that although this choice can provide good results in training GANs, it does not provide satisfactory results in hypothesis testing problems. 

Instead of using a mixture of several Gaussian kernels, we propose choosing a specific value for the bandwidth parameter that maximizes the area under the receiver operating characteristic curve (AUC) empirically. In a binary classifier, which can also be thought of as a two-sample test assessing whether two samples are distinguishable or not, the receiver operating characteristic (ROC) curve is a plot of true positive rates (sensitivity)  against the false positive rates (1-specificity) based on different choices of threshold to display the performance of the test. The positive term refers to rejecting $\mathcal{H}_0$ in \eqref{H_null}, while, the negative term refers to failing to reject $\mathcal{H}_0$. The false positive and false negative rates are equivalent to type I and type II errors, respectively.
Hence, a higher AUC indicates a better diagnostic ability of a binary test.  
It should be noted that since we consider $i_0/M=0.05$ to estimate the RB ratio, the values of $RB$ can vary between 0 and 20. Therefore, in computing the AUC for the semi-BNP test, the threshold should vary from 0 to 20. More details for plotting the ROC and computing the AUC are provided by Algorithm 3 in the supplementary material. The ROC curves and AUC values of the synthetic examples are provided in Figure \ref{ROC-variousSig} for the sample size $n=50$, $d=60$, $a=25$, and various values of the bandwidth parameter, including the median heuristic $\sigma_{MH}$. The red diagonal line represents the random classifier. A ROC curve located higher than the diagonal line indicates better test performance and vice versa. It is obvious from Figure \ref{ROC-variousSig} that 
the best test performance ($\operatorname{AUC}=1$) is achieved for the bandwidth parameter $80$. 

Another test of interest is to assess the effect of different hyperparameter settings for $a$ and $H$ through simulation studies to follow our proposed theoretical convergence results. To do this, we generate 100 $60$-dimensional samples of sizes $n=50$  from both $F_1=t_3(\mathbf{0}_{60},I_{60})$ and $F_2=N(\mathbf{0}_{60},I_{60})$ and represent the result of the semi-BNP test by Figure \ref{variousa} for two choices of the base measure $H$ ($H=F_2$ and $H=LG(\mathbf{0}_{60},I_{60})$) and various values of $a$ ($a=1,\ldots, 100
0$).
%%%%%%%%%%%%%%%%%%%%%%%%%%%
\begin{figure}[h!]
\centering
\subfloat[$\mathbf{y}_1,\ldots, \mathbf{y}_n\sim N(\mathbf{0.5}_{60},I_{60})$]{\includegraphics[width=.34\linewidth]{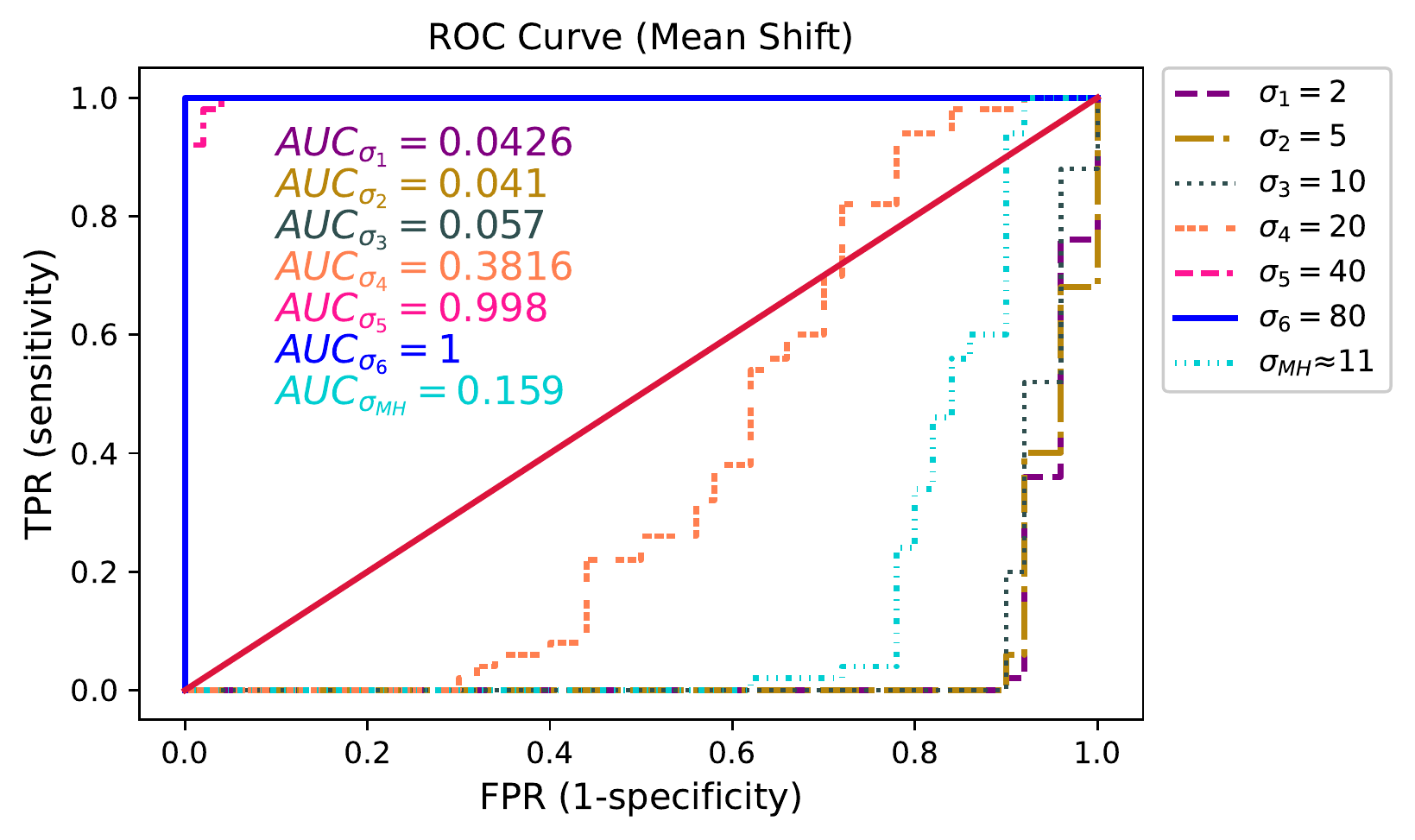}}
\subfloat[$\mathbf{y}_1,\ldots, \mathbf{y}_n\sim N(\mathbf{0}_{60},2I_{60})$]{\includegraphics[width=.34\linewidth]{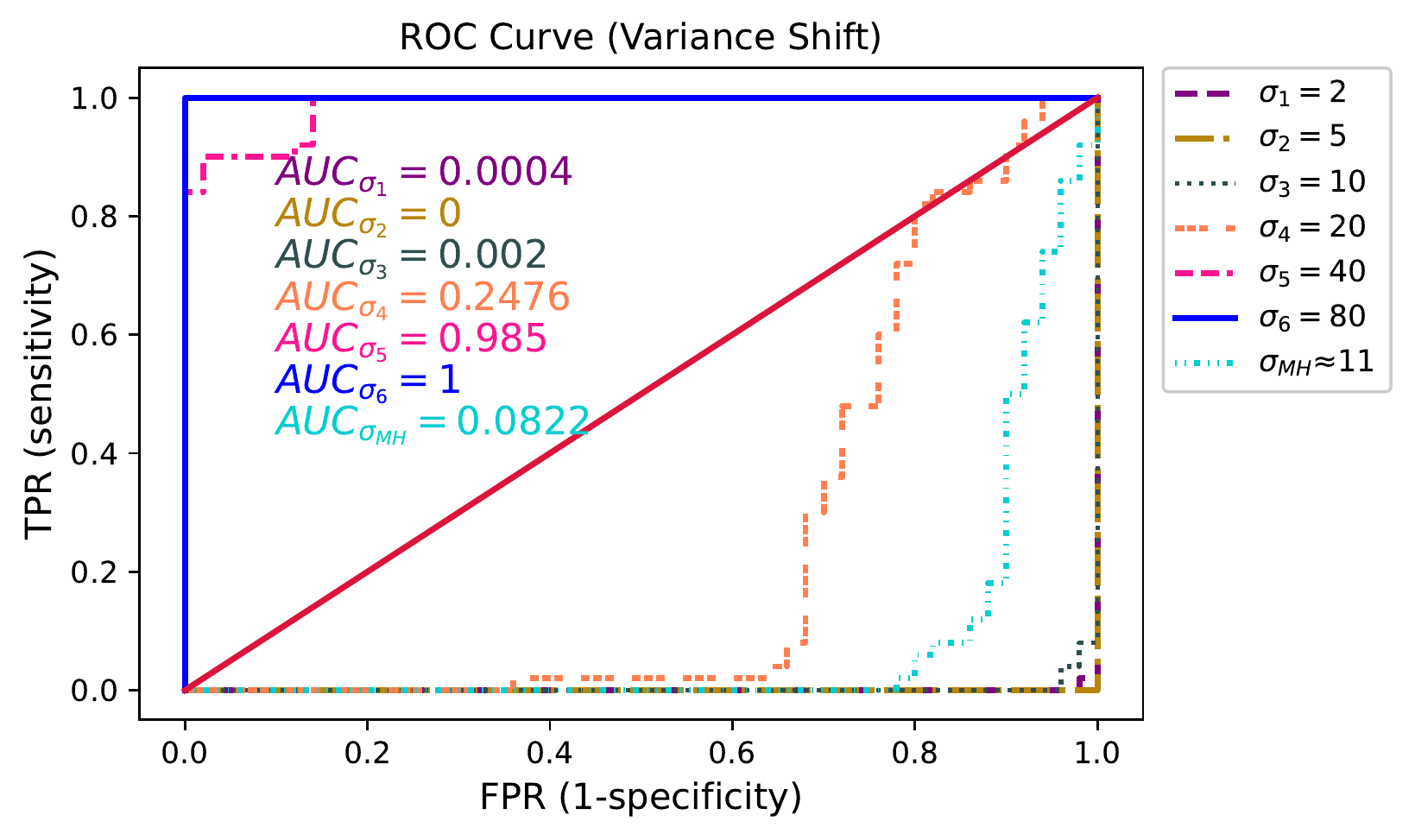}}
%\qquad
\subfloat[$\mathbf{y}_1,\ldots, \mathbf{y}_n\sim t_3(\mathbf{0}_{60},I_{60})$]{\includegraphics[width=.34\linewidth]{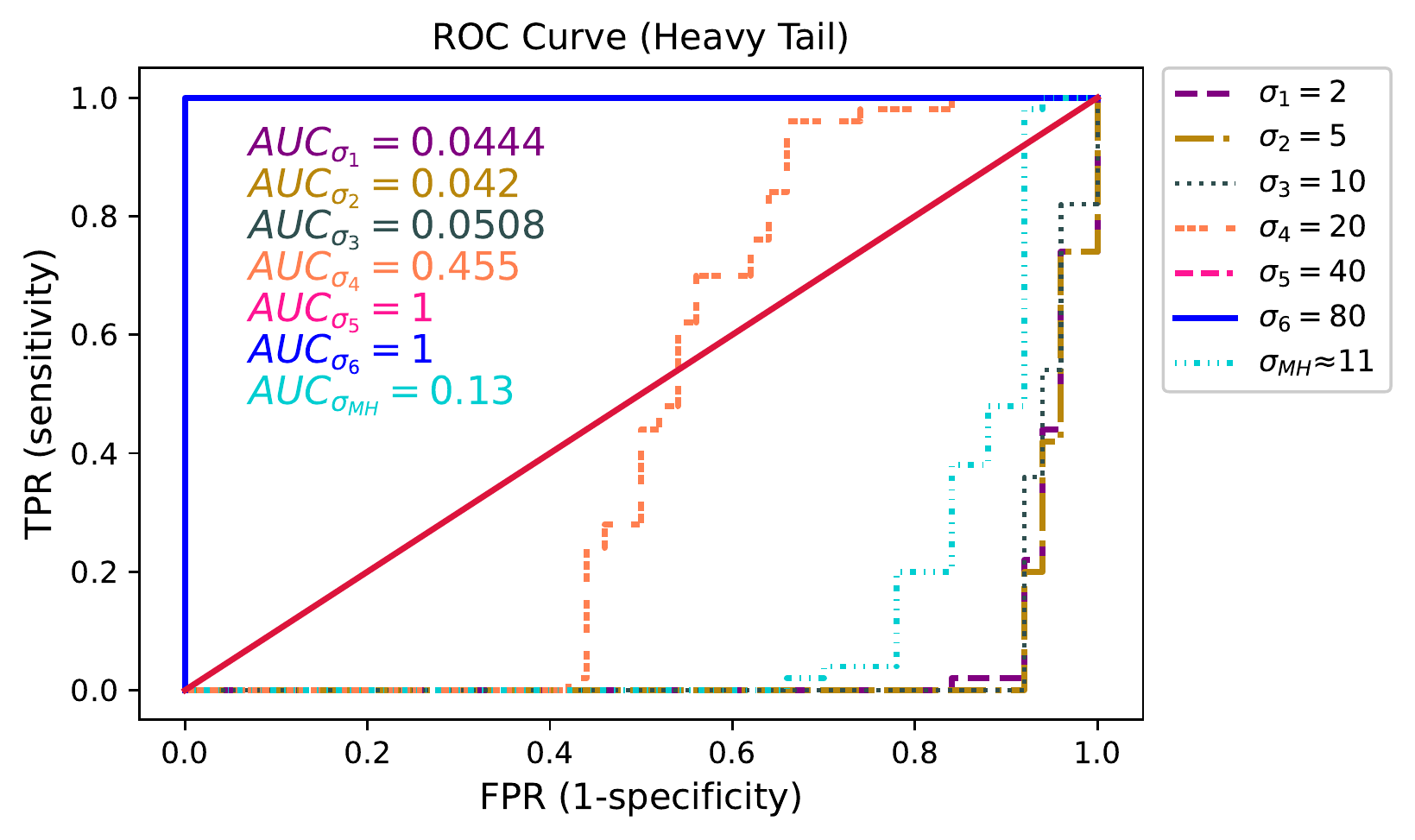}}
\qquad
\subfloat[$\mathbf{y}_1,\ldots, \mathbf{y}_n\sim 0.5N(-\mathbf{1}_{60},I_{60})+0.5N(\mathbf{1}_{60},I_{60})$]{\includegraphics[width=.34\linewidth]{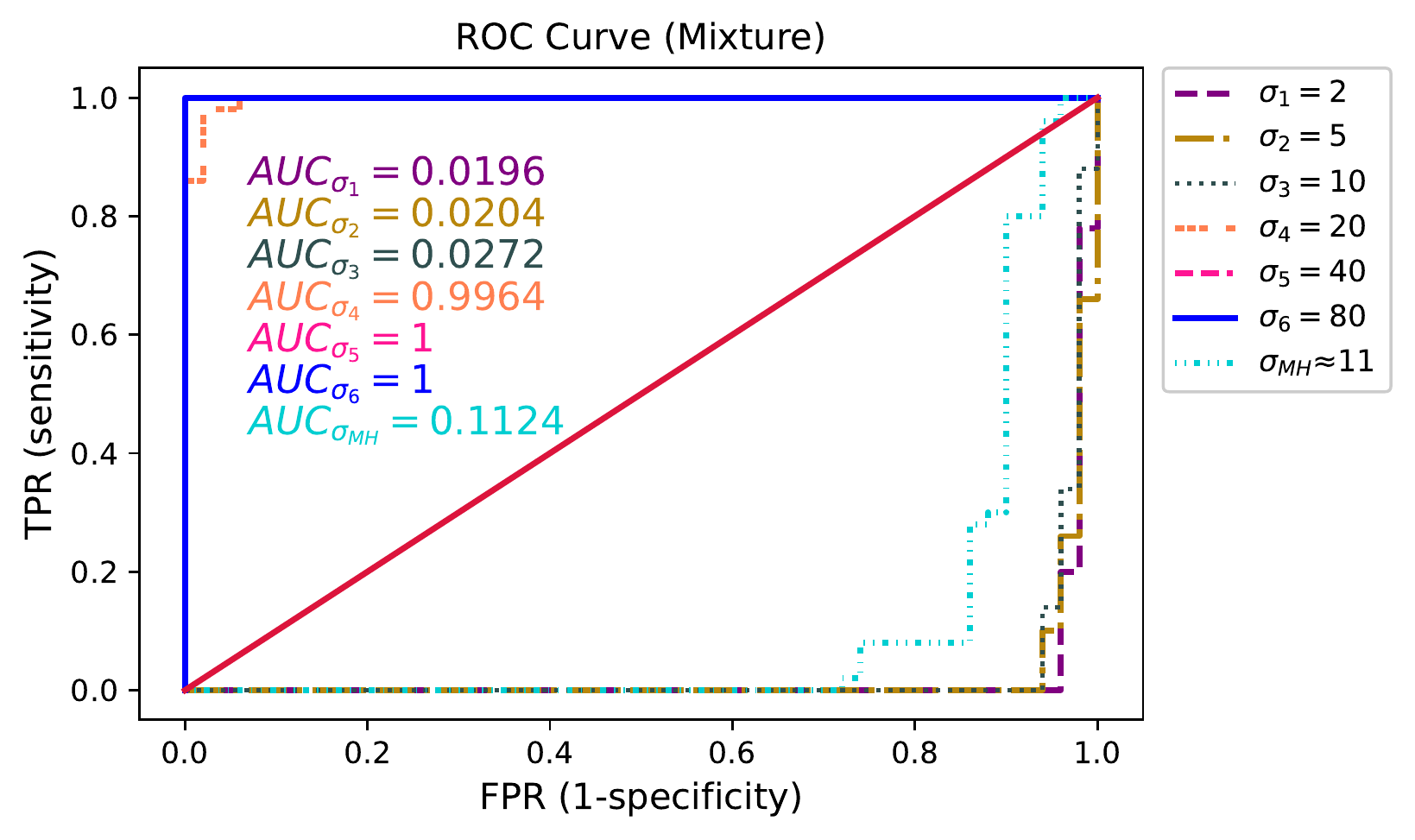}}
%\qquad
\subfloat[$\mathbf{y}_1,\ldots, \mathbf{y}_n\sim LN(\mathbf{0}_{60},B_{60})$]{\includegraphics[width=.34\linewidth]{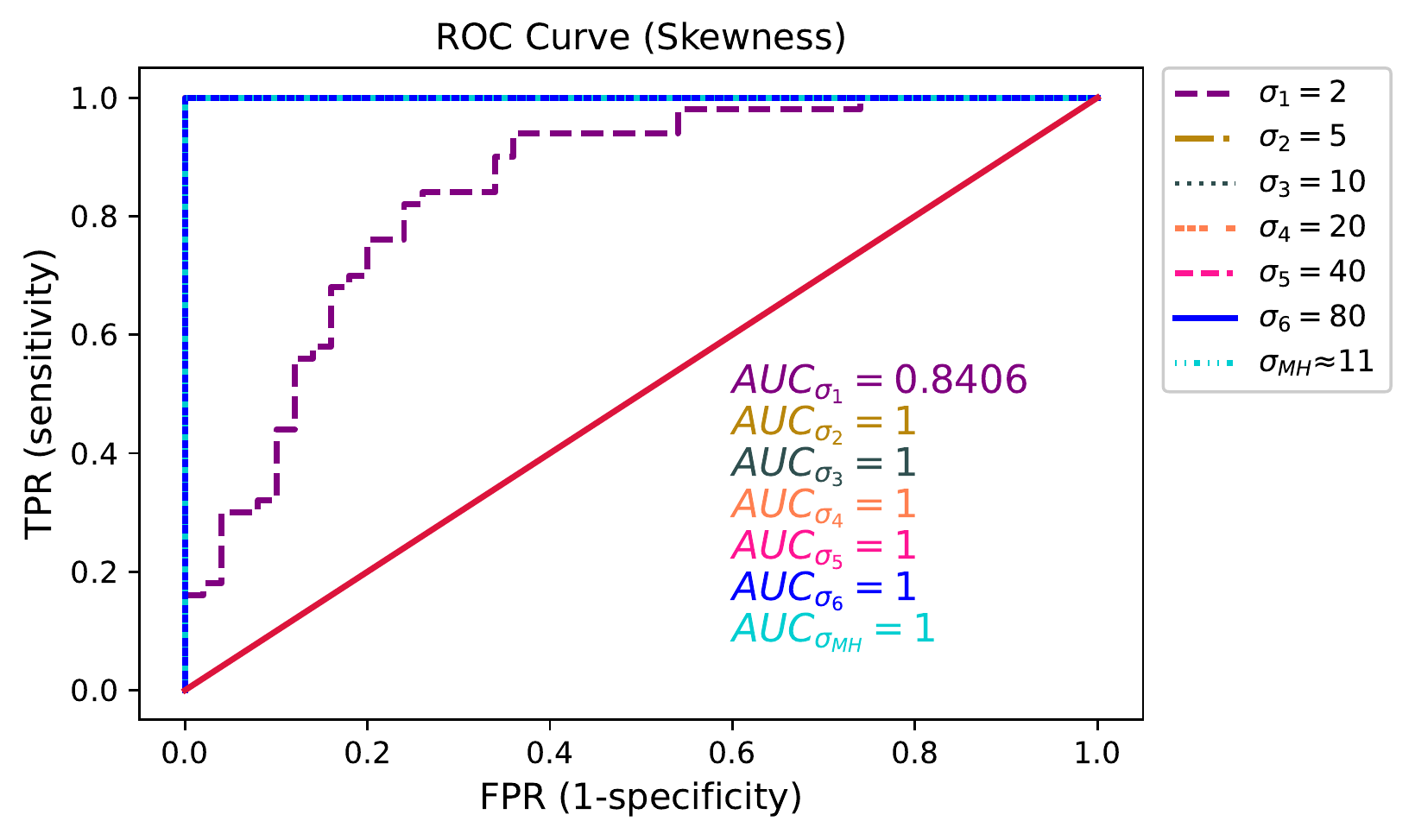}}
\subfloat[$\mathbf{y}_1,\ldots, \mathbf{y}_n\sim LG(\mathbf{0}_{60},I_{60})$]{\includegraphics[width=.34\linewidth]{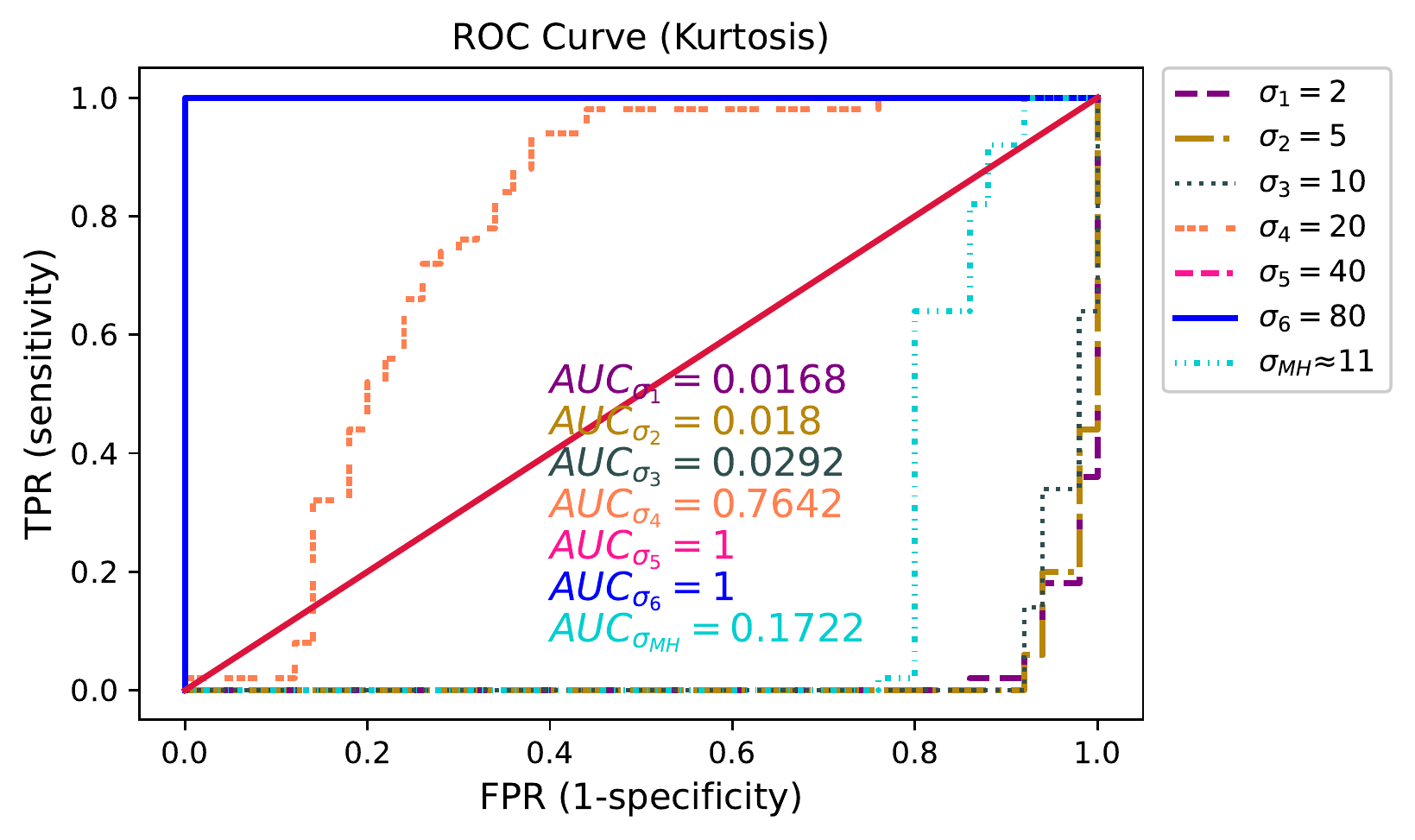}}
\caption{The ROC curves and AUC values of the BNP-MMD test for $\mathbf{x}_1,\ldots, \mathbf{x}_{n}\sim N(\mathbf{0}_{60},I_{60})$, using a range of bandwidth parameters including $\sigma=2,5,10,20,40,80$, as well as the median heuristic $\sigma_{MH}$.}
\label{ROC-variousSig}%
\end{figure}
%%%%%%%%%%%%%%%%%%%%%%%%%%%%
In this figure, the solid line represents the average of the RB and the filled area around the line indicates a $95\%$ confidence interval of the RB over the 100 samples. Figure \ref{variousa}-a clearly shows that by choosing $H\neq F_2$, the test wrongly accepts the null hypothesis. It is because the prior does not support the null hypothesis mentioned earlier when presenting the RB ratio in Section \ref{sec-BNP-test}. On the other hand, when $H= F_2$, Figure \ref{variousa}-b shows good performance for the test at $a=n/2$.
%less than $n$ except for some small values of $a$. However, it is better to consider $a$ at most $n/2$ to avoid any excessive influence of the prior on the test results.
Failing to reject $\mathcal{H}_0$ for small values of $a$ is due to the lack of sufficient support from the null hypothesis by the prior. We remark that the value of $a$ determines the concentration of the prior $F^{pri}$ around $H$, thus it is obvious that for small values of $a$, the test does not perform well. It should also be noted that for any choices of $H$ in Figure \ref{variousa}, the ability of the test to evaluate the null hypothesis is reduced by letting $a$ go to infinity, which can be concluded by Corollary \ref{cor-pos-consis}(i).
%%%%%%%%%%%%%%%%%%%%%%%%%%%
\begin{figure}[ht]
\centering
\subfloat[$H=LG(\mathbf{0}_{60},I_{60})$]{\includegraphics[width=.25\linewidth]{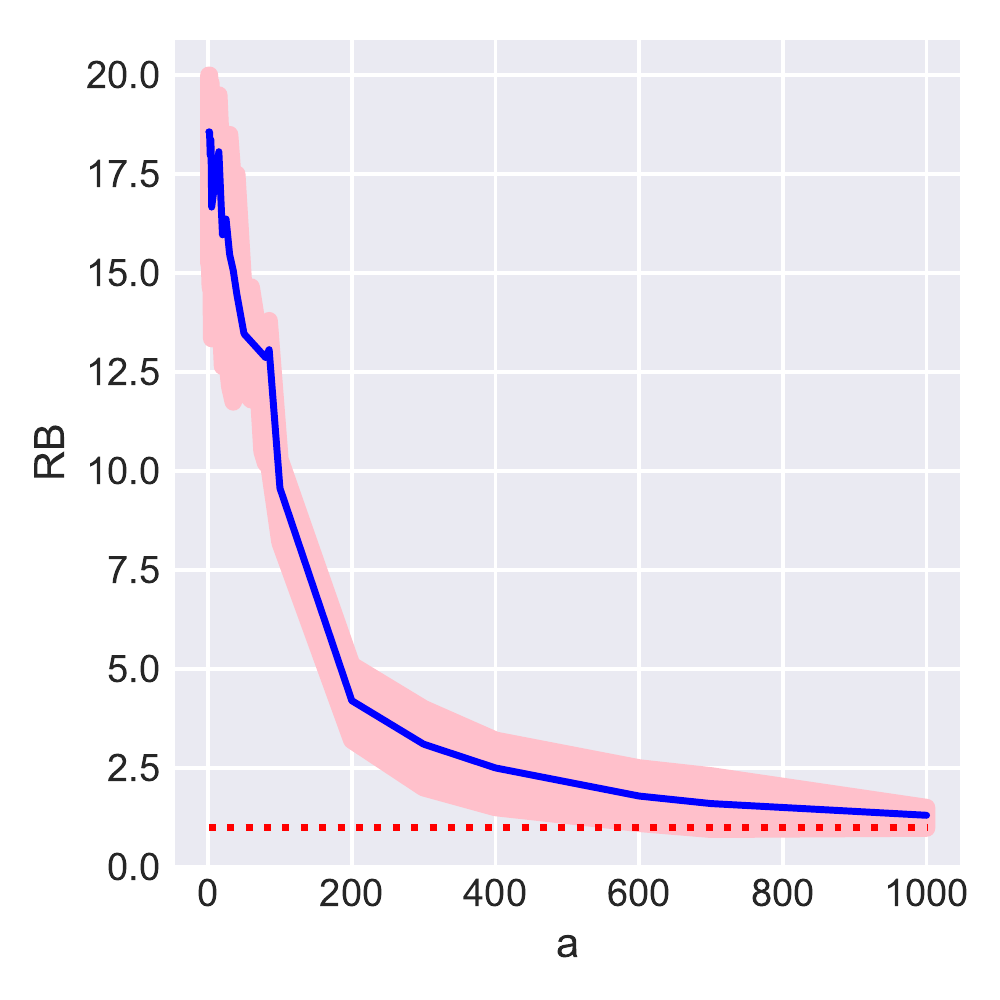}}
\hspace{1.5cm}
\subfloat[$H= F_2$]{\includegraphics[width=.25\linewidth]{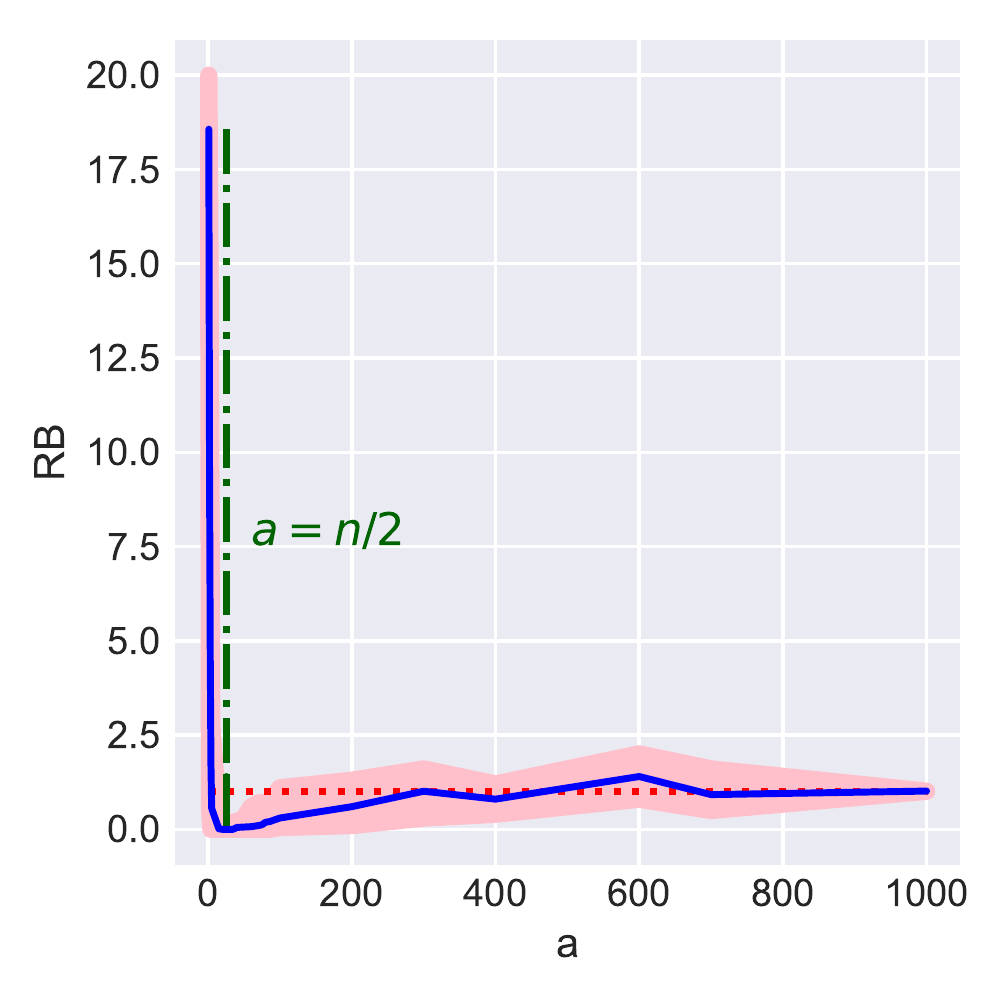}}
\caption{The solid line represents the average of the RB and the pink area represents a $95\%$ confidence interval of the RB over the 100 samples with various choices of $H$ and $a$ for the heavy tail example. The lower and upper bounds are the $2.5\%$ and $97.5\%$ quantiles of the RB, respectively. The red dotted line represents $RB=1$.}
\label{variousa}%
\end{figure}
%%%%%%%%%%%%%%%%%%%%%%%%%%%%

Now, to conduct a more comprehensive investigation, we present the average of RB and its relevant strength over the 100 samples in Table \ref{examples50} for $n=30, 50$. Furthermore, we present the results of the BNP-energy test by \cite{Al-Labadi} in Table \ref{examples50}, which demonstrate its weak performance in certain scenarios. Additional results in the power comparison can be found in Section 6.1 of the supplementary material.
%and Figure \ref{power-comparision} to compare the proposed method with its Bayesian competitor. Figure \ref{power-comparision} illustrates the proportion of rejecting $\mathcal{H}_0$ over the 100 samples for both Bayesian tests mentioned, across different data dimensions. The first row of Figure \ref{power-comparision} represents the type I error, while the remaining rows represent the test power. This Figure demonstrates the effectiveness of the semi-BNP kernel-based test in detecting differences, especially in scenarios involving variance shift, heavy tail, and kurtosis examples, where the BNP-energy test does not perform optimally in high sample sizes.  
%%%%%%%%%%%%%%%%%%%%%%%%%%%%
\begin{table}[htp] \centering
\setlength{\tabcolsep}{1 mm}
\caption{The average of RB, the average of its strength (Str
), and the relevant AUC
out of 100 replications based on using $a=25$, $\ell=1000$, $M=20$, $\epsilon=10^{-3}$ in \eqref{random-stopping}, and bandwidth parameter $\sigma=80$ in RBF kernel for two sample of data with $n=30,50$.}\label{examples50}
\scalebox{0.65}{
\begin{tabular}[c]{cl|llll|llll|llll|llll}\hline
\multirow{5}{*}{Example}&\multirow{5}{*}{$d$}& \multicolumn{8}{c} {BNP} & \multicolumn{8}{c} {FNP}  \\\cmidrule(lr){3-10}\cmidrule(lr){11-18}  
&&\multicolumn{4}{c}{MMD}&\multicolumn{4}{c}{Energy}&\multicolumn{4}{c}{MMD}&\multicolumn{4}{c}{Energy}\\\cmidrule(lr){3-6}\cmidrule(lr){7-10}\cmidrule(lr){11-14}\cmidrule(lr){15-18}
&&\multicolumn{2}{c}{RB(Str)}&\multicolumn{2}{c}{AUC}&\multicolumn{2}{c}{RB(Str)}&\multicolumn{2}{c}{AUC}&\multicolumn{2}{c}{P.value}&\multicolumn{2}{c}{AUC}&\multicolumn{2}{c}{P.value}&\multicolumn{2}{c}{AUC}\\\cmidrule(lr){3-4}\cmidrule(lr){5-6}\cmidrule(lr){7-8}\cmidrule(lr){9-10}\cmidrule(lr){11-12}\cmidrule(lr){13-14}\cmidrule(lr){15-16}\cmidrule(lr){17-18}  
&&\multicolumn{1}{c}{30}&\multicolumn{1}{c}{50}&\multicolumn{1}{c}{30}&\multicolumn{1}{c}{50}&\multicolumn{1}{c}{30}&\multicolumn{1}{c}{50}&\multicolumn{1}{c}{30}&\multicolumn{1}{c}{50}&\multicolumn{1}{c}{30}&\multicolumn{1}{c}{50}&\multicolumn{1}{c}{30}&\multicolumn{1}{c}{50}&\multicolumn{1}{c}{30}&\multicolumn{1}{c}{50}&\multicolumn{1}{c}{30}&\multicolumn{1}{c}{50}\\\hline
No differences&1&$2.08(0.62) $&$2.41(0.67) $  & \multicolumn{2}{!{\hspace*{-0.4pt}\tikzmark{start1}}c!{\tikzmark{end1}}}{} &$1.78(0.59) $&$1.91(0.55)$  & \multicolumn{2}{!{\hspace*{-0.4pt}\tikzmark{start2}}c!{\tikzmark{end2}}}{}&$ 0.50$&$0.45 $ & \multicolumn{2}{!{\hspace*{-0.4pt}\tikzmark{start3}}c!{\tikzmark{end3}}}{}&$0.50$& $ 0.49 $ &\multicolumn{2}{!{\hspace*{-0.4pt}\tikzmark{start4}}c!{\tikzmark{end4}}}{} \\
&5&$4.06(0.77) $&$ 6.91(0.76)$ &  \multicolumn{2}{!{\hspace*{-0.4pt}\tikzmark{start11}}c!{\tikzmark{end11}}}{} &$3.46(0.65) $ &$5.99(0.73) $ & \multicolumn{2}{!{\hspace*{-0.4pt}\tikzmark{start22}}c!{\tikzmark{end22}}}{} &$0.48 $&$0.50 $ &\multicolumn{2}{!{\hspace*{-0.4pt}\tikzmark{start33}}c!{\tikzmark{end33}}}{}&$0.54 $ &$0.52 $ &\multicolumn{2}{!{\hspace*{-0.4pt}\tikzmark{start44}}c!{\tikzmark{end44}}}{} \\
&10&$6.21(0.78) $&$10.74(0.79) $ & \multicolumn{2}{!{\hspace*{-0.4pt}\tikzmark{start111}}c!{\tikzmark{end111}}}{} &$5.92(0.67) $ &$10.42(0.76) $ & \multicolumn{2}{!{\hspace*{-0.4pt}\tikzmark{start222}}c!{\tikzmark{end222}}}{} &$0.50 $&$0.51 $ &  \multicolumn{2}{!{\hspace*{-0.4pt}\tikzmark{start333}}c!{\tikzmark{end333}}}{} &$0.54 $&$0.47 $ &  \multicolumn{2}{!{\hspace*{-0.4pt}\tikzmark{start444}}c!{\tikzmark{end444}}}{} \\
&20&$9.62(0.80) $&$16.02(0.83) $ &\multicolumn{2}{!{\hspace*{-0.4pt}\tikzmark{start1111}}c!{\tikzmark{end1111}}}{} &$8.24(0.73) $ &$14.76(0.78) $ & \multicolumn{2}{!{\hspace*{-0.4pt}\tikzmark{start2222}}c!{\tikzmark{end2222}}}{} &$0.46 $&$0.50 $& \multicolumn{2}{!{\hspace*{-0.4pt}\tikzmark{start3333}}c!{\tikzmark{end3333}}}{} &$0.51 $&$0.50 $& \multicolumn{2}{!{\hspace*{-0.4pt}\tikzmark{start4444}}c!{\tikzmark{end4444}}}{} \\
&40&$13.07(0.88) $&$18.85(0.97) $ & \multicolumn{2}{!{\hspace*{-0.4pt}\tikzmark{start11111}}c!{\tikzmark{end11111}}}{} &$11.56(0.75) $ &$17.58(0.84) $ & \multicolumn{2}{!{\hspace*{-0.4pt}\tikzmark{start22222}}c!{\tikzmark{end22222}}}{} &$0.51 $&$0.49 $ & \multicolumn{2}{!{\hspace*{-0.4pt}\tikzmark{start33333}}c!{\tikzmark{end33333}}}{}&$0.53 $ &$0.46 $ &\multicolumn{2}{!{\hspace*{-0.4pt}\tikzmark{start44444}}c!{\tikzmark{end44444}}}{} \\ 
&60&$14.09(0.87) $&$19.71(1)$&\multicolumn{2}{!{\hspace*{-0.4pt}\tikzmark{start111111}}c!{\tikzmark{end111111}}}{} &$13.38(0.81)$&$18.51(0.93) $&\multicolumn{2}{!{\hspace*{-0.4pt}\tikzmark{start222222}}c!{\tikzmark{end222222}}}{}&$0.52 $ &$0.46 $ &\multicolumn{2}{!{\hspace*{-0.4pt}\tikzmark{start333333}}c!{\tikzmark{end333333}}}{} &$0.50 $&$0.48 $& \multicolumn{2}{!{\hspace*{-0.4pt}\tikzmark{start444444}}c!{\tikzmark{end444444}}}{} \\ 
&80&$ 15.2(0.89)$&$19.57(1) $&\multicolumn{2}{!{\hspace*{-0.4pt}\tikzmark{start1111111}}c!{\tikzmark{end1111111}}}{} &$14.16(0.87) $ &$19.10(1) $ &\multicolumn{2}{!{\hspace*{-0.4pt}\tikzmark{start2222222}}c!{\tikzmark{end2222222}}}{} &$0.46 $&$0.47 $  &\multicolumn{2}{!{\hspace*{-0.4pt}\tikzmark{start3333333}}c!{\tikzmark{end3333333}}}{} &$0.53 $&$0.56 $ &\multicolumn{2}{!{\hspace*{-0.4pt}\tikzmark{start4444444}}c!{\tikzmark{end4444444}}}{} \\
&100&$15.83(0.91) $&$19.74(1) $ &\multicolumn{2}{!{\hspace*{-0.4pt}\tikzmark{start11111111}}c!{\tikzmark{end11111111}}}{} &$14.84(0.92) $ &$19.31(1) $ &\multicolumn{2}{!{\hspace*{-0.4pt}\tikzmark{start22222222}}c!{\tikzmark{end22222222}}}{} &$0.48 $&$0.46 $ &\multicolumn{2}{!{\hspace*{-0.4pt}\tikzmark{start33333333}}c!{\tikzmark{end33333333}}}{} &$0.49 $&$0.55 $&\multicolumn{2}{!{\hspace*{-0.4pt}\tikzmark{start44444444}}c!{\tikzmark{end44444444}}}{} \\\hline
Mean shift&1&$0.76(0.24) $&$0.40(0.09) $&$0.82 $ &$0.96 $ &$0.67(0.21) $ &$0.45(0.11) $ &$0.87 $&$0.90 $ &$ 0.15$&$0.05 $ &$ 0.86$&$0.91 $ &$ 0.19$&$0.12 $&$0.79 $ &$0.86 $  \\
&5&$0.21(0.03) $&$0.07(0) $ &$0.99$&$0.99 $ &$0.28(0.04) $ &$0.09(0.01) $ &$0.98 $&$1 $ &$0.01 $&$0.002 $&$ 1$ &$0.98 $&$0.02 $ &$0.004 $& $0.97 $&$0.97 $  \\
&10&$0.09(0.01) $&$0.05(0) $ &$1 $ &$1 $ & $0.17(0.05) $& $0.02(0) $& $0.98 $&$1 $&$0.001 $&$0.001 $& $1 $& $1 $&$ 0.006$&$0.004 $&$ 0.98$ &$1 $  \\
&20&$ 0.09(0.01)$&$0(0) $ &$ 1$ &$1 $ &$0.09(0.01) $ &$0(0) $ &$1 $&$1 $ &$0.001 $&$0.001 $&$1 $ &$1 $&$0.004 $ &$0.004 $&$1 $ &$1 $  \\
&40&$0.08(0) $&$0(0) $ &$1 $ &$1 $ &$0.06(0.02) $ & $0(0) $&$1 $&1 &$ 0.001$& $0.001 $&$ 1$& $1 $&$ 0.004$& $0.004 $&$1 $&$1 $  \\ 
&60&$0.09(0.03) $&$0(0)$ &$1 $ &$1$ &$0.07(0.04) $&$0(0)$ &$1 $ &$1 $ &$ 0.001$&$0.001 $&$ 1$ &$1 $& $ 0.004$&$0.004 $&$ 1$ &$ 1$  \\ 
&80&$ 0.06(0.02)$&$0(0) $ &$ 1$ &$1 $ &$0.05(0.03) $ & $0(0) $&$1 $ &1&$ 0.001$ &$0.001 $&$ 1$& $1 $&$ 0.004$& $0.004 $&$1 $ &$1 $ \\
&100&$ 0.04(0.01)$&$0(0) $ &$1 $ &$1 $ &$0.03(0) $ &$0(0) $ &$1 $ &1&$0.001 $& $0.001 $&$ 1$&$1 $ &$ 0.004$&$0.004 $&$ 1$ &$1 $  \\\hline
Skewness&1&$0.01(0) $&$0(0) $ &$0.99 $ &$1 $ &$0.07(0) $ &$0(0) $ &$0.99 $&1 &$0.009 $&$0.001 $&$0.98 $ &$1 $ &$0.007 $&$0.004 $ &$0.94 $&$1 $  \\
&5&$0(0) $&$0(0) $ &$1 $ &$1 $ &$0(0) $ &$0(0) $ &$1 $&1&$0.001 $ &$0.001 $ &$1$&$1 $ &$0.004 $&$0.004 $ &$1 $&$1 $  \\
&10&$0(0) $&$0(0) $ &$1 $ &$1 $ &$0(0) $ &$0(0) $ &$1 $ &1&$0.001$&$0.001 $&$1$ &$1 $ &$0.004$&$ 0.004$&$1$ &$1 $  \\
&20&$0(0)$&$0(0) $ &$1 $ &$1 $ &$0(0) $ &$0(0) $ &$1 $ &1&$0.001$&$0.001 $ &$1$&$1 $&$0.004$ &$0.004 $ &$1$&$1 $  \\
&40&$0(0)$&$0(0) $ &$ 1$ &$1 $ &$0(0) $ &$0(0) $ &$1 $ &1&$0.001$&$0.001 $ &$1$&$ 1$ &$0.004$&$ 0.004$ &$1$&$1 $  \\ 
&60&$0(0)$&$0(0)$ &$1 $ &$1$ &$0(0)$&$0(0)$ &$1 $ &$1 $&$0.001$ &$0.001 $ &$1$&$1 $ &$0.004$&$0.004 $ &$1$&$1 $  \\ 
&80&$0(0)$&$0(0) $ &$1 $ &$1 $ &$0(0) $ &$0(0) $ &$1 $&1&$0.001$ &$0.001 $ &$1$&$1 $ &$0.004$&$0.004 $&$1$ &$1 $  \\
&100&$0(0)$&$0(0) $ &$1 $ &$1 $ &$ 0(0)$ &$0(0) $ &$1 $&$1 $&$0.001$ &$0.001 $&$1$ &$1 $&$0.004$ &$0.004 $&$1$ &$1 $  \\\hline
Mixture&1&$0.06(0)$&$0(0) $ &$ 0.90$ &$0.97 $ &$0.19(0.03) $ &$0.04(0) $ &$0.97 $&$1 $&$0.43$ &$0.38 $&$0.58$ &$0.57 $&$0.29$ &$0.17 $&$0.69$ &$0.81 $  \\
&5&$0(0)$&$0(0) $ & $1 $& $1 $&$0(0) $ &$0(0) $ &$1 $ &$1 $&$0.15$&$0.09 $ &$0.84$&$0.91 $ &$0.06$&$0.01 $ &$0.95$&$1 $  \\
&10&$0(0)$&$0(0) $ &$ 1$ &$1 $ &$0(0) $ &$0(0) $ &$1 $&$1 $&$0.03$ &$ 0.007$ &$0.95$&$0.98 $ &$0.02$&$0.007 $&$0.96$ &$1 $ \\
&20&$0(0)$&$0(0) $ &$1 $ &$1 $ &$0(0) $ &$0(0) $ &$1 $&$1 $&$0.002$ &$0.001 $&$0.96$ &$1 $&$0.01$ &$0.006 $ &$1$&$1 $\\
&40&$0(0)$&$0(0) $ &$ 1$ &$1 $ &$0(0) $ &$ 0(0)$ &$1 $&$1 $&$0.001$ &$0.001 $& $1$&$1 $&$0.01$ &$0.006 $&$1$ &$1 $\\ 
&60&$0(0)$&$0(0)$ &$ 1$ &$1$ &$0(0)$&$0(0)$ &$1 $ &$1 $ &$0.001$&$0.001 $ &$1$&$1 $ &$0.006$&$0.009 $&$1$ &$1 $\\
&80&$0(0)$&$0(0) $ &$1 $ &$1 $ &$0(0) $ &$0(0) $ &$1 $&$1 $&$0.001$ &$0.001 $& $1$&$1 $ &$0.008$&$0.006 $& $1$&$1 $\\
&100&$0(0)$&$0(0) $ & $1 $&$ 1$ &$0(0) $ &$0(0) $ &$1 $&$1 $& $0.001$&$0.001 $ &$1$&$1 $ &$0.004$&$0.006 $&$1$ &$1 $\\\hline
Variance shift&1&$0.87(0.29)$& $0.85(0.19) $&$0.71 $ &$0.83 $ &$1.10(0.36) $ &$1.08(0.33) $ &$0.53 $&$0.63$ &$0.46$&$0.38 $ &$0.54$&$0.57 $ &$0.33$&$0.21 $ &$0.65$&$0.77 $\\ 
&5&$0.55(0.12)$&$0.56(0.15) $ &$0.99 $ &$0.99 $ &$1.06(0.35) $ &$0.99(0.32) $ &$ 0.89$&$0.98$&$0.34$ &$0.20 $ &$0.65$&$0.80 $ &$0.20$&$0.07 $ &$0.82$&$0.93 $\\
&10&$0.44(0.11)$&$0.27(0.05) $ &$ 0.99$ &$1 $ &$0.87(0.24) $ &$0.80(0.25) $ &$0.97 $ &$1$&$0.14$&$0.03 $ &$0.85$&$0.97 $&$0.10$ &$0.02 $&$0.89$ &$0.97 $\\
&20&$0.34(0.07)$&$0.08(0) $ &$1 $ &$1 $ &$0.65(0.17) $ &$0.60(0.13) $ &$0.99 $ &$1$&$0.01$&$0.001 $ &$0.95$&$1 $&$0.03$ &$0.006 $ &$0.95$&$1 $\\
&40&$0.13(0.01)$&$0.02(0) $ &$ 1$ &$1 $ &$0.61(0.18) $ &$0.58(0.14) $ &$1 $&1&$ 0.001$ &$0.001 $&$1$ &$1 $ &$0.01 $&$0.004 $ &$0.98$&$1 $\\ 
&60&$0.12(0.01)$&$0.01(0)$ &$1 $ &$1$ &$0.47(0.10)$&$0.45(0.11)$ &$1 $ &$1 $ &$0.001$&$0.001 $ &$1$&$1 $ &$0.006$&$0.004 $&$1$ &$1 $\\
&80&$0.17(0.01)$&$0(0) $ &$ 1$ &$1 $ &$0.54(0.12) $ &$0.47(0.11) $ &$1 $&1& $0.001$&$0.001 $ &$1$&$1 $ &$0.005$&$0.004 $ &$1$&$1 $\\
&100&$0.14(0.01)$& $0(0) $&$ 1$ &$1 $ &$0.45(0.10) $ &$0.41(0.08) $ &$1 $&1&$0.001$ &$0.001 $& $1$&$1 $ &$0.004$&$0.004 $ &$1$&$1 $ \\\hline
Heavy tail&1&$0.93(0.28)$&$0.66(0.20) $ &$0.89 $ &$0.92 $ &$1.19(0.41) $ &$1.10(0.38) $ &$0.70 $ &$0.78$&$0.43$&$0.39 $ &$0.57$&$0.56 $ &$0.39$&$0.36 $ &$0.59$&$0.62 $\\
&5&$0.32(0.06)$&$0.37(0.08) $ &$0.99 $ &$0.99 $ &$0.77(0.24) $ &$0.78(0.23) $ &$0.93 $&$0.99$& $0.20$&$0.11 $ &$0.79$&$0.89 $ &$0.03$&$0.006 $&$0.97$ &$0.99$\\
&10&$0.35(0.08)$&$0.13(0.02) $ &$0.99 $ &$1 $ &$0.61(0.16) $ &$0.68(0.19) $ &$0.98 $&1&$0.06$ &$0.007 $&$0.92$ &$0.98 $& $0.09$&$0.01 $&$0.90$ &$0.97 $\\
&20&$0.15(0.02)$&$0(0) $ &$ 1$ &$1 $ &$0.48(0.12) $ &$0.46(0.12) $ &$1 $&$1 $&$0.002$ &$0.001 $&$0.96$ &$1 $ &$0.02$&$0.005 $&$0.96$ &$1 $\\
&40&$0.07(0.01)$&$0(0) $ &$1 $ &$1 $ &$0.25(0.04) $ &$0.18(0.04)$ &$1 $ &1&$0.001$&$0.001 $ &$1$&$1 $ &$0.005$&$0.004 $&$1$ &$1 $\\ 
&60&$0.02(0)$&$0(0)$ &$1 $ &$1$ &$0.22(0.03)$&$0.14(0.01)$ &$1 $ &$1 $ &$0.001$&$0.001 $ &$1$&$1 $ &$0.004$&$0.004 $&$1$ &$1 $\\
&80&$0.01(0)$&$0(0) $ &$1 $ &$1 $ &$0.13(0.01) $ &$0.15(0.02) $ &$1 $ &1&$0.001$&$0.001 $ &$1$&$1 $ &$0.004$&$0.004 $ &$1$&$1 $\\
&100&$0.04(0)$&$0(0) $ & $ 1$&$1 $ &$0.14(0.01) $ &$0.09(0.01) $ &$1 $&1& $0.001$&$0.001 $&$1$ &$1 $& $0.004$&$0.004 $&$1$ &$1 $\\\hline
Kurtosis&1&$0.47(0.12) $&$0.19(0.04) $ &$0.89 $ &$0.98 $ &$1.09(0.37) $ &$0.88(0.28) $ &$0.77 $ &$0.90$&$0.28 $&$0.23 $ &$ 0.74$&$0.72 $ &$0.18 $&$0.11 $ &$0.79 $&$0.88 $\\
&5&$0.16(0.03) $&$0.06(0.01) $ &$1 $ &$1 $ &$0.63(0.18) $ &$0.41(0.09) $ &$0.96 $ &$0.99$&$0.04 $&$0.01 $ &$0.94 $&$0.98 $ &$0.03 $&$0.008 $&$0.97 $ &$0.96 $\\
&10&$0.02(0) $&$0(0) $ &$1 $ &$1 $ &$0.35(0.08) $ &$0.32(0.06) $ &$0.97 $ &1&$ 0.001$&$0.001 $ &$ 1$&$1 $ &$0.007 $&$0.004 $ &$ 0.96$&$1 $\\
&20&$0(0) $&$0(0) $ &$ 1$ &$1 $ &$0.20(0.03) $ &$0.18(0.02) $ &$1 $&1&$ 0.001$ &$0.001 $ &$ 1$&$1 $ &$ 0.004$&$0.004 $ &$ 1$&$1 $\\
&40&$ 0(0)$&$0(0) $ &$ 1$ &$1 $ &$0.06(0.01) $ &$0.06(0) $ &$1 $ &1&$ 0.001$&$0.001 $ &$ 1$&$1 $ &$ 0.004$&$0.004 $ &$ 1$&$1 $\\ 
&60&$ 0(0)$&$0(0)$ &$ 1$ &$1$ &$0.05(0)$&$0.04(0) $ &$1 $ &$1 $ &$ 0.001$&$0.001 $ &$ 1$&$1 $ &$ 0.004$&$0.004 $ &$ 1$&$1 $\\
&80&$ 0(0)$&$0(0) $ &$1 $ &$1 $ &$0.05(0) $ &$0.03(0) $ &$1 $ &1&$ 0.001$&$0.001 $ &$ 1$&$ 1$ &$ 0.004$&$0.004 $ &$ 1$&$1 $\\
&100&$ 0(0)$&$0(0) $ &$ 1$ &$1 $ &$0.02(0) $ &$0(0) $ &$1 $ &1&$ 0.001$&$0.001 $ &$ 1$&$1 $ &$ 0.004$&$0.004 $ &$ 1$&$1 $\\
\\\bottomrule
\end{tabular}
\HatchedCell{start1}{end1}{%
pattern color=black!70,pattern=north east lines}
\HatchedCell{start2}{end2}{%
pattern color=black!70,pattern=north east lines}
\HatchedCell{start3}{end3}{%
pattern color=black!70,pattern=north east lines}
\HatchedCell{start4}{end4}{%
pattern color=black!70,pattern=north east lines}
%%%%%%%%%%%%%%%
\HatchedCell{start11}{end11}{%
pattern color=black!70,pattern=north east lines}
\HatchedCell{start22}{end22}{%
pattern color=black!70,pattern=north east lines}
\HatchedCell{start33}{end33}{%
pattern color=black!70,pattern=north east lines}
\HatchedCell{start44}{end44}{%
pattern color=black!70,pattern=north east lines}
%%%%%%%%%%%%%%%%%
\HatchedCell{start111}{end111}{%
pattern color=black!70,pattern=north east lines}
\HatchedCell{start222}{end222}{%
pattern color=black!70,pattern=north east lines}
\HatchedCell{start333}{end333}{%
pattern color=black!70,pattern=north east lines}
\HatchedCell{start444}{end444}{%
pattern color=black!70,pattern=north east lines}
%%%%%%%%%%%%%%%
\HatchedCell{start1111}{end1111}{%
pattern color=black!70,pattern=north east lines}
\HatchedCell{start2222}{end2222}{%
pattern color=black!70,pattern=north east lines}
\HatchedCell{start3333}{end3333}{%
pattern color=black!70,pattern=north east lines}
\HatchedCell{start4444}{end4444}{%
pattern color=black!70,pattern=north east lines}
%%%%%%%%%%%%%%%
\HatchedCell{start11111}{end11111}{%
pattern color=black!70,pattern=north east lines}
\HatchedCell{start22222}{end22222}{%
pattern color=black!70,pattern=north east lines}
\HatchedCell{start33333}{end33333}{%
pattern color=black!70,pattern=north east lines}
\HatchedCell{start44444}{end44444}{%
pattern color=black!70,pattern=north east lines}
%%%%%%%%%%%%%%
\HatchedCell{start111111}{end111111}{%
pattern color=black!70,pattern=north east lines}
\HatchedCell{start222222}{end222222}{%
pattern color=black!70,pattern=north east lines}
\HatchedCell{start333333}{end333333}{%
pattern color=black!70,pattern=north east lines}
\HatchedCell{start444444}{end444444}{%
pattern color=black!70,pattern=north east lines}
%%%%%%%%%%%%%%
\HatchedCell{start1111111}{end1111111}{%
pattern color=black!70,pattern=north east lines}
\HatchedCell{start2222222}{end2222222}{%
pattern color=black!70,pattern=north east lines}
\HatchedCell{start3333333}{end3333333}{%
pattern color=black!70,pattern=north east lines}
\HatchedCell{start4444444}{end4444444}{%
pattern color=black!70,pattern=north east lines}
%%%%%%%%%%%%%%%%%
\HatchedCell{start11111111}{end11111111}{%
pattern color=black!70,pattern=north east lines}
\HatchedCell{start22222222}{end22222222}{%
pattern color=black!70,pattern=north east lines}
\HatchedCell{start33333333}{end33333333}{%
pattern color=black!70,pattern=north east lines}
\HatchedCell{start44444444}{end44444444}{%
pattern color=black!70,pattern=north east lines}
}
\end{table}%

To compare the BNP and FNP tests, the $p$-values of the frequentists counterparts corresponding to each Bayesian test are presented in Table \ref{examples50} using $\mathrm{R}$ packages \textbf{energy}\footnote{\url{https://CRAN.R-project.org/package=energy}} and \textbf{maotai}\footnote{\url{https://CRAN.R-project.org/package=maotai}}. AUC values of all tests are also given to facilitate comparison between tests.
Generally, the proposed test reflects better performances than its frequentist counterparts in lower dimensions. For instance, in the variance shift example, when $d=5$ and $ n=30$, the average of the $RB$ and its strength for the semi-BNP-MMD test are $0.55$ and $0.12$, respectively, which shows strong evidence to reject the null. While the average of the $p$-value corresponding to the MMD frequentist test is $0.34$, which shows a failure to reject the null hypothesis. The AUC value of the semi-BNP test is also $0.99$ which indicates a better ability than its frequentist counterpart with an AUC of $0.65$. To examine the large sample property, additional results for $n=500,1000$ are presented in Section 6.1 of the supplementary material, revealing the relatively poor performance of the BNP-Energy test in comparison to other tests.

\subsection{The Semi-BNP GAN}
According to the results reported in the previous subsection, the semi-BNP estimator suggests a test that outperforms other competing tests in many scenarios. Therefore, we expect that embedding this estimator in GANs as the discriminator makes an accurate comparison to distinguish real and fake data. We use the database of handwritten digits with 10 modes, bone marrow biopsy histopathology, human faces, and brain MRI images to analyze the model performance. 
Following \citet{Li}, we consider the Gaussian neural network for the generator with four hidden layers each having rectified linear units activation function and a sigmoid function for the output layer. There are numerous methods to choose network parameters. We adopt the Bayesian optimization method, used in \citet{Li}, to determine the number of nodes in hidden layers and tuning parameters of the network thanks to its good performance. We also set mini-batch sizes of $n_{mb}=1,000$ and a mixture of six Gaussian kernels corresponding to the bandwidth parameters $2, 5, 10, 20, 40,$ and $ 80$ to train networks discussed in this section. 
\subsubsection{MNIST Dataset \citep{lecun1998mnist}:}
The MNIST dataset includes 60,000 handwritten digits of 10 numbers from 0 to 9 each having 784 ($28\times28$) dimensions. This dataset is split into 50000 training and 10000 testing images and is a good example to demonstrate the performance of the method in dealing with the mode collapse problem. We use the training set to train the network.
A sample from the training MNIST dataset is shown in Figure \ref{mnist-pos-fre}-a. Following $r_{mb}=40,000$ iterations, we generate samples from the trained semi-BNP GAN using Algorithm 2 from the supplementary material, as depicted in Figure \ref{mnist-pos-fre}-b. %A sample generated by the frequentist counterpart of our model, presented by \citet{Li}, is also given in Figure \ref{mnist-pos-fre}-c. 
The results of \citet{Li}\footnote{The implementation codes for the GAN proposed by \citet{Li} is available at \url{https://www.dropbox.com/s/anf9z1zyqi7379n/Generative-Moment-Matching-Networks-master.zip?file_subpath=\%2FREADME.md} } are also presented by Figure \ref{mnist-pos-fre}-c as the frequentist counterpart of our semi-BNP procedure. Based on these preliminary results, we can see that our generated images can, at least, replicate the results of \citet{Li} and in some cases produce sharper images. This result can also be deduced from the presented values of certain score functions in Section 6.2 of the supplementary material.
%%%%%%%%%%%%%%%%%%%%%%%%%%%
\begin{figure}[ht]
\centering
\subfloat[Training data]{\includegraphics[width=.24\linewidth]{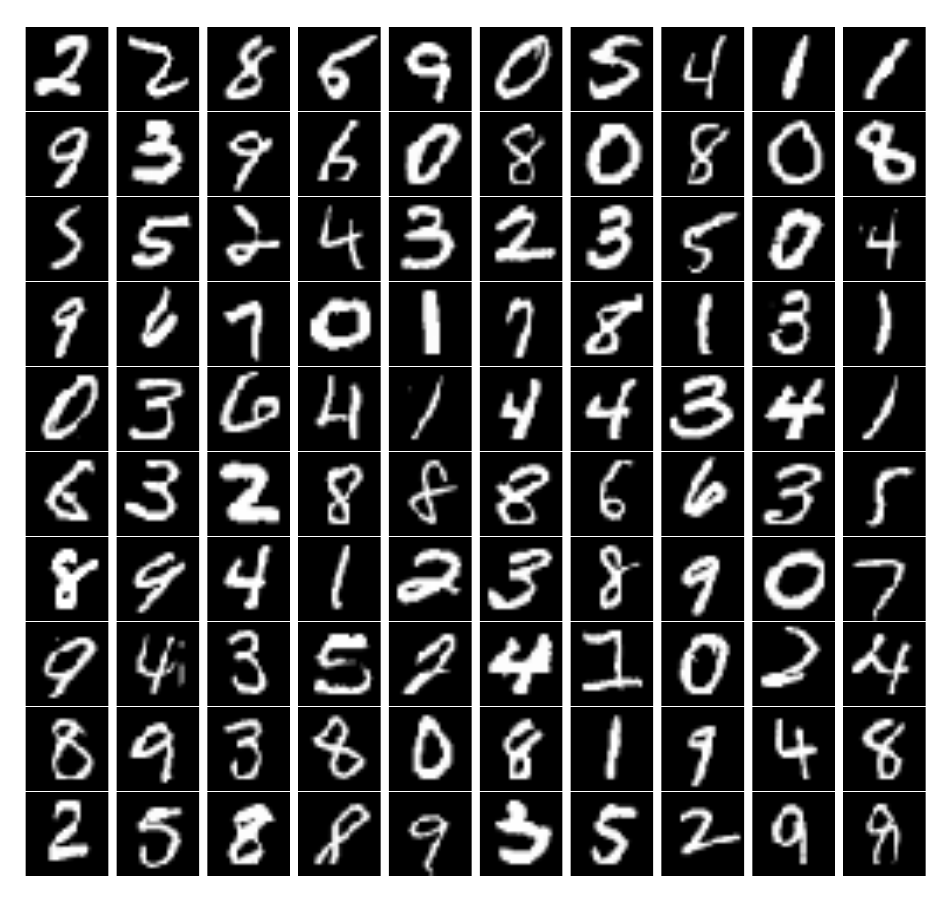}}
%\hspace{.5cm}
\subfloat[Semi-BNP-MMD GAN]{\includegraphics[width=.24\linewidth]{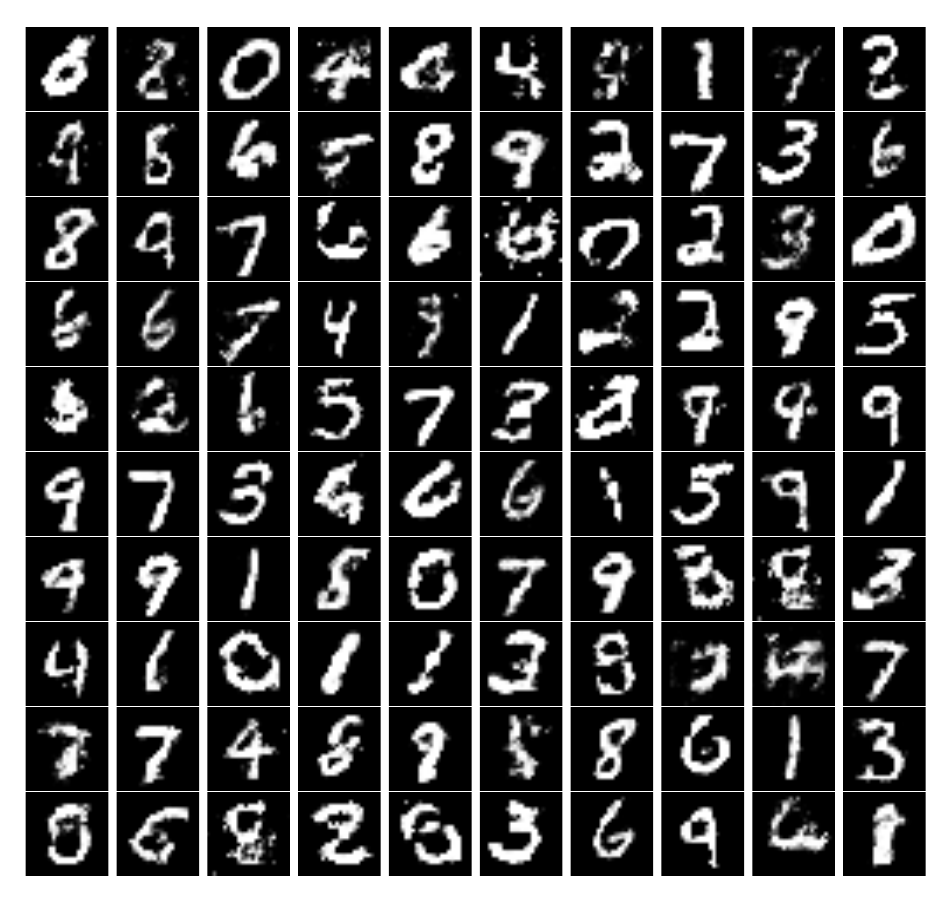}}
%\hspace{.5cm}
\subfloat[FNP-MMD GAN]{\includegraphics[width=.24\linewidth]{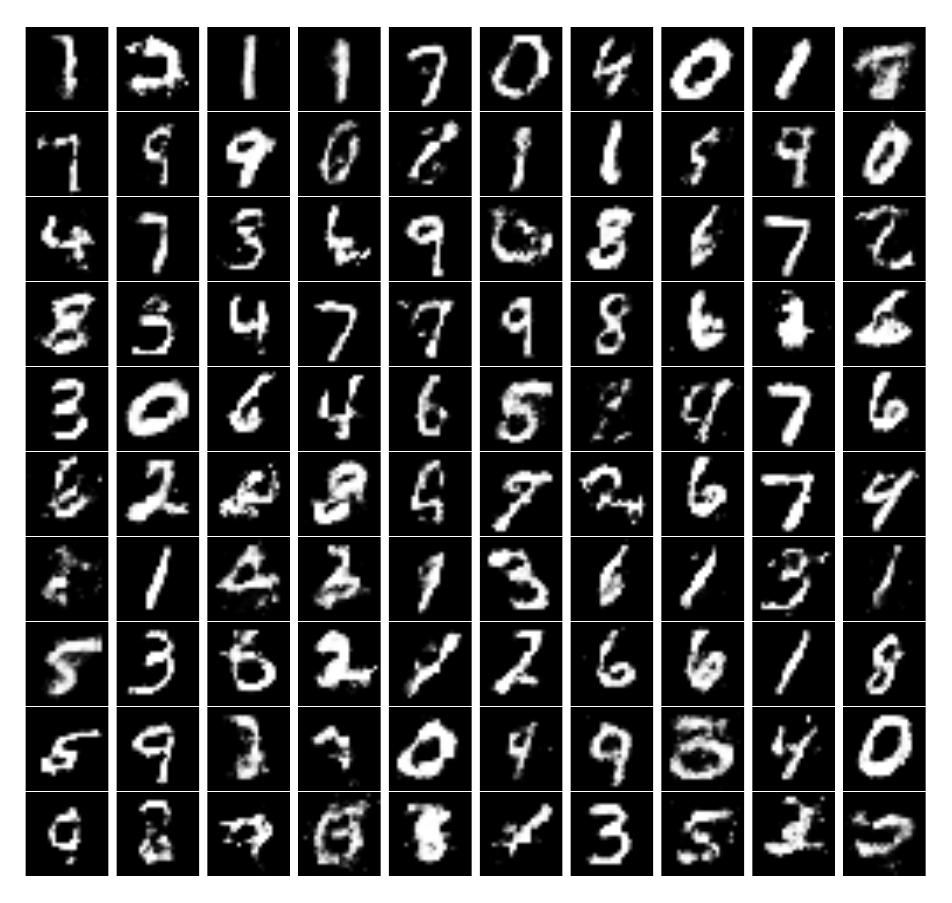}}
\caption{Generated samples of sizes ($10\times10$) from semi-BNP-MMD and MMD-FNP GAN for the MNIST dataset using a mixture of Gaussian kernels in 40,000 iterations.}\label{mnist-pos-fre}
\end{figure}
%%%%%%%%%%%%%%%%%%%%%%%%%%%%

On the other hand, unlike the semi-BNP test, our experimental results demonstrate that the semi-BNP GAN, using a mixture of Gaussian kernels, outperforms the approach that considers only a single Gaussian kernel. To investigate this matter further, we present several samples of the trained generator using a Gaussian kernel with different values of $\sigma$, as well as the median heuristic $\sigma_{MH}$, in Figure \ref{mnist-pos-Vsig}. Note that the value of $\sigma_{MH}$ is updated in each iteration, and therefore, no specific value is reported for it in this figure. While increasing the value of $\sigma$ enhances the diversity of the generated images, it is evident that the resolution of the images in Figure \ref{mnist-pos-Vsig} does not reach the image quality achieved by the mixture kernel.

In contrast to using MMD kernel-based measures, it may also be interesting to consider the energy distance in learning GANs from a BNP perspective. To address this concern, we embed the two-sample BNP-energy test of \cite{Al-Labadi} in
training GANs as a discriminator and showing the generated samples in Figure \ref{mnist-energy}-a. This image clearly shows the inefficiency of the two-sample BNP test of \cite{Al-Labadi} in training the generator. The main issue in this test procedure is treating $F_{G_{\boldsymbol{\omega}}}$ as unknown distribution to place a DP prior on it which is contrary to update parameter $\boldsymbol{\omega}$ in the parameterized generative neural network $G_{\boldsymbol{\omega}}$.

One may also be interested in considering the semi-BNP-energy procedure in learning GANs which makes more sense to compare the semi-BNP-MMD results. To do this, we use the energy distance instead of the MMD in Algorithm 2 in the supplementary material. The results are presented in Figure \ref{mnist-energy}-b and show blurry and unclear images with no variety, which reflect the inefficiency of using the energy distance compared to the MMD kernel-based measure. More experiments are given in Section 6.2 of the supplementary material.

%%%%%%%%%%%%%%%%%%%%%%%%%%%
\begin{figure}[ht]
\centering
\subfloat[$\sigma=2$]{\includegraphics[width=.24\linewidth]{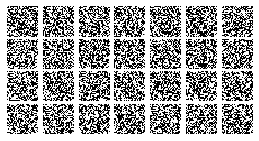}}
\hspace{0cm}
\subfloat[$\sigma=5$]{\includegraphics[width=.24\linewidth]{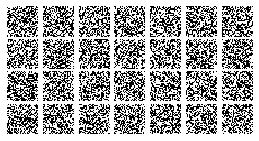}}
\hspace{0cm}
\subfloat[$\sigma=10$]{\includegraphics[width=.24\linewidth]{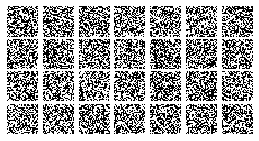}}
\subfloat[$\sigma=20$]{\includegraphics[width=.24\linewidth]{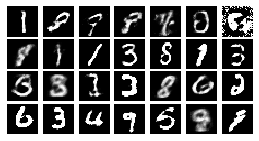}}
\qquad
\subfloat[$\sigma=40$]{\includegraphics[width=.24\linewidth]{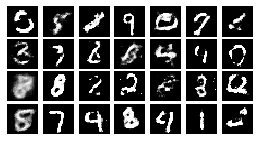}}
\hspace{0cm}
\subfloat[$\sigma=80$]{\includegraphics[width=.24\linewidth]{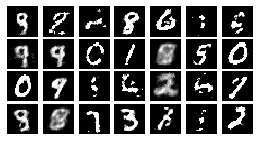}}
\hspace{0cm}
\subfloat[$\sigma_{MH}$]{\includegraphics[width=.24\linewidth]{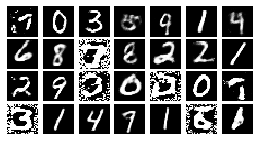}}
\caption{Generated samples from semi-BNP-MMD for the MNIST dataset using a single Gaussian kernel with various values of bandwidth parameter $\sigma$ in 40,000 iterations.}\label{mnist-pos-Vsig}
\end{figure}
%%%%%%%%%%%%%%%%%%%%%%%%%%%%

%%%%%%%%%%%%%%%%%%%%%%%%%%%
\begin{figure}[ht]
\centering
\subfloat[]{\includegraphics[width=.24\linewidth]{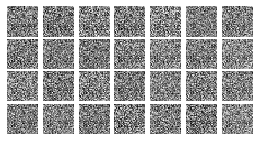}}
\hspace{1cm}
\subfloat[]{\includegraphics[width=.24\linewidth]{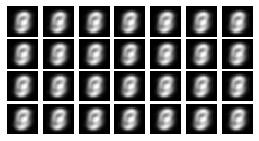}}
\caption{Generated samples from BNP-Energy GAN (a) and semi-BNP-Energy GAN (b) for the MNIST dataset in 40,000 iterations. }\label{mnist-energy}
\end{figure}
%%%%%%%%%%%%%%%%%%%%%%%%%%%%

\section{Conclusion}\label{sec-conclusion}

Our semi-BNP approach effectively estimates the MMD measure between an unknown distribution and an intractable parametric distribution. It outperforms frequentist counterparts and even surpasses a recent BNP competitor in certain scenarios \citep{Al-Labadi}. This approach shows great potential in training GANs, where the proposed estimator serves as a discriminator, inducing a posterior distribution on the generator's parameter space. Stick-breaking representation lacks normalization terms and exhibits stochastic decrease, making it inefficient for simulations \citep{zarepour2012rapid}. Thus, exploring alternative DP approximations for MMD estimation presents an intriguing avenue for future research. %Theoretical properties of our approach have been established for testing and data generation, and extensive experiments on popular datasets confirm the exceptional performance of our semi-BNP GAN. 
Future work will focus on generating 3D medical images to further enhance results.
\section*{Acknowledgments}
The contribution of Michael Zhang was partially funded by the HKU-URC Seed Fund for Basic Research for New Staff.

%%%%%%%%%%%%%%%%%%%%%%%%%%%%

\clearpage
\bibliographystyle{apalike}
\bibliography{GAN-proposal}
\clearpage
\appendix
\begin{center}
	\Large{\bf Supplementary Material}
\end{center}
\section{Technical Proofs}
\subsection{Theoretical Properties of the DP Approximation given by \cite{Ishwaran}}\label{app-prob}
\begin{proposition}\label{proposition1}
	For a non-negative real value $a$ and fixed probability distribution $H$, let $F^{pri}_{1}:=F_1\sim DP(a,H)$ and $(J_{1,N},\ldots,J_{N,N})\sim \mbox{Dirichlet}(\frac{a}{N},\ldots,\frac{a}{N})$ be the weights in  the approximation of $F^{pri}$, given by \cite{Ishwaran}. Then, as $a\rightarrow{\infty}$,\\
	\noindent $i.$ 
	$J_{\ell,N}\xrightarrow{a.s.}\frac{1}{N},$ for any $\ell\in\lbrace 1,\ldots,N\rbrace$,\\
	\noindent $ii.$ $J_{\ell,N}J_{t,N}\xrightarrow{a.s.}\frac{1}{N^2},$ for any $\ell,t\in\lbrace 1,\ldots,N\rbrace,$
	where $\ell\neq t $.\\
\end{proposition}
\begin{proof}
	Recall 
	\begin{equation}\label{approx of DP}
		F^{pri}_{N}=\sum_{i=1}^{N}J_{i,N}\delta_{Y_{i}}.
	\end{equation}
	Since $E_{F^{pri}_1}(J_{\ell,N})=\frac{1}{N}$, for any $\ell\in\lbrace 1,\ldots,N\rbrace$ and  $\epsilon>0$, Chebyshev’s
	inequality implies
	\begin{align*}
		\operatorname {Pr}\left\lbrace \left| J_{\ell,N}-1/N\right|\geq\epsilon\right\rbrace \leq \frac{Var(J_{\ell,N})}{\epsilon^2},
	\end{align*}
	where, $Var_{F^{pri}_1}(J_{\ell,N})=\frac{N-1}{N^2(a+1)}$. Assuming $a=\kappa^2 c$ for $\kappa\in \mathbb{N}$ and a fixed positive number $c$, gives
	\begin{align*}
		\operatorname {Pr}\left\lbrace \left| J_{\ell,N}-1/N\right|\geq\epsilon\right\rbrace \leq \frac{1}{\kappa^2c\epsilon^2}.
	\end{align*}
	The convergence of series $\sum_{\kappa=0}^{\infty}\kappa^{-2}$ implies 
	$\sum_{\kappa=0}^{\infty} \operatorname {Pr}\left\lbrace \left| J_{\ell,N}-1/N\right|\geq\epsilon\right\rbrace<\infty$. By letting $ a\rightarrow\infty$, the first Borel Cantelli lemma concludes 
	$\left| J_{\ell,N}-1/N\right|\xrightarrow{a.s.}0$ and the result of (i) follows. To prove (ii), it is enough to show $\operatorname {Pr}\left\lbrace \lim_{a \to \infty} (J_{\ell,N}J_{t,N})\neq\frac{1}{N^2}\right\rbrace=0$. To prove this for the probability space $ {\displaystyle (\Omega ,{\mathcal {F}},\operatorname {Pr} )}$, let
	\begin{align*}
		A&=\left\lbrace \omega\in\Omega:\, \lim_{a \to \infty} \left(J_{\ell,N}(\omega)J_{t,N}(\omega)\right)\neq\frac{1}{N^2} \right\rbrace,\hspace{2mm}
		B=\left\lbrace \omega\in\Omega:\, \lim_{a \to \infty} \left(J_{\ell,N}(\omega)\right)\neq\frac{1}{N} \right\rbrace,\nonumber\\
		C&=\left\lbrace \omega\in\Omega:\, \lim_{a \to \infty} \left(J_{t,N}(\omega)\right)\neq\frac{1}{N} \right\rbrace,
	\end{align*}
	where, $\operatorname {Pr}(B)$ and $\operatorname {Pr}(C)$ are zero by (i). Since $A\subseteq B\cup C$, then,
	\begin{align*}
		1-\operatorname {Pr}\left\lbrace \omega\in\Omega:\, \lim_{a \to \infty} \left(J_{\ell,N}(\omega)J_{t,N}(\omega)\right)=\frac{1}{N^2} \right\rbrace
		=\operatorname {Pr}(A)\leq \operatorname {Pr}(B)+\operatorname {Pr}(C)=0,
	\end{align*}
	which concludes the result.
\end{proof}
\subsection{Proof of Theorem 1}
\begin{proof}
	For samples $\lbrace \mathbf{V}_{\ell} \rbrace_{\ell=1}^{N}$ and  $\lbrace \mathbf{Y}_{\ell} \rbrace_{\ell=1}^{m}$, respectively, from $H$ and $F_2$, the triangle inequality implies
	\begin{align*}
		\left| \mathrm{MMD}^{2}_{\mathrm{BNP}}(F_{1,N}^{pri},F_{2,m})-\mathrm{MMD}^{2}(H_N,F_{2,m})\right|&\leq
		K\Bigg\lbrace\sum_{\ell,t=1}^{N}\left| J_{\ell,N}J_{t,N}-\frac{1}{N^2} \right| \nonumber\\
		&~~~~~~~~~~+\frac{2}{m} \sum_{\ell=1}^{N}\sum_{t=1}^{m}\left| J_{\ell,N}-\frac{1}{N} \right| \Bigg\rbrace.
	\end{align*}    
	By Proposition~\ref{proposition1}, which provides some theoretical properties of the DP approximation given in \eqref{approx of DP},  the right-hand side of the above inequality converges almost surely to 0 as $a\rightarrow\infty$ for fixed $N$. This convergence immediately concludes the proof of (i). To prove (ii),
	since $(J_{1,N},\ldots,J_{N,N})\sim \mbox{Dirichlet}(\frac{a}{N},\ldots,\frac{a}{N})$, $E_{F^{pri}_{1}}(J_{\ell,N})=\frac{1}{N}$ and
	\begin{align*}
		E_{F^{pri}_{1}}(J_{\ell,N}J_{t,N})=
		\begin{cases}
			\dfrac{a}{(a+1)N^2} & \text{if $\ell\neq t$,}\\
			\dfrac{a+N}{(a+1)N^2} & \text{if $\ell= t$}.\\
		\end{cases}       
	\end{align*}
	Applying these properties in definition of $\mathrm{MMD}^{2}_{\mathrm{BNP}}(F_{1,N}^{pri},F_{2,m})$ results in
	\begin{align}\label{pri-a-mmd}
		E_{F^{pri}_{1}}(\mathrm{MMD}^{2}_{\mathrm{BNP}}(F_{1,N}^{pri},F_{2,m})|\mathbf{V}_{1:N})&=\sum_{\ell=1}^{N}\sum_{t\neq\ell}^{N} \dfrac{ak(\mathbf{V}_{\ell},\mathbf{V}_{t})}{(a+1)N^2}+\sum_{\ell=1}^{N}\sum_{t=\ell}^{N} \frac{(a+N)k(\mathbf{V}_{\ell},\mathbf{V}_{t})}{(a+1)N^2} \nonumber\\
		&-\dfrac{2}{Nm}\sum_{\ell=1}^{N}\sum_{t=1}^{m} k(\mathbf{V}_{\ell},\mathbf{Y}_{t})
		+\dfrac{1}{m^2}\sum_{\ell,t=1}^{m} k(\mathbf{Y}_{\ell},\mathbf{Y}_{t}).
	\end{align} 
	%letting real data distribution \ref{pri-a-mmd} 
	%for a fixed base measure $H$. 
	%\textcolor{gray}{To prove (ii), let
		%    \begin{align*}
			%        &\Delta=\lim_{a \to \infty}E_{F^{pri}_{1}}(\mathrm{MMD}^{2}_{\mathrm{BNP}}(F_{1,N}^{pri},F_{2,m}))-\mathrm{MMD}^2(H,F_2),\nonumber\\
			%&\Delta^{\prime}=\mathrm{MMD}^2(H_N,F_{2,m})-\mathrm{MMD}^2(H,F_2).
			%    \end{align*}
		%   Following \citet[Theorem 7]{Gretton} and then applying part (i), gives us
		%   \begin{align}\label{prob}
			%       \operatorname {Pr}\left\lbrace \left|\Delta \right|> \frac{2\sqrt{K}%(\sqrt{N}+\sqrt{m})}{\sqrt{Nm}}+\epsilon\right\rbrace&=
			%    \operatorname {Pr}\left\lbrace \left|\Delta^{\prime} %\right|>\frac{2\sqrt{K}(\sqrt{N}+\sqrt{m})}%{\sqrt{Nm}}+\epsilon\right\rbrace
			%      \nonumber\\
			%      &\leq 2\exp{\dfrac{-\epsilon^2Nm}{2K(N+m)}},
			%  \end{align}
		%   for any $\epsilon>0$, which concludes the results.}
	Now, it is sufficient to compute the following conditional expectation,
	\begin{align}\label{cond-exp-pri}
		E(\mathrm{MMD}^{2}_{\mathrm{BNP}}(F_{1,N}^{pri},F_{2,m}))&= E_{H,F_2}(E_{F^{pri}_{1}}(\mathrm{MMD}^{2}_{\mathrm{BNP}}(F_{1,N}^{pri},F_{2,m})|\mathbf{V}_{1:N})).
	\end{align}
	Since sets $\lbrace V_i\rbrace_{i=1}^{N}$ and $\lbrace Y_i\rbrace_{i=1}^{m}$ include i.i.d. random variables, separately, replacing Equation \eqref{pri-a-mmd} in expectation \eqref{cond-exp-pri} implies:
	\begin{align}\label{exp-pri-m}
		\eqref{cond-exp-pri}&= \dfrac{a(N-1)}{(a+1)N} E_{H}[k(\mathbf{V}_{1},\mathbf{V}_{2})]
		+\dfrac{a+N}{(a+1)N}E_{H}[k(\mathbf{V}_{1},\mathbf{V}_{1})]
		-2 E_{H,F_2}[k(\mathbf{V}_{1},\mathbf{Y}_{1})]\nonumber\\
		&~~~+\dfrac{m-1}{m} E_{F_2}[k(\mathbf{Y}_{1},\mathbf{Y}_{2})]+\dfrac{1}{m} E_{F_2}[k(\mathbf{Y}_{1},\mathbf{Y}_{1})].
	\end{align}
	The proof of (ii) is concluded by letting $a\rightarrow\infty$, $N\rightarrow\infty$, and $m\rightarrow\infty$ in the above equation. Lastly, since $\frac{1}{m}<1$, $\frac{m-1}{m}<1$, $\frac{a(N-1)}{(a+1)N}<1$, and $ \frac{a+N}{(a+1)N}<2$, then, for any $N,m\in \mathbb{N}$ and $a\in \mathbb{R}^+$, 
	\begin{align*}
		\eqref{exp-pri-m}<E_{H}[k(\mathbf{V}_{1},\mathbf{V}_{2})]
		-2 E_{H,F_2}[k(\mathbf{V}_{1},\mathbf{Y}_{1})]
		+ E_{F_2}[k(\mathbf{Y}_{1},\mathbf{Y}_{2})]+3K,
	\end{align*}
	which concludes the proof of (iii).
\end{proof}

\subsection{Proof of  Theorem 2}
\begin{proof}
	Applying triangular inequality implies
	\begin{align}\label{con-empri-pos}
		\left| \mathrm{MMD}^{2}_{\mathrm{BNP}}(F_{1,N}^{pos},F_{2,m})-\mathrm{MMD}^{2}(H_N,F_{2,m})\right|&\leq
		\sum_{\ell,t=1}^{N}\left| J^{\ast}_{\ell,N}J^{\ast}_{t,N} k(\mathbf{V}^{\ast}_\ell,\mathbf{V}^{\ast}_t)-\frac{1}{N^2}  k(\mathbf{V}_\ell,\mathbf{V}_t)\right|\nonumber\\
		&+\frac{2}{m} \sum_{\ell=1}^{N}\sum_{t=1}^{m}\left| J^{\ast}_{\ell,N}k(\mathbf{V}^{\ast}_\ell,\mathbf{Y}_t)-\frac{1}{N} k(\mathbf{V}_\ell,\mathbf{Y}_t)\right|,
	\end{align} 
	where, samples $\lbrace \mathbf{V}^{\ast}_{\ell} \rbrace_{\ell=1}^{N}$ and  $\lbrace \mathbf{Y}_{\ell} \rbrace_{\ell=1}^{m}$ are generated from $H^{\ast}$ and $F_2$, respectively. Similar to Proposition \ref{proposition1}, it can be shown that $J^{\ast}_{\ell,N}\rightarrow 1/N$ and $J^{\ast}_{\ell,N}J^{\ast}_{t,N}\rightarrow 1/N^2$, as $a\rightarrow\infty$, using conjugacy property of DP. On the other hand, since $H^{\ast}\rightarrow H$ as $a\rightarrow\infty$, the chance of sampling from $H$ and $F_{1,n}$ tends, respectively, to one and zero, which implies $V_i^{\ast}\rightarrow V_i$, where $V_i\sim H$, for $i=1,2$. Applying the continuous mapping theorem implies $k(\mathbf{V}^{\ast}_{1},\mathbf{V}^{\ast}_{2})\rightarrow k(\mathbf{V}_{l},\mathbf{V}_{t})$ and $k(\mathbf{V}^{\ast}_{l},\mathbf{Y}_{t})\rightarrow k(\mathbf{V}_{l},\mathbf{Y}_{t})$, which completes the proof of (i)(a).
	To prove (i)(b), it follows from the proof of Theorem 1: 
	\begin{align}
		\label{exp-pos-m}
		E(\mathrm{MMD}^{2}_{\mathrm{BNP}}(F_{1,N}^{pos},F_{2,m}))&= h_1(a,n,N) E_{H^{\ast}}[k(\mathbf{V}^{\ast}_{1},\mathbf{V}^{\ast}_{2})]
		+h_2(a,n,N) E_{H^{\ast}}[k(\mathbf{V}^{\ast}_{1},\mathbf{V}^{\ast}_{1})]\nonumber\\
		&~~~~-2 E_{H^{\ast},F_2}[k(\mathbf{V}^{\ast}_{1},\mathbf{Y}_{1})]+\frac{m-1}{m} E_{F_2}[k(\mathbf{Y}_{1},\mathbf{Y}_{2})]\nonumber\\
		&~~~~+\frac{1}{m} E_{F_2}[k(\mathbf{Y}_{1},\mathbf{Y}_{1})],
	\end{align}
	where $h_1(a,n,N)=\frac{(a+n)(N-1)}{(a+n+1)N}$ and $h_2(a,n,N)=\frac{a+n+N}{(a+n+1)N}$. Since $k(\cdot,\cdot)$ is bounded above by $K$, the dominated convergence theorem implies $E_{H^{\ast}}[k(\mathbf{V}^{\ast}_{1},\mathbf{V}^{\ast}_{2})]\rightarrow E_{H}[k(\mathbf{V}_{1},\mathbf{V}_{2})]$ and $E_{H^{\ast},F_2}[k(\mathbf{V}^{\ast}_{1},\mathbf{Y}_{1})]\rightarrow E_{H,F_2}[k(\mathbf{V}_{1},\mathbf{Y}_{1})]$. Since $h_1(a,n,N)\rightarrow 1$ and $h_2(a,n,N)\rightarrow 0$  as $a\rightarrow\infty$, $N\rightarrow\infty$; and, $m/(m-1)\rightarrow 1$ and $1/m\rightarrow 0$ , as $ m\rightarrow\infty$, the results follow.
	
	To prove (ii)(a) and (ii)(b), $F_{1,n}\rightarrow F_1$, and then $H^{\ast}\rightarrow F_1$ as $n\rightarrow\infty$ by the Glivenko-Cantelli theorem. It indicates that the probability of sampling from $H$ and $F_{n1}$ tends, respectively, to zero and one. Therefore, $V_i^{\ast}\rightarrow X_i$ as $n\rightarrow\infty$, where $X_i\sim F_1$, for $i=1,2$. The proof of (ii)(a) is completed with the same strategy as the proof of (i)(a) by letting $n\rightarrow\infty$ in \eqref{con-empri-pos}. The proof of (ii)(b) is also concluded with a similar argument that in (i)(b), when $n\rightarrow\infty$ in \eqref{exp-pos-m}.
\end{proof}
\subsection{Proof of Corollary 3}
\begin{proof}
	The proofs are immediately followed by Theorem 1 and Theorem 2.    
\end{proof}
\subsection{Proof of Lemma 4}
\begin{proof}
	The proof of Lemma 4(i) relies on the proof given in \citet[Theorem 9]{dellaporta2022robust} which is expanded for infinite stick-breaking representation, while we consider the finite DP approximation given in \eqref{approx of DP}. By employing a similar technique as in the previously mentioned theorem, we have 
	\begin{align*}
		E\left(\mathrm{MMD}(F,F_{G_{\boldsymbol{\omega}^{\ast}}})\right)&=E_F\left(E_{F^{pos}}\mathrm{MMD}(F,F_{G_{\boldsymbol{\omega}^{\ast}}})|\textbf{X}_{1:n}\right)\\
		&\leq
		\min\limits_{\boldsymbol{\omega}\in\mathcal{W}}\mathrm{MMD}(F,F_{G_{\boldsymbol{\omega}}})
		+2E_{F}\left(\mathrm{MMD}(F_n,F)\right)
		+2E_{F^{pos}}\left(\mathrm{MMD}(F^{pos}_{N},H^{\ast})\right)\\
		&~~~+2E_{F}\left(E_{H}(\mathrm{MMD}(F_n,H^{\ast})|\textbf{X}_{1:n})\right).
	\end{align*}
	Building on the results of \citet[Lemma 7]{dellaporta2022robust}, we can establish that 
	\begin{align*}
		E_{F^{pos}}\left(\mathrm{MMD}^{2}(F^{pos}_{N},H^{\ast})\right)\leq \sum_{\ell=1}^{N}E_{F^{pos}}[J_{\ell,N}^{\ast^2}]E_{H^{\ast}}[k(\mathbf{V}^{\ast}_{\ell},\mathbf{V}^{\ast}_{\ell})]\leq \dfrac{(a+n+N)K}{(a+n+1)N},
	\end{align*}
	where the right-hand side of the above inequality follows from the fact that $k(\cdot,\cdot)\leq K$ and $E_{F^{pos}}[J_{\ell,N}^{\ast^2}]=\frac{a+n+N}{(a+n+1)N^2}$. Now, the Jensen's inequality implies 
	\begin{align*}
		E_{F^{pos}}\left(\mathrm{MMD}(F^{pos}_{N},H^{\ast})\right)\leq \sqrt{\dfrac{(a+n+N)K}{(a+n+1)N}}.
	\end{align*}
	On the other hand, \citet[Lemma 7.1]{cherief2022finite} and \citet[Lemma 8]{dellaporta2022robust}, respectively, imply that
	\begin{align*}
		E_{F}\left(\mathrm{MMD}(F_n,F)\right)\leq\dfrac{K}{\sqrt{n}},E_{F}\left(E_{H}(\mathrm{MMD}(F_n,H^{\ast})|\textbf{X}_{1:n})\right)\leq\dfrac{2aK}{a+n},
	\end{align*}
	which concludes the
	proof of (i). To establish (ii), we adopt the approach used in the proof of \citet[Corollary 5]{dellaporta2022robust}. Initially, we employ \citet[Lemma 3.3]{cherief2022finite} to bound $\mathrm{MMD}(F_0,F_{G_{\boldsymbol{\omega}^{\ast}}})$ by $2\epsilon+\mathrm{MMD}(F,F_{G_{\boldsymbol{\omega}^{\ast}}})$, resulting in:
	\begin{align*}
		E\left(\mathrm{MMD}(F_0,F_{G_{\boldsymbol{\omega}^{\ast}}})\right)\leq2\epsilon+E\left(\mathrm{MMD}(F,F_{G_{\boldsymbol{\omega}^{\ast}}})\right).
	\end{align*}
	Applying the result in (i) to the right-hand side of the above inequality implies:
	\begin{align*}
		E\left(\mathrm{MMD}(F_0,F_{G_{\boldsymbol{\omega}^{\ast}}})\right)\leq 2\epsilon+\min\limits_{\boldsymbol{\omega}\in\mathcal{W}}\mathrm{MMD}(F,F_{G_{\boldsymbol{\omega}}})+\dfrac{2K}{\sqrt{n}}+\dfrac{4aK}{a+n}+2\sqrt{\dfrac{(a+n+N)K}{(a+n+1)N}}.
	\end{align*}
	Finally, we employ \citet[Lemma 3.3]{cherief2022finite} once again, but this time to bound $\mathrm{MMD}(F,F_{G_{\boldsymbol{\omega}}})$ by $2\epsilon+\mathrm{MMD}(F_0,F_{G_{\boldsymbol{\omega}}})$ for any $\boldsymbol{\omega}\in\mathcal{W}$, thereby completing the proof of (ii).
\end{proof}
\subsection{Proof of Lemma 5}
\begin{proof}
	Let $\mathcal{L}_{\mathrm{BNP}}(\boldsymbol{\omega})=\mathrm{MMD}(F^{pos}_{N},F_{G_{\boldsymbol{\omega}}})$, $\mathcal{L}_{n,m}(\boldsymbol{\omega})=\mathrm{MMD}(F_{n},F_{G_{\boldsymbol{\omega}},m})$, and $\mathcal{L}(\boldsymbol{\omega})=\mathrm{MMD}(F,F_{G_{\boldsymbol{\omega}}})$. Then, for $\boldsymbol{\omega}^{\ast}\in \mathcal{W}$, \citet[Theorem 7]{Gretton} implies 
	\begin{align}\label{prob1-lem}
		\operatorname{Pr}\left(|\mathcal{L}_{n,m}(\boldsymbol{\omega}^{\ast})-\mathcal{L}(\boldsymbol{\omega}^{\ast})|>h(N,m,K,\epsilon)\right)<2\exp{\frac{-\epsilon^2nm}{2K(n+m)}}.
	\end{align}
	%By applying factorization in \eqref{prob1-lem}, we have
	%\begin{align*}%\label{prob2-lem}
	%    \operatorname{Pr}\left(|\mathcal{L}_{\mathrm{n,m}}(%\boldsymbol{\omega}^{\ast})-\mathcal{L}(%\boldsymbol{\omega}^{\ast})|>\frac{2\sqrt{K}(\sqrt{n%}+\sqrt{m})/\sqrt{nm}+\epsilon}{|\mathcal{L}_{\mathrm{%n,m}}(\boldsymbol{\omega}^{\ast})+\mathcal{L}(%\boldsymbol{\omega}^{\ast})|}\right)<2\exp{\frac{-%\epsilon^2nm}{2K(n+m)}}.
	%\end{align*}
	Hence, with a probability at least $1-2\exp{\frac{-\epsilon^2nm}{2K(n+m)}}$, 
	\begin{align}\label{prob4-lem}
		|\mathcal{L}_{n,m}(\boldsymbol{\omega}^{\ast})-\mathcal{L}(\boldsymbol{\omega}^{\ast})|\leq h(n,m,K,\epsilon).
	\end{align}
	%Similarly, for $\boldsymbol{\omega}^{\prime}\in \mathcal{W}$, with a probability at least $1-2\exp{\frac{-\epsilon^2Nm}{2K(N+m)}}$, $|\lim_{n \to \infty}\mathcal{L}_{\mathrm{BNP}}(\boldsymbol{\omega}^{\prime})-\mathcal{L}(\boldsymbol{\omega}^{\prime})|\leq h(N,m,\epsilon)$. Therefore,
	On the other hand, the triangle inequality implies
	\begin{align}\label{prob5-lem}
		|\mathcal{L}_{\mathrm{BNP}}(\boldsymbol{\omega}^{\ast})-\mathcal{L}(\boldsymbol{\omega}^{\prime})|\leq 
		|\mathcal{L}_{n,m}(\boldsymbol{\omega}^{\ast})-\mathcal{L}(\boldsymbol{\omega}^{\ast})|+
		|\mathcal{L}_{\mathrm{BNP}}(\boldsymbol{\omega}^{\ast})-\mathcal{L}_{n,m}(\boldsymbol{\omega}^{\ast})|+
		|\mathcal{L}(\boldsymbol{\omega}^{\ast})-\mathcal{L}(\boldsymbol{\omega}^{\prime})|.
	\end{align}
	Finally, the proof of (i) is concluded by considering inequality \eqref{prob4-lem} in \eqref{prob5-lem}. To prove (ii), Markov's inequality implies
	\begin{align*}
		\mathrm{Pr} \left( \mathrm{MMD}(F,F_{G_{\boldsymbol{\omega}^{\ast}}})\geq\epsilon\right)\leq 
		\dfrac{E\left( \mathrm{MMD}(F,F_{G_{\boldsymbol{\omega}^{\ast}}})\right)}{\epsilon}.
	\end{align*}
	The result follows by substituting the bounds from Lemma 4(i) into the right-hand side of the above inequality. 
\end{proof}
\section{Computational Algorithms}
\subsection{Implementing the Semi-BNP GOF Kernel-based Test}
Recall
\begin{align}\label{random-stopping}
	N=\inf\left\lbrace j:\, \frac{H_{j,j}}{\sum_{i=1}^{j}H_{i,j}}<\epsilon\right\rbrace.
\end{align}
{\setstretch{.5}
	\begin{breakablealgorithm}\label{alg-test}
		\caption{Pseudocode of semi-BNP two-sample MMD kernel test}
		\begin{algorithmic}[1]
			%\algrestore{bkbreak}
			%\State 
			\State Initialize $a$, $\ell$, $M$, $i_0$, and $\epsilon$  in Equation \eqref{random-stopping} to determine $N$.
			\State $H\gets F_2$
			\Statex \underline{STEP 1: Computing the BNP MMD}
			\For{$r\gets 0$ to $\ell$}{}{} 
			\State Generate an approximate sample of $F_1\sim DP(a,H)$  by using  $\sum_{i=1}^{N}J_{i,N}\delta_{\mathbf{V}_{i}}$, where $\lbrace J_{i,N}\rbrace_{i=1}^{N}\sim \mbox{Dirichlet}(\frac{a}{N},\cdots,\frac{a}{N})$, and  $\lbrace\mathbf{V}_{i}\rbrace_{i=1}^{N}\sim H$.
			\State Generate an approximate sample of $F_1|\mathbf{x}_{1:n}\sim DP(a+n,H^{\ast})$  by using  $\sum_{i=1}^{N}J^{\ast}_{i,N}\delta_{\mathbf{V}^{\ast}_{i}}$, where $\lbrace J^{\ast}_{i,N}\rbrace_{i=1}^{N}\sim \mbox{Dirichlet}(\frac{a+n}{N},\cdots,\frac{a+n}{N})$,  and $\lbrace\mathbf{V}^{\ast}_{i}\rbrace_{i=1}^{N}\sim H^{\ast}$.
			\State Use %\eqref{BNP-pri-MMD} and \eqref{BNP-pos-MMD} for the 
			samples generated in steps 4 and 5 to compute
			\Statex $\mathrm{MMD}^2_{\mathrm{BNP}}(F^{pri}_{1,N},F_{2,m})$ and $\mathrm{MMD}^2_{\mathrm{BNP}}(F^{pos}_{1,N},F_{2,m})$, respectively.
			\EndFor
			\State \textbf{return} $\lbrace \mathrm{MMD}^2_{\mathrm{BNP}_r}(F^{pri}_{1,N},F_{2,m})\rbrace_{r=1}^{\ell}$ and $\lbrace \mathrm{MMD}^2_{\mathrm{BNP}_r}(F^{pos}_{1,N},F_{2,m})\rbrace_{r=1}^{\ell}$
			\Statex \underline{STEP 2: Estimating RB and Str}
			\State {\footnotesize$\widehat{\Pi}_{\mathrm{MMD}^2}(\cdot|\mathbf{x}_{1:n})\gets ECDF(\lbrace \mathrm{MMD}^2_{\mathrm{BNP}_r}(F^{pos}_{1,N},F_{2,m})\rbrace_{r=1}^{\ell})$} \Comment The ECDF of  posterior-based MMD
			\State {\footnotesize$\widehat{\Pi}_{\mathrm{MMD}^2}(\cdot)\gets ECDF(\lbrace \mathrm{MMD}^2_{\mathrm{BNP}_r}(F^{pri}_{1,N},F_{2,m})\rbrace_{r=1}^{\ell})$} \Comment The ECDF of prior-based MMD
			\State {\footnotesize$\widehat{d}_{i_0/M}\gets quantile(\lbrace \mathrm{MMD}^2_{\mathrm{BNP}_r}(F^{pri}_{1,N},F_{2,m})\rbrace_{r=1}^{\ell},i_0/M)$} \Comment The estimation of the $i_0/M$-th quantile of $\mathrm{MMD}^2_{\mathrm{BNP}}(F^{pri}_{1,N},F_{2,m})$
			\State {\footnotesize$ \widehat{RB}_{\mathrm{MMD}^2}(0|\mathbf{x}_{1:n})\gets\frac{\widehat{\Pi}_{\mathrm{MMD}^2}(\hat{d}_{i_0/M}|\mathbf{x}_{1:n})}{\widehat{\Pi}_{\mathrm{MMD}^2}(\hat{d}_{i_0/M})}$}
			\State {\footnotesize$\widehat{Str}\gets 0$}
			\For{$i\gets 0$ to $M-1$}{}{} 
			\State {\footnotesize$\widehat{d}_{i/M}\gets quantile(\lbrace \mathrm{MMD}^2_{\mathrm{BNP}_r}(F^{pri}_{1,N},F_{2,m})\rbrace_{r=1}^{\ell},i/M)$}  
			\State {\footnotesize$\widehat{d}_{(i+1)/M}\gets quantile(\lbrace \mathrm{MMD}^2_{\mathrm{BNP}_r}(F^{pri}_{1,N},F_{2,m})\rbrace_{r=1}^{\ell},(i+1)/M)$}
			\State {\footnotesize$ \widehat{RB}_{\mathrm{MMD}^2}(\widehat{d}_{i/M}|\mathbf{x}_{1:n})\gets\frac{\widehat{\Pi}_{\mathrm{MMD}^2}(\widehat{d}_{(i+1)/M}|\mathbf{x}_{1:n})-\widehat{\Pi}_{\mathrm{MMD}^2}(\widehat{d}_{i/M}|\mathbf{x}_{1:n})}{\widehat{\Pi}_{\mathrm{MMD}^2}(\widehat{d}_{(i+1)/M})-\widehat{\Pi}_{\mathrm{MMD}^2}(\widehat{d}_{i/M})}$}
			\If{$\widehat{RB}_{\mathrm{MMD}^2}(\widehat{d}_{i/M}|\mathbf{x}_{1:n})\leq \widehat{RB}_{\mathrm{MMD}^2}(0|\mathbf{x}_{1:n})$}
			\State {\footnotesize$\widehat{Str}(0|\mathbf{x}_{1:n})\gets\widehat{Str}+[\widehat{\Pi}_{\mathrm{MMD}^2}(\widehat{d}_{(i+1)/M}|\mathbf{x}_{1:n})-\widehat{\Pi}_{\mathrm{MMD}^2}(\widehat{d}_{i/M}|\mathbf{x}_{1:n})]$}
			\EndIf
			\EndFor
			\State \textbf{return} {\footnotesize$\widehat{RB}_{\mathrm{MMD}^2}, \widehat{Str}$}
		\end{algorithmic}
	\end{breakablealgorithm}
}
\subsection{Training the Semi-BNP GAN}
{\setstretch{.5}
	\begin{breakablealgorithm}\label{alg-train}
		\caption{Pseudocode of training a GAN using the semi-BNP approach}
		\begin{algorithmic}[1]
			%\algrestore{bkbreak}
			\State Set $a=0$ to employ a non-informative prior leading DP posterior $DP(n,F_n)$.
			\State Initialize $\epsilon$ in Equation \eqref{random-stopping} to determine $N$ using conjugacy  property of DP.
			\State $r_{mn}\gets\text{Number of training iteration}$, $n_{mb}\gets\text{Mini-batch size}$\\
			$\boldsymbol{\omega}_{0}\gets$ An initial parameter for generator $G_{\boldsymbol{\omega}}$, $\lbrace \mathbf{x}_{\ell}\rbrace_{\ell=1}^{n}\gets$ real dataset
			\For{$i$ $\gets 0$ to $r_{mb}$}{}{} 
			%\State $H\gets G_{\boldsymbol{\omega}_i}$
			\State Generate a random sample $\lbrace \mathbf{x}^{mb}_{\ell}\rbrace_{\ell=1}^{n_{mb}}$ from real dataset $\lbrace \mathbf{x}_{\ell}\rbrace_{\ell=1}^{n}$
			%\State $F_{n_{mb}}\gets ECDF(\lbrace \mathbf{X}^{mb}_{\ell}\rbrace_{\ell=1}^{n_{mb}})$, $a^{\ast}_{mb}\gets a+n_{mb}$, $H^{\ast}_{mb}=\frac{a}{a^{\ast}_{mb}}H+\frac{n_{mb}}{a^{\ast}_{mb}}F_{n_{mb}}$
			%\State Generate an approximate sample of $DP(a^{\ast}_{mb},H^{\ast}_{mb})$  by using $\sum_{\ell=1}^{n_{mb}}J^{\ast}_{\ell,n_{mb}}\delta_{\mathbf{X}^{\ast}_{\ell}}$, where $\lbrace J^{\ast}_{\ell,n_{mb}}\rbrace_{\ell=1}^{n_{mb}}\sim Dir(\frac{a^{\ast}_{mb}}{n_{mb}},\cdots,\frac{a^{\ast}_{mb}}{n_{mb}})$ and $\lbrace\mathbf{X}^{\ast}_{\ell}\rbrace_{\ell=1}^{n_{mb}}\sim H^{\ast}_{mb}$
			\State Generate a sample of noise vector $\lbrace \mathbf{u}_{\ell}\rbrace_{\ell=1}^{n_{mb}}$ from uniform distribution $U(-1,1)$
			\State Generate a sample from $F_{G_{\boldsymbol{\omega}_{i}}}$, distribution of $G_{\boldsymbol{\omega}_{i}}$, as $\lbrace \mathbf{y}_{\ell}=G_{\boldsymbol{\omega}_{i}}( \mathbf{u}_{\ell})\rbrace_{\ell=1}^{n_{mb}}$
			\State Generate a sample of size $N$  from $F^{pos}=F|\lbrace \mathbf{x}^{mb}_{\ell}\rbrace_{\ell=1}^{n_{mb}}$  using  $\sum_{i=1}^{N}J^{\ast}_{i,N}\delta_{\mathbf{v}^{\ast}_{i}}$  by replacing $F_1$ by $F$, and $\lbrace \mathbf{x}^{mb}_{\ell}\rbrace_{\ell=1}^{n_{mb}}$ by $\mathbf{x}$ in step (4) of Algorithm \ref{alg-test}.
			\State Use generated samples in steps 9 and 10 to compute $\mathrm{MMD}^{2}_{\mathrm{BNP}}(F^{pos}_{N},F_{G_{\boldsymbol{\omega}_{i}},N})$.% given by \eqref{BNP-pos-MMD}.
			\State %Use \eqref{gradiant} to 
			Compute the gradient: {\scriptsize \begin{align*}
					\frac{\partial \mathrm{MMD}_{\mathrm{BNP}}(F^{pos}_{N},F_{G_{\boldsymbol{\omega}_{i}},m})}{\partial \boldsymbol{\omega}_i}=\frac{1}{2\sqrt{\mathrm{MMD}_{\mathrm{BNP}}^{2}(F^{pos}_{N},F_{G_{\boldsymbol{\omega}},m})}}\frac{\partial \mathrm{MMD}_{\mathrm{BNP}}^{2}(F^{pos}_{N},F_{G_{\boldsymbol{\omega}},m})}{\partial\boldsymbol{\omega}}.
			\end{align*}}
			\State Use backpropagation for calculating partial derivatives $\frac{\partial \mathbf{G_{\boldsymbol{\omega}_{i}}( \mathbf{u}_{\ell})}}{\partial \boldsymbol{\omega}_i} $ in the previous step to update parameter $\boldsymbol{\omega}_{i}$.
			%\State \small{$\mathrm{MMD}(F^{pos}_{mb},F_{G_{\boldsymbol{\omega}_{i}}})\gets \sum_{\ell,m=1}^{n_{mb}} J^{\ast}_{\ell,n_{mb}}J^{\ast}_{m,n_{mb}}k(\mathbf{X}^{\ast}_{\ell},\mathbf{X}^{\ast}_{m})+\frac{2}{n^{2}_{mb}}\sum_{\ell,m=1}^{n_{mb}} k(\mathbf{Y}_{\ell},\mathbf{Y}_{m})$\\ \hspace{.5cm}$-\frac{2}{n_{mb}}\sum_{\ell,m=1}^{n_{mb}}J^{\ast}_{\ell,n_{mb}}k(\mathbf{X}^{\ast}_{\ell},\mathbf{Y}_{m})$}
			\EndFor
			\State \textbf{return}  $\boldsymbol{\omega}^{\ast}$ \Comment An optimized parameter for $G_{\boldsymbol{\omega}}$ that minimizes the cost function.
		\end{algorithmic}
	\end{breakablealgorithm}
}
\subsection{Hypothesis Testing Evaluation}\label{app-alg}
{\setstretch{.5}
	\begin{breakablealgorithm}\label{alg-roc}
		\caption{Pseudocode of plotting ROC and computing AUC in semi-BNP test}
		\begin{algorithmic}[1]
			%\algrestore{bkbreak}
			%\State 
			\State Initialize $a$, $N$, $\ell$, and $M$.
			\State $r\gets 100$
			\State $RB^{\dagger}|\mathcal{H}_0\gets$ Compute $RB$ for $r$ sample of sizes $n$ generated under the null hypothesis.
			\State $RB|\mathcal{H}_1\gets$ Compute $RB$ for $r$ sample of sizes $n$ generated under the alternative hypothesis.
			%\State Mixed-RB$\gets (RB|\mathcal{H}_0,RB|\mathcal{H}_1)$
			\State $T\gets$ A sequence of numbers between 0 to $20^{\ddagger}$ with length $L$.
			\Comment The discrimination threshold for the semi-BNP test. 
			\State $TP\gets$ A vector whose each component represents the number of components of the vector $RB|\mathcal{H}_1$ which is less than each component of $T$.
			\State $FN\gets$ A vector whose each component represents the number of components of the vector $RB|\mathcal{H}_1$ which is greater than each component of $T$.
			\State $FP\gets$ A vector whose each component represents the number of components of the vector $RB|\mathcal{H}_0$ which is less than each component of $T$.
			\State $TN\gets$ A vector whose each component represents the number of components of the vector $RB|\mathcal{H}_0$ which is greater than each component of $T$.
			\State Compute the confusion matrix as:
			\begin{small}
				\begin{equation*}
					\begin{pmatrix}
						TNR:=\frac{TN}{TN+FP}& FNR:= \frac{FN}{FN+TP} \\
						\text{(1-Type I error)} & \text{(Type II error)} \\
						&  \\
						&  \\
						FPR:= \frac{FP}{FP+TN} & TPR:=\frac{TP}{TP+FN}\\
						\text{(Type I error)} & \text{(1-Type II error)}
					\end{pmatrix}.
				\end{equation*}
			\end{small} 
			\State ROC~$\gets$ Drawing a linear plot of $TPR$ against $FPR$.
			\State AUC~$\gets$ Computing the area under the ROC.
			\State \textbf{return}  ROC and AUC. 
		\end{algorithmic}
		\centering
		\begin{tablenotes}
			\small
			\item \fontsize{8}{8}\selectfont{$^{\dagger}$ It should be changed to the $p$-value in the FNP test.}
			\item $^{\ddagger}$ It should be changed to 1 in the FNP test.
		\end{tablenotes}
	\end{breakablealgorithm}
}
\section{Relative Belief Ratio: A Bayesian Measure of Evidence}\label{RB-sec}
The RB ratio \citep{Evans15} is a form of Bayesian evidence in hypothesis testing problems %to investigate the correctness of a given statement.
%It was developed by  
and has shown excellent performance in many statistical hypothesis testing procedures \citep{Al-Labadi,al2021bayesian,al2022test}. The RB ratio is defined by the ratio of the posterior density %(updated statistician knowledge by data) 
to the prior density %(initialize statistician knowledge) 
at a particular parameter of interest in the population distribution whose correctness is under investigation. Precisely,
for a statistical model $(\mathfrak{X},\mathcal{F})$ with $\mathcal{F}=\lbrace f_{\theta}:\, \theta\in\Theta \rbrace$, let $\pi$ be a prior on the parameter space $\Theta$ and $\pi(\theta\,|\,x)$ be the posterior distribution of $\theta$ after observing the data $x$. 
Consider a parameter of interest, $\psi=\Psi(\theta)$ such that $\Psi$ satisfies regularity conditions so that the prior density $\pi_{\Psi}$ and the posterior density $\pi_{\Psi}(\cdot\,|\,x)$ of $\psi$ exist with respect to some support measure on the range space for $\Psi$.
When $\pi_{\Psi}$ and  $\pi_{\Psi}(\cdot\,|\,x)$ are continuous at $\psi$, the RB ratio for a value $\psi$ is given by
\begin{equation*}
	RB_{\Psi}(\psi\,|\,x)=\pi_{\Psi}(\psi\,|\,x)/\pi_{\Psi}(\psi). \label{relbel}%
\end{equation*} 
Otherwise for a sequence $N_{\delta}(\psi\,)$, the neighborhoods of $\psi$ that converge
nicely to $\psi$ as $\delta\rightarrow0$, the RB ratio is defined by $RB_{\Psi}(\psi\,|\,x)=\lim_{\delta\rightarrow0}%
\Pi_{\Psi}(N_{\delta}(\psi\,)|\,x)/\Pi_{\Psi}(N_{\delta}(\psi\,)),$ where $\Pi_{\Psi}$ and $\Pi_{\Psi}(\cdot|\,x)$ are the marginal prior and the marginal posterior probability measures, respectively.

Note that $RB_{\Psi}(\psi\,|\,x)$ measures the change in the belief of $\psi$ being the true value \textit{a priori} to \textit{a posteriori}. Therefore, it is a measure of evidence. 
If $RB_{\Psi}(\psi\,|\,x)$ $>1$, then the probability of $\psi$ being the true value from a priori to a posteriori is increased, consequently there is evidence based on the data that $\psi$ is the true value. If $RB_{\Psi}(\psi\,|\,x)<1$, then the probability of  $\psi$ being the true value from a priori to a posteriori is decreased. Accordingly, there is evidence against based on the data that $\psi$ being the true value. For the case $RB_{\Psi}(\psi\,|\,x)=1$ there is no
evidence in either direction. For the null hypothesis $\mathcal{H}_{0}:\Psi(\theta)=\psi_{0}$, it is obvious $RB_{\Psi}(\psi_{0}\,|\,x)$ measures the evidence in favor of or against $\mathcal{H}_{0}$. In this scenario where evidence for the null hypothesis is plausible, the frequentist notion of controlling the probability of falsely rejecting $\mathcal{H}_0$ (type I error) does not apply.

The possibility of calibrating RB ratios is a desirable feature that makes it attractive in hypothesis testing problems. After computing the RB ratio, it is very critical to know whether the obtained value represents strong or weak evidence for
or against $\mathcal{H}_{0}$. A typical
calibration of $RB_{\Psi}(\psi_{0}\,|\,x)$ is given by the \textit{strength of  evidence}
\begin{equation}
	Str_{\Psi}(\psi_0\,|\,x)=\Pi_{\Psi}\left[RB_{\Psi}(\psi\,|\,x)\leq RB_{\Psi}(\psi_{0}\,|\,x)\,|\,x\right].
	\label{strength}%
\end{equation}
The value in \eqref{strength} indicates that the posterior probability that the true value of $\psi$ has a RB ratio no greater than that of the hypothesized value $\psi_{0}.$ When $RB_{\Psi}(\psi_{0}\,|\,x)<1$, there is evidence
against $\psi_{0},$ then a small value of (\ref{strength}) indicates
strong evidence against $\psi_{0}$ because the posterior probability of the true value having RB ratio bigger is large. On the other hand, a large value for \eqref{strength}    indicates weak evidence against $\psi_{0}$.
Similarly, when $RB_{\Psi}(\psi_{0}\,|\,x)>1$, there is  evidence in favor
of $\psi_{0},$ then a small value of (\ref{strength}) indicates  weak
evidence in favor of $\psi_{0}$, while a large value of \eqref{strength} indicates
strong evidence in favor of $\psi_{0}$.

The RB can be considered as a strong alternative to the Bayes factor (BF) criteria. The BF is defined as the ratio of the marginal likelihood of data under the null hypothesis to the alternative hypothesis in Bayesian hypothesis testing problems. However, computing the BF often involves intractable calculations of marginal likelihoods, which typically require computationally burdensome methods such as MCMC.
The tests proposed by \citet{Holmes} and \citet{Borgwardt} are two examples of BNP tests that utilize marginal likelihood computation, and their practical usage in high-dimensional statistics is low due to this computational issue.

On the other hand, the construction of tests using the BF relies on assigning a prior $\pi_0$ to the null hypothesis $\mathcal{H}_0$, a prior $\pi_1$ to the alternative hypothesis $\mathcal{H}_1$, and a discrete probability mass $p_0$ for $\mathcal{H}_0$. However, practitioners often face challenges in eliciting these prior components within the overall prior $\pi = p_{0}\pi_{0} + (1-p_{0})\pi_1$. 
Another concern of using BFs is their calibration to indicate whether weak or strong evidence is attained. For example, \citet{Jeffreys} and \citet{Kass} proposed similar rules to calibrate BFs but \citet{Gonzalo} pointed out that such rules are inappropriate to calibrate BFs as they ignore the randomness of the data and, again, lead to improper inference\footnote{A comprehensive study that explains why the RB ratio is a more appropriate measure of evidence than the BF can also be found in \cite{al2023measure}.}.
\section{Radial Basis Function Kernels Family}\label{RBF-sec}
The construction of MMD-based procedures is proposed based on considering a kernel function with feature space corresponding to a universal RKHS. The radial basis function (RBF) kernel is the most well-known kernel family satisfying the above situation. For two vectors $\mathbf{X},\mathbf{Y}\in \mathbb{R}^d$, the RBF kernel is represented by
\begin{align*}
	k(\mathbf{X},\mathbf{Y})=h(||\mathbf{X}-\mathbf{Y}||/\sigma),
\end{align*}
where, $h$ is a function from the positive real numbers $\mathbb{R}^+$ to $\mathbb{R}^+$, $||\cdot||$ represents the $L^2$-norm, and $\sigma$ is the bandwidth parameter that indicates the kernel size. There are many functions assigned to $h$, for example, the Gaussian, exponential, rational quadratic kernels, and Matern, represented by
\begin{align*}
	h_1(x)=\exp{(-\frac{x^2}{2})},\,h_2(x)=\exp{(-x)},\,
	h_3(x)=\left( 1+\frac{x^2}{2\alpha}\right)^{-\alpha},\,
	h_4(x)=(1+\sqrt{2\nu}x)e^{-\sqrt{2\nu}x},
\end{align*}
respectively; where, $\alpha$ in $h_3$ is a positive-valued scale-mixture parameter, and the $\nu$ in $h_4$ is a parameter that controls the smoothness of the kernel results \citep{zhao2022predicting,genton2001classes}.
%, for instance,  $h_4$ is the exponential kernel when $\nu=0.5$, while it is converging to the Gaussian kernel when $\nu\rightarrow\infty$

One of the simplest kernel functions above is the Gaussian kernel, which is mostly used in machine learning problems and only depends on bandwidth parameter $\sigma$. The Gaussian kernel tends to 0 and 1 when $\sigma\rightarrow 0$ and $\sigma\rightarrow\infty$, respectively. Both situations lead to $\mathrm{MMD}^2$ being zero. Hence, the choice of the parameter $\sigma$ has a crucial effect on the performance of this kernel.  Numerous methods are proposed to choose the value of $\sigma$, however, there is no definitive optimization method for this problem. The median heuristic is one of the first methods used in choosing $\sigma$ empirically and will be denoted in our experimental results by $\sigma_{MH}$. More precisely, for two samples $\lbrace \mathbf{X}_i\rbrace_{i=1}^{n}$ and $\lbrace  \mathbf{Y}_i\rbrace_{i=1}^{m}$, the $\sigma_{MH}$ is considered as the median of $\lbrace||\mathbf{X}_i-\mathbf{Y}_j||^2:\, 1\leq i\leq n, 1\leq j\leq m\rbrace$, which is mostly used in kernel-based tests \citep{scholkopf2002learning}. Selecting $\sigma$ based on maximizing the power of two-sample problems is another strategy considered by  \citet{jitkrittum2016interpretable}. The selection of the MMD bandwidth on held-out data to maximize power was first proposed by \cite{gretton2012optimal} for linear-time estimates and by \cite{sutherland2016generative} for quadratic-time estimates. Recently, bandwidth selection without data splitting has been proposed for quadratic \citep{schrab2021mmd} and linear \citep{schrab2022efficient} MMD estimates.
Regarding the choice of $\sigma$ in kernel-based GANs, a common idea is assigning several fixed values to $\sigma$ and then considering the mixture of their corresponding Gaussian kernel. This strategy has received much attention and shown an acceptable performance in training GANs\footnote{For further details, see \citet{Li} and \cite{li2017mmd}.}. 
\section{Training Evaluation}
\subsection{Traditional Approaches}\label{sec-mcs}
Evaluating the quality of samples generated by GANs is considered to assess the mode collapse problem \citep{Kaifeng}. The inception score, proposed by \citet{salimans2016improved}, is one common tool used to evaluate GANs. Let $\mathbf{Y}$ represent a sample generated by the generator $G_{\boldsymbol{\omega}}$ and $z$ be the label given to $\mathbf{Y}$ by the discriminator. For instance, if $\mathbf{Y}$ can not be distinguished from the real dataset, $z=1$; otherwise, $z=0$. Then, the inception score is given by
\begin{align*}%\label{IS}
	IS&=\exp\big\lbrace E_{\mathbf{Y}}\big[D_{KL}(p(z|\mathbf{Y}),E_{\mathbf{Y}}[p(z|\mathbf{Y})]) \big]\big\rbrace\nonumber\\
	&=\exp\big\lbrace H(E_{\mathbf{Y}}[p(z|\mathbf{Y})])-E_{\mathbf{Y}}(H(p(z|\mathbf{Y})))\big\rbrace
\end{align*}
where $p(z|Y)$ is the probability that $\mathbf{Y}$ takes label $z$ by the discriminator, $D_{KL}(\cdot,\cdot)$ denotes the Kullback-Leibler divergence, and $H(\cdot)$ denotes the entropy. Higher values of $IS$ indicate greater sample diversity. The lowest value of $IS$ is achieved if and only if for any $\mathbf{Y}$ generated by $G_{\boldsymbol{\omega}}$, $p(z|Y)=E_Y[p(z|Y)]$. It means the probability that the discriminator gives label $z$ to $\mathbf{Y}$ is the same, for any $\mathbf{Y}$ generated by the generator.

If a generated sample with low quality, the entropy of $E_Y[p(z|Y)]$ and $p(z|Y)$ can still be, respectively, high and low, which leads to a good inception score. \citet{che2016mode} also mentioned this issue and proposed the mode score function to deal with this issue by
\begin{align}\label{MS1}
	MS&=\exp\big\lbrace E_{\mathbf{Y}}\big[D_{KL}(p(z|\mathbf{Y}),p(z))\big]-D_{KL}(E_{\mathbf{Y}}[p(z|\mathbf{Y})],p(z))  \big\rbrace,
\end{align}
where $p(z)$ is the distribution of labels in the training data. The first part of \eqref{MS1} assesses the quality of the generated sample and the last part deals to assess the variety of the generated sample. The higher values of $MS$ again indicate greater diversity and higher quality for the generated sample. 
%After simplification, Equation \eqref{MS1} can be written as
%\begin{align*}
%    MS&=\exp\big\lbrace H(p(z))-E_{\mathbf{Y}}(H(p(z|\mathbf{Y})))-D_{KL}(E_{\mathbf{Y}}[p(z|\mathbf{Y})],p(z))\big\rbrace.
%\end{align*}
However, \citet{che2016mode} pointed out that the above score does not work well when training datasets are unlabeled. %\citet{Kaifeng} also remarked that for some datasets, $H(p(z))$ might has a high value leading to a perfect $MS$, while the generated sample $\mathbf{Y}$ may not be of good quality. 

Despite using Kullback-Leibler divergence, \citet{Kaifeng} designed a matching score to evaluate the sample qualification as follows. For a real dataset $ U=\lbrace \mathbf{X}_1,\ldots,\mathbf{X}_n \rbrace$, let  $\boldsymbol{\omega}^{\ast}$ be a parameter of $G_{\boldsymbol{\omega}}$ that optimized the desired GAN objective function. Then, for any similarity function $s(\cdot,\cdot)$, the matching score between the real and generated sample is given by
\begin{align}\label{MCS}
	MCS=\frac{1}{n}\max\limits_{t\in \mathcal{T}} \sum_{i=1}^{n} s(\mathbf{X}_i,\mathbf{Y}_{t(i)}(\boldsymbol{\omega}^{\ast})),
\end{align}
where $\mathcal{T}$ is all permutations of $n$ elements in $\lbrace 1,\ldots,n\rbrace$ and $ V=\lbrace\mathbf{Y}_{t(1)}(\boldsymbol{\omega}^{\ast}),\ldots,\mathbf{Y}_{t(n)}(\boldsymbol{\omega}^{\ast})\rbrace$ drawn from the trained generator $G_{\boldsymbol{\omega}^{\ast}}$. A larger matching score guarantees more modes in the generated manifold. Since the computation of $n!$ terms in \eqref{MCS} is time-consuming, \citet{Kaifeng} applied the maximum bipartite matching (MBM) algorithm to find the optimal permutation of realistic samples to the corresponding permutation of the real dataset and then uses the cosine similarity,
\begin{equation*}
	s(\mathbf{X}_i,\mathbf{Y}_{t(i)}(\boldsymbol{\omega}^{\ast}))=\frac{\sum_{j=1}^{d}(\mathbf{X}_{ij}\mathbf{Y}_{t(i)j}(\boldsymbol{\omega}^{\ast}))}{\sqrt{\sum_{j=1}^{d}\mathbf{X}_{ij} \sum_{j=1}^{d}\mathbf{Y}_{t(i)j}(\boldsymbol{\omega}^{\ast})}},  
\end{equation*}
where $\mathbf{Y}_{t(i)}\in \mathbb{R}^d$ and $\mathbf{Y}_{t(i)j}$ denotes the $j$-th element of the vector $\mathbf{Y}_{t(i)}$.
The Ford–Fulkerson (FF), Edmonds–Karp (EK), and Hopcroft–Karp (HK) are among the most famous matching algorithms to compute this permutation \citep{ford1956maximal,edmonds1972theoretical,hopcroft1973n}. A particular consideration that should be taken into account is the running time of these algorithms. For example, the running time of the FF, EK, and HK algorithms are $O(|U\cup V|f)$, $O(|U\cup V||E|^2)$, and $O(\sqrt{|U\cup V|}|E|)$, respectively, where $f$ is the maximum flow in the graph, $E$ is the set of all edges connecting the nodes in the set $U$ to the nodes in the set $V$, and $|\cdot|$ denotes the number of components in the relevant set. %The EK algorithm is an implementation of the FF algorithm for computing the maximum flow in a flow network and it is not possible to say which algorithm is universally faster as it depends on $f$ in the FF algorithm. But it is obvious that the HK is faster than the EK algorithm.

\subsection{An MMD Matching Score Function}
We first revisit 
%begin this section by revisiting 
the MBM method used in the matching score function \eqref{MCS} proposed by \citet{Kaifeng} who argued that considering $n!$ permutations in \eqref{MCS} is time-consuming, an optimal permutation chosen by the MBM algorithm is instead considered to compute $MCS$. To continue the discussion, we need to briefly review some of the main concepts in the bipartite graph theory.

Let a bipartite graph be denoted by $\mathcal{B}=(U,V, E)$, where $E$ is the set of all edges connecting the nodes in the set $U$ to the nodes in the set $V$. A bipartite matching is a subset $E_{MBM}\subseteq E$ for $\mathcal{B}$ such that no edges in $E_{MBM}$ share an endpoint \citep{lovasz1986plummer}. An MBM is a bipartite matching with the maximum number of edges such that if an edge is added to its edges set, the bipartite graph is no longer a matching. It should be noted that more than one maximum matching can exist for a bipartite graph $\mathcal{B}$ and then MBMs are not unique in such graphs \citep{jia2022finding}. For instance, when the number of nodes in sets $U$ and $V$ is the same, there could be $n!$ MBMs for bipartite graph $\mathcal{B}$. 

Now, consider $U$ as the set of the real dataset $ \mathbf{X}_1,\ldots,\mathbf{X}_n $ and $V$ as the set of $ \mathbf{Y}_{1}(\boldsymbol{\omega}^{\ast}),\ldots,$ $\mathbf{Y}_{n}(\boldsymbol{\omega}^{\ast})$, drawn from the trained generator $G_{\boldsymbol{\omega}^{\ast}}$, in the matching score procedure given by Section \ref{sec-mcs}. Since each permutation of nodes in $V$ must be compared to the elements of $U$, there are $n!$ MBMs between $U$ and $V$. 
To be clearer, all MBM graphs are given for $n=3$ by Figure \ref{MBMs}.   
It is worth mentioning that MBM algorithms mentioned in Section \ref{sec-mcs} often randomly output one of $n!$ possible MBMs. Hence, we prefer to use the term ``random permutation'' as opposed to using the term ``optimal permutation" in the procedure proposed by \citet{Kaifeng}. %does not make sense and should be changed to ``random permutation". 
On the other hand, the MBM may not be a particularly informative score to demonstrate the similarity between the two samples. For example, for $i=1,\cdots,n$, let $\mathbf{X}_i$ be a handwritten image for the number $i$. Also, assume that samples  $\mathbf{Y}_i(\boldsymbol{\omega}^{\ast})$'s, produced by the trained generator, have high resolution and great diversity. However, a randomly chosen MBM may connect none of the generated data to its corresponding data, or very few $\mathbf{Y}_i(\boldsymbol{\omega}^{\ast})$ to the corresponding $\mathbf{X}_i(\boldsymbol{\omega}^{\ast})$. In this case, $s(\mathbf{X}_i,\mathbf{Y}_{t(i)}(\boldsymbol{\omega}^{\ast}))$ in \eqref{MCS} might have a low value leading to a poor $MCS$, while the observed generated samples may in fact exhibit good performance in terms of diversity and resolution. 

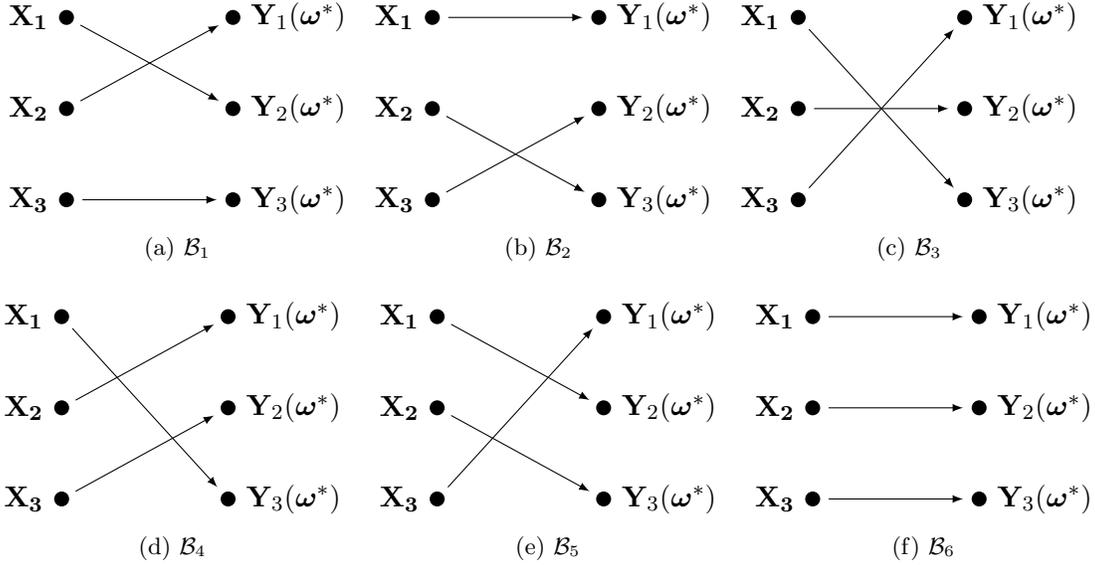
\begin{figure}[ht]\label{MBMs}
	\centering
	\subfloat[$\mathcal{B}_1$]{\begin{tikzpicture}[
			mydot/.style={
				circle,
				fill,
				inner sep=2pt
			},
			>=latex,
			shorten >= 3pt,
			shorten <= 3pt
			]
			\node[mydot,label={left:$\mathbf{X_1}$}] (a1) {}; 
			\node[mydot,below=of a1,label={left:$\mathbf{X_2}$}] (a2) {}; 
			\node[mydot,below=of a2,label={left:$\mathbf{X_3}$}] (a3) {};

			\node[mydot,right=2cm of a1,label={right:$\mathbf{Y}_{1}(\boldsymbol{\omega}^{\ast})$}] (b1) {}; 
			\node[mydot,below=of b1,label={right:$\mathbf{Y}_{2}(\boldsymbol{\omega}^{\ast})$}] (b2) {}; 
			\node[mydot,below=of b2,label={right:$\mathbf{Y}_{3}(\boldsymbol{\omega}^{\ast})$}] (b3) {};

			\path[->] (a1) edge (b2);
			\path[->] (a2) edge (b1);
			\path[->] (a3) edge (b3);
	\end{tikzpicture}}
	%\hfill
	%%%%%%%%% SECOND graph
	\subfloat[$\mathcal{B}_2$]{
		\begin{tikzpicture}[
			mydot/.style={
				circle,
				fill,
				inner sep=2pt
			},
			>=latex,
			shorten >= 3pt,
			shorten <= 3pt
			]
			\node[mydot,label={left:$\mathbf{X_1}$}] (a1) {}; 
			\node[mydot,below=of a1,label={left:$\mathbf{X_2}$}] (a2) {}; 
			\node[mydot,below=of a2,label={left:$\mathbf{X_3}$}] (a3) {};

			\node[mydot,right=2cm of a1,label={right:$\mathbf{Y}_{1}(\boldsymbol{\omega}^{\ast})$}] (b1) {}; 
			\node[mydot,below=of b1,label={right:$\mathbf{Y}_{2}(\boldsymbol{\omega}^{\ast})$}] (b2) {}; 
			\node[mydot,below=of b2,label={right:$\mathbf{Y}_{3}(\boldsymbol{\omega}^{\ast})$}] (b3) {};

			\path[->] (a1) edge (b1);
			\path[->] (a2) edge (b3);
			\path[->] (a3) edge (b2);
	\end{tikzpicture}}
	%%%%%%%%% THIRD graph
	\subfloat[$\mathcal{B}_3$]{
		\begin{tikzpicture}[
			mydot/.style={
				circle,
				fill,
				inner sep=2pt
			},
			>=latex,
			shorten >= 3pt,
			shorten <= 3pt
			]
			\node[mydot,label={left:$\mathbf{X_1}$}] (a1) {}; 
			\node[mydot,below=of a1,label={left:$\mathbf{X_2}$}] (a2) {}; 
			\node[mydot,below=of a2,label={left:$\mathbf{X_3}$}] (a3) {};

			\node[mydot,right=2cm of a1,label={right:$\mathbf{Y}_{1}(\boldsymbol{\omega}^{\ast})$}] (b1) {}; 
			\node[mydot,below=of b1,label={right:$\mathbf{Y}_{2}(\boldsymbol{\omega}^{\ast})$}] (b2) {}; 
			\node[mydot,below=of b2,label={right:$\mathbf{Y}_{3}(\boldsymbol{\omega}^{\ast})$}] (b3) {};

			\path[->] (a1) edge (b3);
			\path[->] (a2) edge (b2);
			\path[->] (a3) edge (b1);
		\end{tikzpicture}
	}
	\\
	%%%%%%%%% FOURTH graph
	\subfloat[$\mathcal{B}_4$]{
		\begin{tikzpicture}[
			mydot/.style={
				circle,
				fill,
				inner sep=2pt
			},
			>=latex,
			shorten >= 3pt,
			shorten <= 3pt
			]
			\node[mydot,label={left:$\mathbf{X_1}$}] (a1) {}; 
			\node[mydot,below=of a1,label={left:$\mathbf{X_2}$}] (a2) {}; 
			\node[mydot,below=of a2,label={left:$\mathbf{X_3}$}] (a3) {};

			\node[mydot,right=2cm of a1,label={right:$\mathbf{Y}_{1}(\boldsymbol{\omega}^{\ast})$}] (b1) {}; 
			\node[mydot,below=of b1,label={right:$\mathbf{Y}_{2}(\boldsymbol{\omega}^{\ast})$}] (b2) {}; 
			\node[mydot,below=of b2,label={right:$\mathbf{Y}_{3}(\boldsymbol{\omega}^{\ast})$}] (b3) {};

			\path[->] (a1) edge (b3);
			\path[->] (a2) edge (b1);
			\path[->] (a3) edge (b2);
		\end{tikzpicture}
	}
	%%%%%%%%% FIFTH graph
	\subfloat[$\mathcal{B}_5$]{
		\begin{tikzpicture}[
			mydot/.style={
				circle,
				fill,
				inner sep=2pt
			},
			>=latex,
			shorten >= 3pt,
			shorten <= 3pt
			]
			\node[mydot,label={left:$\mathbf{X_1}$}] (a1) {}; 
			\node[mydot,below=of a1,label={left:$\mathbf{X_2}$}] (a2) {}; 
			\node[mydot,below=of a2,label={left:$\mathbf{X_3}$}] (a3) {};

			\node[mydot,right=2cm of a1,label={right:$\mathbf{Y}_{1}(\boldsymbol{\omega}^{\ast})$}] (b1) {}; 
			\node[mydot,below=of b1,label={right:$\mathbf{Y}_{2}(\boldsymbol{\omega}^{\ast})$}] (b2) {}; 
			\node[mydot,below=of b2,label={right:$\mathbf{Y}_{3}(\boldsymbol{\omega}^{\ast})$}] (b3) {};

			\path[->] (a1) edge (b2);
			\path[->] (a2) edge (b3);
			\path[->] (a3) edge (b1);
		\end{tikzpicture}
	}
	%%%%%%%%% SIXTH graph
	\subfloat[$\mathcal{B}_6$]{
		\begin{tikzpicture}[
			mydot/.style={
				circle,
				fill,
				inner sep=2pt
			},
			>=latex,
			shorten >= 3pt,
			shorten <= 3pt
			]
			\node[mydot,label={left:$\mathbf{X_1}$}] (a1) {}; 
			\node[mydot,below=of a1,label={left:$\mathbf{X_2}$}] (a2) {}; 
			\node[mydot,below=of a2,label={left:$\mathbf{X_3}$}] (a3) {};

			\node[mydot,right=2cm of a1,label={right:$\mathbf{Y}_{1}(\boldsymbol{\omega}^{\ast})$}] (b1) {}; 
			\node[mydot,below=of b1,label={right:$\mathbf{Y}_{2}(\boldsymbol{\omega}^{\ast})$}] (b2) {}; 
			\node[mydot,below=of b2,label={right:$\mathbf{Y}_{3}(\boldsymbol{\omega}^{\ast})$}] (b3) {};

			\path[->] (a1) edge (b1);
			\path[->] (a2) edge (b2);
			\path[->] (a3) edge (b3);
		\end{tikzpicture}
	}
	\caption{All possible MBMs between real and generated datasets with the same sample size $n=3$.}\label{MBMs}
\end{figure}

Instead of considering only a random MBM, it is more reasonable to consider several bipartite graphs constructed based on resampling from $\lbrace \mathbf{X}_i\rbrace_{i=1}^{n}$ and $\lbrace \mathbf{Y}_{i}(\boldsymbol{\omega}^{\ast})\rbrace_{i=1}^{n}$ with smaller sample sizes than $n$ and then collect a random MBM in each bipartite (mini-batch strategy). In this case, more matchings are considered, which provides more comparison for checking the quality of the generated samples.
However, the implementation of MBM algorithms will be time-consuming and also most of the data information will still be lost due to neglecting to consider all matchings.

To develop a stronger method for evaluating the differences between real and generated data manifolds, we propose using the MMD dissimilarity measure instead of using the cosine similarity measure as follows: For $i=1,\ldots,r_{mb}$, let $\lbrace \mathbf{X}_{i_j}\rbrace_{j=1}^{n_{mb}}$ and $\lbrace \mathbf{Y}_{i_j}(\boldsymbol{\omega}^{\ast})\rbrace_{j=1}^{n_{mb}}$ be two samples drawn, respectively,  from the real dataset $\mathbf{X}_{1},\ldots,\mathbf{X}_{n}$ and the generated dataset $\mathbf{Y}_{1}(\boldsymbol{\omega}^{\ast}),\ldots,$
$\mathbf{Y}_{n}(\boldsymbol{\omega}^{\ast})$  with the same sample size $n_{mb}<n$. 
%For each $i$ above, let $MBM_i$ be a random maximum matching corresponding to the bipartite graph $\mathcal{B}_i(\lbrace \mathbf{X}_{i_j}\rbrace_{j=1}^{n_{mb}}),\lbrace \mathbf{Y}_{i_j}(\boldsymbol{\omega}^{\ast})\rbrace_{j=1}^{n_{mb}},E)$ with $E_{MBM_i}=\left\lbrace \left(t_{HK}(i_1),t^{\prime}_{HK}(i_{1})\right),\ldots,\left(t_{HK}(i_{n_{mb}}),t^{\prime}_{HK}(i_{n_{mb}})\right)\right\rbrace$, such that $\lbrace t_{HK}(i_j)\rbrace_{j=1}^{n_{mb}}$ and $\lbrace t^{\prime}_{HK}(i_j)\rbrace_{j=1}^{n_{mb}}$ are two permutations of elements in $\lbrace i_1,\ldots,i_{n_{mb}} \rbrace$ chosen by the HK algorithm.
Then, we define the MMD-based matching score  as  
\begin{align}\label{MMDS}
	MMDS=\max\limits_{i\in\lbrace 1,\ldots, r_{mb}\rbrace}\mathrm{MMD}^2(F_{n_{mb}}(i),F_{G_{\boldsymbol{\omega}^{\ast}},n_{mb}}(i)),
\end{align}
where, $\mathrm{MMD}^2(F_{n_{mb}}(i),F_{G_{\boldsymbol{\omega}^{\ast}},n_{mb}}(i))$ is the   MMD approximation given by Equation (2, main paper) using samples
$\lbrace \mathbf{X}_{i_j}\rbrace_{j=1}^{n_{mb}}$ and $\lbrace \mathbf{Y}_{i_j}(\boldsymbol{\omega}^{\ast})\rbrace_{j=1}^{n_{mb}}$ (mini-batch samples). 
%Our proposed matching score returns the maximum value of the MMD approximation corresponding to the random MBM between a subset of the real and a subset of the generated dataset with the same size $n_{mb}$ (mini-batch sample size) over $r_{mb}$ resamplings (mini-batch iteration). Therefore, it is obvious smaller values of $MMDS_{HK}$ indicate better quality and more diversity of the generated samples.
Our proposed matching score returns the maximum value of the MMD approximation between a subset of the real and a subset of the generated dataset with the same size $n_{mb}$ (mini-batch sample size) over $r_{mb}$ resamplings (mini-batch iteration). According to Equation Equation (2, main paper), all components of mini-batch samples are compared together in the MMD measure, which provides a comprehensive assessment between subsets of the data in each iteration.
Eventually, it is obvious smaller values of $MMDS$ indicate better quality and more diversity of the generated samples.
\section{Additional Experiments}
\subsection{The Semi-BNP Test}
To further illustrate the difference in performance between the BNP and FNP tests, we conducted tests on two alternative distributions: $F_1=N(0,\sigma^2)$ for $\sigma^2\in[1,4]$ and $F_1=0.5N(-1+\upsilon,1)+0.5N(1-\upsilon,1)$ for $\upsilon\in[0,1]$. The corresponding results are reported in Figure \ref{variousVar} and \ref{variousNMIX} for univariate cases with $n=50$. Figure \ref{variousVar}(a) specifically shows that the proposed test exhibits a higher growth rate of the AUC when $\sigma^2$ is increased compared to the other tests. Additionally, Figure \ref{variousVar}(b) indicates that our test starts to detect differences earlier than other tests ($\sigma^2\geq 1.67$). Similar results can be found in Figure \ref{variousNMIX} for mixture distribution with various means.
%%%%%%%%%%%%%%%%%%%%%%%%
%%%%%%%%%%%%%%%%%%%%%%%%%%%
\begin{figure}[ht]
	\centering
	\subfloat[]{\includegraphics[width=.35\linewidth,height=.35\linewidth]{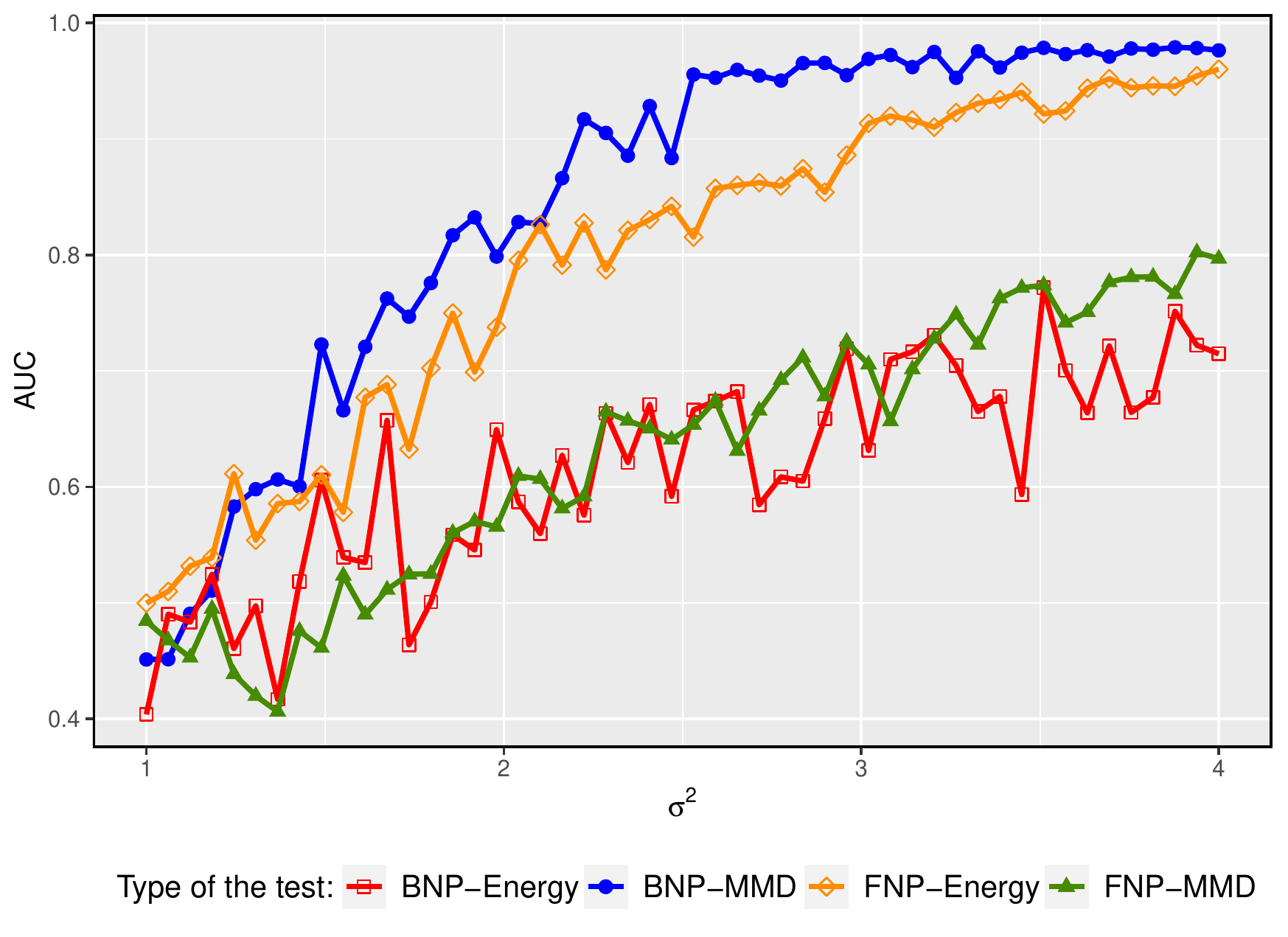}}
	%\hspace{1cm}
	\subfloat[]{\includegraphics[width=.35\linewidth,height=.35\linewidth]{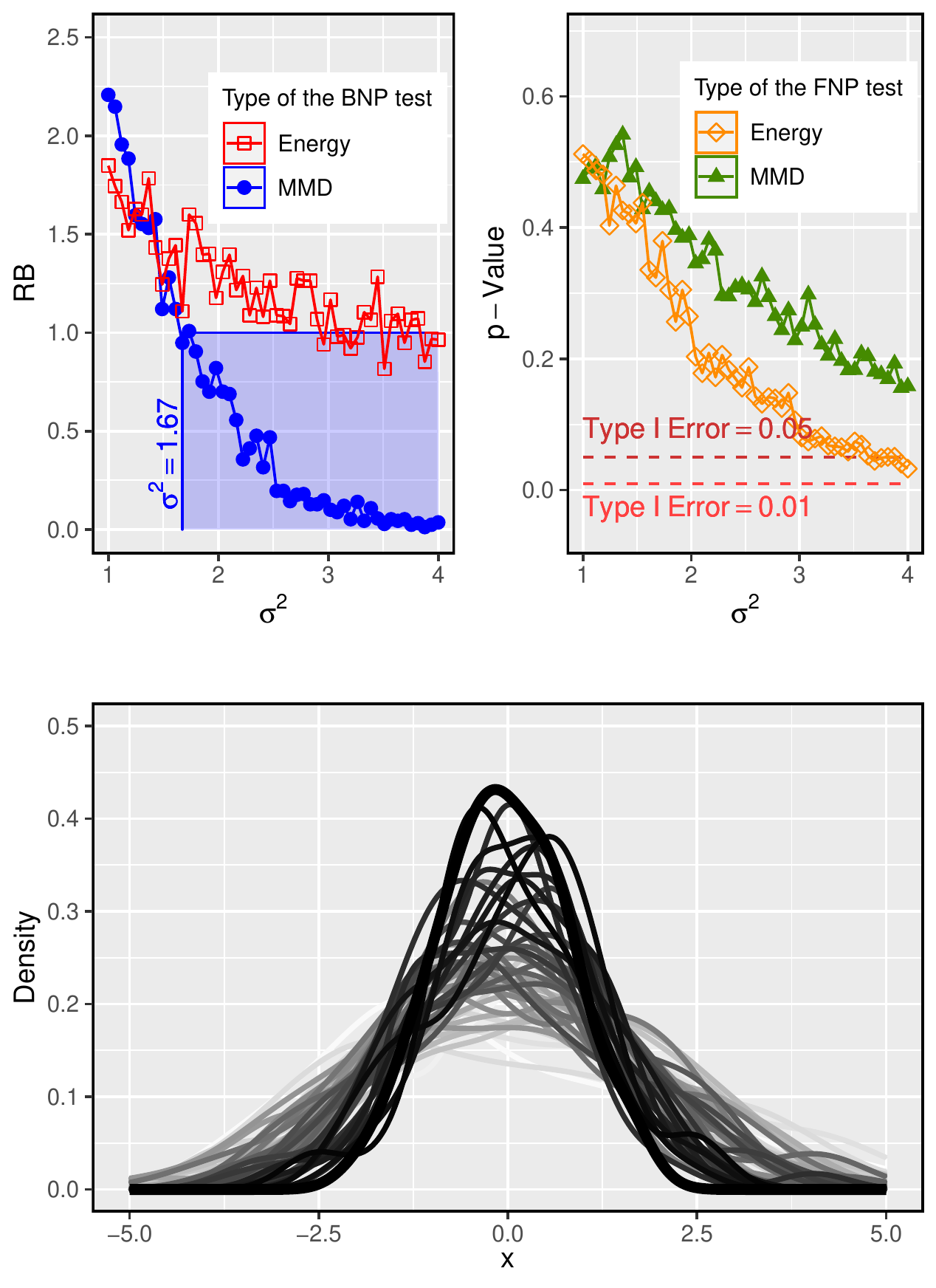}}
	\caption{(a) AUC values in testing alternative $F_1=N(0,\sigma^2)$ for $\sigma^2\in(1,4)$ in variance shift example. (b)-Top: Test critical values against different values of $\sigma^2$. (b)-Bottom: The lighter density corresponds to a larger value of $\sigma^2$.}
	\label{variousVar}%
\end{figure}

%%%%%%%%%%%%%%%%%%%%%%%%%%
%%%%%%%%%%%%%%%%%%%%%%%%%%
\begin{figure}[ht]
	\centering
	\subfloat[]{\includegraphics[width=.35\linewidth,height=.35\linewidth]{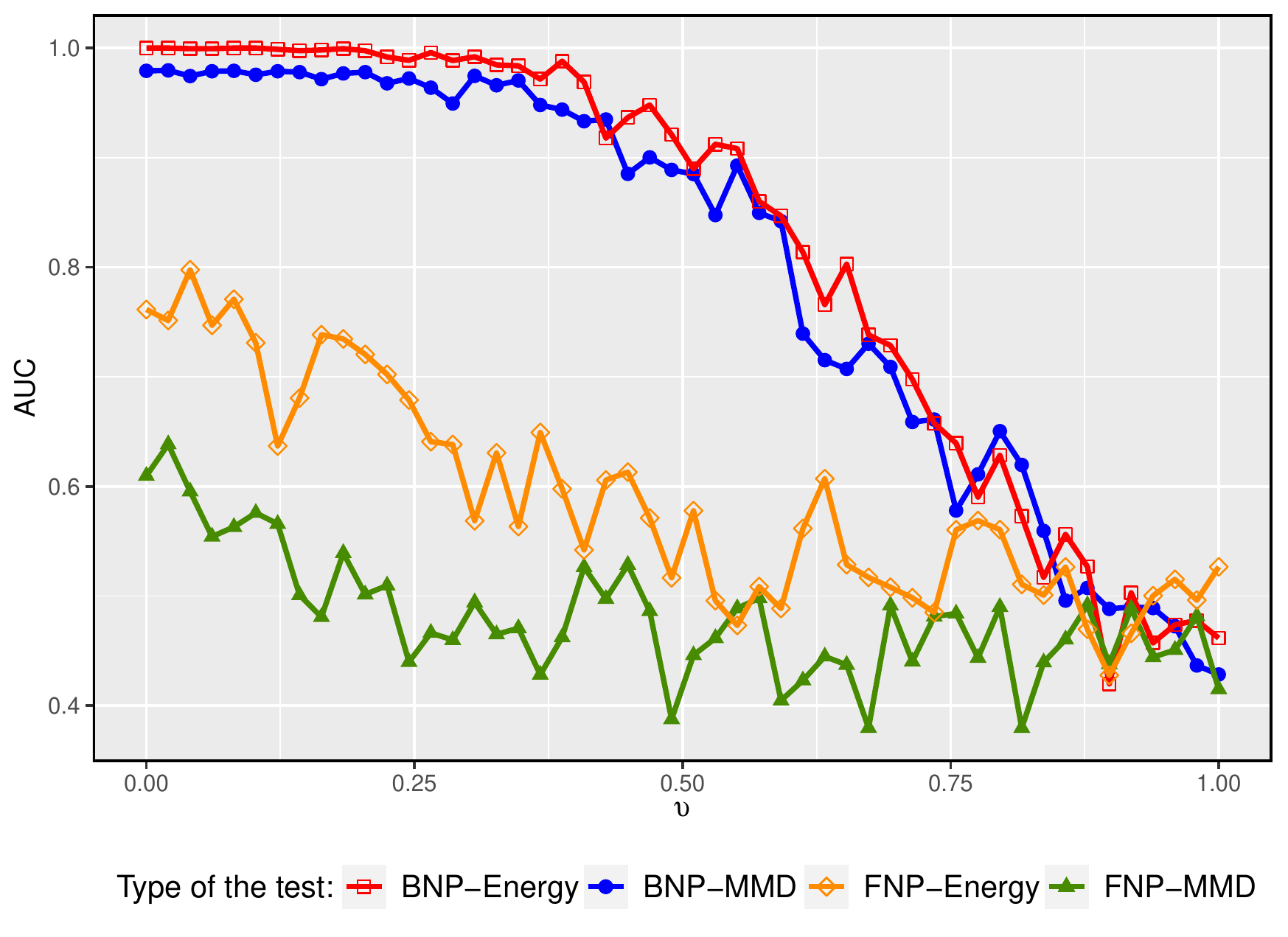}}
	%\hspace{1cm}
	\subfloat[]{\includegraphics[width=.35\linewidth,height=.35\linewidth]{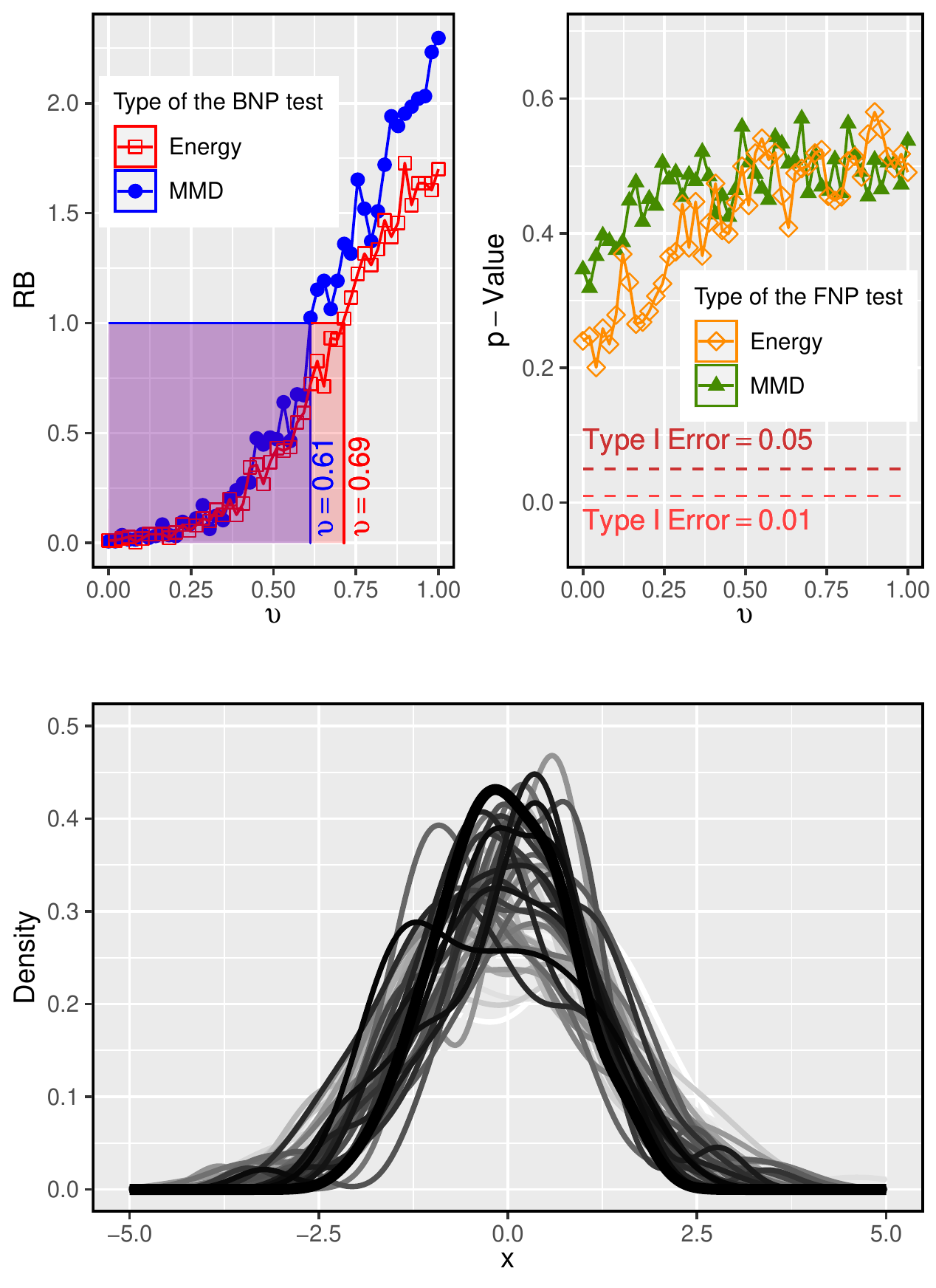}}
	\caption{(a) AUC values in testing alternative $F_1=0.5N(-1+\upsilon,1)+0.5N(1-\upsilon,1)$ for $\upsilon\in(0,1)$ in mixture example. (b)-Top: Test critical values against different values of $\sigma^2$. (b)-Bottom: The lighter density corresponds to a smaller value of $\upsilon$.}
	\label{variousNMIX}%
\end{figure}

%%%%%%%%%%%%%%%%%%%%%%%%%%
Figure \ref{power-comparision} provides a more focused comparison between the semi-BNP test and its Bayesian competitor, the BNP energy test. This figure illustrates the proportion of rejecting $\mathcal{H}_0$ over the 100 samples for both Bayesian tests mentioned, across different data dimensions. The first row of Figure \ref{power-comparision} represents the type I error, while the remaining rows represent the test power. The figure demonstrates the effectiveness of the semi-BNP kernel-based test in detecting differences, especially in scenarios involving variance shift, heavy tail, and kurtosis examples, where the BNP-energy test does not perform optimally in high sample sizes.

Moreover, to conduct a comprehensive analysis of the large sample property of all the tests in comparison, we present Table \ref{examp500-1000} for sample sizes $n=500, 1000$. This table clearly demonstrates the weak performance of the BNP-Energy test in particular scenarios that are currently being mentioned.

%%%%%%%%%%%%%%%%%%%%%%%%%%

\begin{figure}[htp]
	\centering
	{\includegraphics[width=1\linewidth]{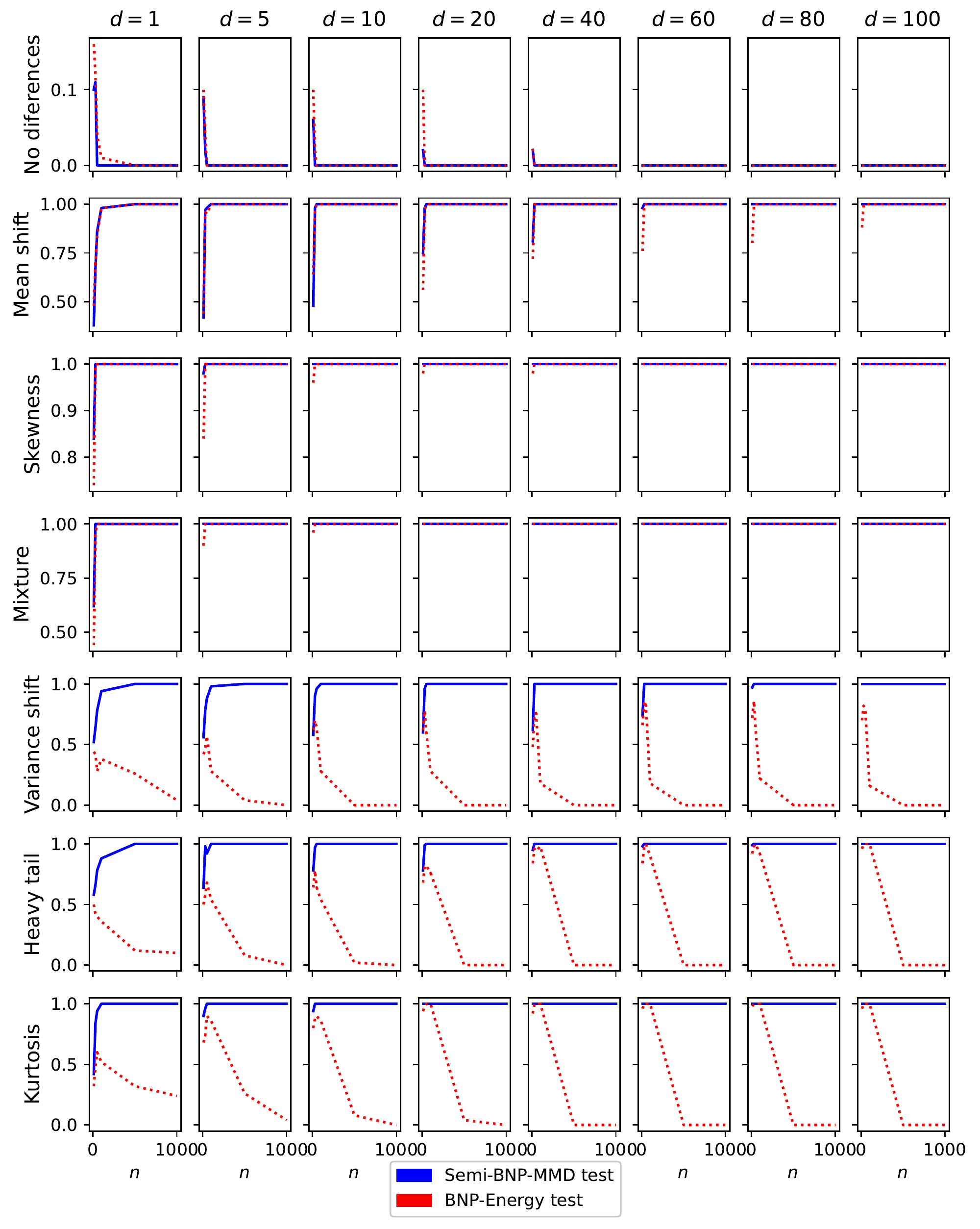}}
	\caption{The proportion of rejecting $\mathcal{H}_0$ out of 100 replications against sample of sizes $n=10,\ldots,1000$ based on using $a=25$, $\ell=1000$, $\epsilon=10^{-3}$ in \eqref{random-stopping}, $M=20$ for the semi-BNP-MMD (blue line) and BNP-energy (red dotted) tests.}
	\label{power-comparision}%
\end{figure}
%%%%%%%%%%%%%%%%%%%%%%%%%%

%%%%%%%%%%%%%%%%%%%%%%%%%%%%%%%%%
%\newpage
\begin{table}[htp] \centering
	\setlength{\tabcolsep}{1 mm}
	\caption{The average of RB, the average of its strength (Str
		), and the relevant AUC
		out of 100 replications based on using $a=25$, $\ell=1000$, $\epsilon=10^{-3}$ in \eqref{random-stopping}, $M=20$, and bandwidth parameter $\sigma=80$ in RBF kernel for two sample of data with $n=500,1000$.}\label{examp500-1000}
	\scalebox{0.65}{
		\begin{tabular}[c]{cl|llll|llll|llll|llll}\hline
			\multirow{5}{*}{Example}&\multirow{5}{*}{$d$}& \multicolumn{8}{c} {BNP} & \multicolumn{8}{c} {FNP}  \\\cmidrule(lr){3-10}\cmidrule(lr){11-18}  
			&&\multicolumn{4}{c}{MMD}&\multicolumn{4}{c}{Energy}&\multicolumn{4}{c}{MMD}&\multicolumn{4}{c}{Energy}\\\cmidrule(lr){3-6}\cmidrule(lr){7-10}\cmidrule(lr){11-14}\cmidrule(lr){15-18}
			&&\multicolumn{2}{c}{RB(Str)}&\multicolumn{2}{c}{AUC}&\multicolumn{2}{c}{RB(Str)}&\multicolumn{2}{c}{AUC}&\multicolumn{2}{c}{P.value}&\multicolumn{2}{c}{AUC}&\multicolumn{2}{c}{P.value}&\multicolumn{2}{c}{AUC}\\\cmidrule(lr){3-4}\cmidrule(lr){5-6}\cmidrule(lr){7-8}\cmidrule(lr){9-10}\cmidrule(lr){11-12}\cmidrule(lr){13-14}\cmidrule(lr){15-16}\cmidrule(lr){17-18}  
			&&\multicolumn{1}{c}{500}&\multicolumn{1}{c}{1000}&\multicolumn{1}{c}{500}&\multicolumn{1}{c}{1000}&\multicolumn{1}{c}{500}&\multicolumn{1}{c}{1000}&\multicolumn{1}{c}{500}&\multicolumn{1}{c}{1000}&\multicolumn{1}{c}{500}&\multicolumn{1}{c}{1000}&\multicolumn{1}{c}{500}&\multicolumn{1}{c}{1000}&\multicolumn{1}{c}{500}&\multicolumn{1}{c}{1000}&\multicolumn{1}{c}{500}&\multicolumn{1}{c}{1000}\\\hline
			No diferences&1&$4.72(0.78)  $&$6.53(0.80)  $  & \multicolumn{2}{!{\hspace*{-0.4pt}\tikzmark{start1}}c!{\tikzmark{end1}}}{} &$3.75(0.60) $&$4.30(0.60)$  & \multicolumn{2}{!{\hspace*{-0.4pt}\tikzmark{start2}}c!{\tikzmark{end2}}}{}&$ 0.52$&$0.50 $ & \multicolumn{2}{!{\hspace*{-0.4pt}\tikzmark{start3}}c!{\tikzmark{end3}}}{}&$0.48$& $ 0.49 $ &\multicolumn{2}{!{\hspace*{-0.4pt}\tikzmark{start4}}c!{\tikzmark{end4}}}{} \\
			&5&$18.84(0.86) $&$ 19.65(0.93)$ &  \multicolumn{2}{!{\hspace*{-0.4pt}\tikzmark{start11}}c!{\tikzmark{end11}}}{} &$18.74(0.88) $ &$19.58(0.76) $ & \multicolumn{2}{!{\hspace*{-0.4pt}\tikzmark{start22}}c!{\tikzmark{end22}}}{} &$0.50 $&$0.51 $ &\multicolumn{2}{!{\hspace*{-0.4pt}\tikzmark{start33}}c!{\tikzmark{end33}}}{}&$0.51 $ &$0.44 $ &\multicolumn{2}{!{\hspace*{-0.4pt}\tikzmark{start44}}c!{\tikzmark{end44}}}{} \\
			&10&$19.98(0.92) $&$20(1) $ & \multicolumn{2}{!{\hspace*{-0.4pt}\tikzmark{start111}}c!{\tikzmark{end111}}}{} &$20(1) $ &$20(1) $ & \multicolumn{2}{!{\hspace*{-0.4pt}\tikzmark{start222}}c!{\tikzmark{end222}}}{} &$0.51 $&$0.50 $ &  \multicolumn{2}{!{\hspace*{-0.4pt}\tikzmark{start333}}c!{\tikzmark{end333}}}{} &$0.53 $&$0.48 $ &  \multicolumn{2}{!{\hspace*{-0.4pt}\tikzmark{start444}}c!{\tikzmark{end444}}}{} \\
			&20&$20(1) $&$20(1) $ &\multicolumn{2}{!{\hspace*{-0.4pt}\tikzmark{start1111}}c!{\tikzmark{end1111}}}{} &$20(1) $ &$20(1) $ & \multicolumn{2}{!{\hspace*{-0.4pt}\tikzmark{start2222}}c!{\tikzmark{end2222}}}{} &$0.53 $&$0.51 $& \multicolumn{2}{!{\hspace*{-0.4pt}\tikzmark{start3333}}c!{\tikzmark{end3333}}}{} &$0.51 $&$0.44 $& \multicolumn{2}{!{\hspace*{-0.4pt}\tikzmark{start4444}}c!{\tikzmark{end4444}}}{} \\
			&40&$20(1) $&$20(1) $ & \multicolumn{2}{!{\hspace*{-0.4pt}\tikzmark{start11111}}c!{\tikzmark{end11111}}}{} &$20(1) $ &$20(1) $ & \multicolumn{2}{!{\hspace*{-0.4pt}\tikzmark{start22222}}c!{\tikzmark{end22222}}}{} &$0.45 $&$0.52 $ & \multicolumn{2}{!{\hspace*{-0.4pt}\tikzmark{start33333}}c!{\tikzmark{end33333}}}{}&$0.51 $ &$0.50 $ &\multicolumn{2}{!{\hspace*{-0.4pt}\tikzmark{start44444}}c!{\tikzmark{end44444}}}{} \\ 
			&60&$20(1) $&$20(1)$&\multicolumn{2}{!{\hspace*{-0.4pt}\tikzmark{start111111}}c!{\tikzmark{end111111}}}{} &$20(1)$&$20(1) $&\multicolumn{2}{!{\hspace*{-0.4pt}\tikzmark{start222222}}c!{\tikzmark{end222222}}}{}&$0.51 $ &$0.50 $ &\multicolumn{2}{!{\hspace*{-0.4pt}\tikzmark{start333333}}c!{\tikzmark{end333333}}}{} &$0.50 $&$0.53 $& \multicolumn{2}{!{\hspace*{-0.4pt}\tikzmark{start444444}}c!{\tikzmark{end444444}}}{} \\ 
			&80&$ 20(1)$&$20(1) $&\multicolumn{2}{!{\hspace*{-0.4pt}\tikzmark{start1111111}}c!{\tikzmark{end1111111}}}{} &$20(1) $ &$20(1) $ &\multicolumn{2}{!{\hspace*{-0.4pt}\tikzmark{start2222222}}c!{\tikzmark{end2222222}}}{} &$0.49 $&$0.48 $  &\multicolumn{2}{!{\hspace*{-0.4pt}\tikzmark{start3333333}}c!{\tikzmark{end3333333}}}{} &$0.54 $&$0.49 $ &\multicolumn{2}{!{\hspace*{-0.4pt}\tikzmark{start4444444}}c!{\tikzmark{end4444444}}}{} \\
			&100&$20(1) $&$20(1) $ &\multicolumn{2}{!{\hspace*{-0.4pt}\tikzmark{start11111111}}c!{\tikzmark{end11111111}}}{} &$20(1) $ &$20(1) $ &\multicolumn{2}{!{\hspace*{-0.4pt}\tikzmark{start22222222}}c!{\tikzmark{end22222222}}}{} &$0.49 $&$0.48 $ &\multicolumn{2}{!{\hspace*{-0.4pt}\tikzmark{start33333333}}c!{\tikzmark{end33333333}}}{} &$0.51 $&$0.50 $&\multicolumn{2}{!{\hspace*{-0.4pt}\tikzmark{start44444444}}c!{\tikzmark{end44444444}}}{} \\\hline
			Mean shift&1&$0(0) $&$0(0) $&$ 1$ &$ 1$ &$0(0) $ &$0(0) $ &$0.98 $&$0.98 $ &$0.001 $&$0.001 $ &$1 $&$1 $ &$0.004  $&$0.004 $&$1 $ &$1 $  \\
			&5&$0(0) $&$0(0) $ &$1 $&$ 1$ &$0(0) $ &$0(0) $ &$1 $&$1 $ &$0.001 $&$0.001 $&$1 $ &$1 $&$0.004 $ &$0.004 $& $1  $&$1 $  \\
			&10&$0(0) $&$0(0) $ &$1 $ &$1 $ & $0(0) $& $0(0) $& $1  $&$1  $&$0.001  $&$0.001 $& $1 $& $1 $&$0.004 $&$0.004 $&$1 $ &$1 $  \\
			&20&$0(0) $&$0(0) $ &$1 $ &$1 $ &$0(0) $ &$0(0) $ &$1 $&$1 $ &$0.001 $&$0.001 $&$1 $ &$1 $&$0.004 $ &$0.004 $&$1 $ &$1 $  \\
			&40&$0(0) $&$0(0) $ &$1 $ &$1 $ &$0(0) $ & $0(0) $&$1 $&$1 $ &$0.001  $& $0.001 $&$1 $& $1 $&$0.004 $& $0.004 $&$1 $&$1 $  \\ 
			&60&$0(0) $&$0(0) $ &$1 $ &$1 $ &$0(0) $&$0(0) $ &$1 $ &$1 $ &$ 0.001 $&$0.001 $&$1 $ &$1 $& $0.004 $&$0.004 $&$1 $ &$1 $  \\ 
			&80&$0(0)  $&$0(0) $ &$1 $ &$1 $ &$0(0)  $ & $0(0)  $&$ 1 $ &$1 $ &$0.001  $ &$0.001  $&$1  $& $1  $&$0.004 $& $0.004 $&$1 $ &$1 $ \\
			&100&$0(0)  $&$0(0) $ &$1 $ &$1 $ &$0(0) $ &$0(0) $ &$1 $ &$1 $ &$0.001  $& $0.001 $&$1 $&$1 $ &$0.004 $&$0.004 $&$1 $ &$1 $  \\\hline
			Skewness&1&$0(0) $&$0(0) $ &$1 $ &$1 $ &$0(0) $ &$0(0) $ &$1 $&$1 $  &$0.001 $&$0.001 $&$1 $ &$1  $ &$0.004 $&$0.004 $ &$1  $&$1 $  \\
			&5&$0(0) $&$0(0) $ &$1 $ &$1 $ &$0(0) $ &$0(0) $ &$1 $&$1 $ &$0.001 $ &$0.001 $ &$1 $&$1 $ &$0.004 $&$0.004 $ &$1 $&$1 $  \\
			&10&$0(0) $&$0(0) $ &$1 $ &$1 $ &$0(0) $ &$0(0) $ &$1 $ &$1 $ &$0.001 $&$0.001  $&$1 $ &$1  $ &$0.004 $&$0.004  $&$1 $ &$1 $  \\
			&20&$0(0) $&$0(0) $ &$1 $ &$1 $ &$0(0) $ &$0(0) $ &$1 $ &$1 $ &$0.001 $&$0.001 $ &$1 $&$1 $&$0.004 $ &$0.004 $ &$1 $&$1 $  \\
			&40&$0(0) $&$0(0) $ &$1 $ &$1 $ &$0(0) $ &$0(0) $ &$1 $ &$1 $ &$0.001 $&$0.001 $ &$1 $&$1  $ &$0.004 $&$0.004  $ &$1 $&$1 $  \\ 
			&60&$0(0) $&$0(0) $ &$1  $ &$1 $ &$0(0) $&$0(0) $ &$ 1 $ &$1  $&$0.001 $ &$0.001  $ &$1 $&$1 $ &$0.004 $&$0.004 $ &$1 $&$1 $  \\ 
			&80&$0(0) $&$0(0) $ &$1 $ &$1 $ &$0(0) $ &$0(0) $ &$1 $&$1 $ &$0.001 $ &$0.001 $ &$1 $&$1 $ &$0.004 $&$0.004 $&$1 $ &$1  $  \\
			&100&$0(0) $&$0(0) $ &$1 $ &$1 $ &$0(0) $ &$0(0) $ &$1 $&$1 $&$0.001 $ &$0.001  $&$1 $ &$1 $&$0.004 $ &$0.004 $&$1 $ &$1 $  \\\hline
			Mixture&1&$0(0) $&$0(0) $ &$1 $ &$1 $ &$0(0) $ &$0(0) $ &$1  $&$1 $&$0.06 $ &$0.01 $&$0.93 $ &$0.99 $&$0.004 $ &$0.004 $&$1 $ &$1 $  \\
			&5&$0(0) $&$0(0) $ & $1 $& $1  $&$0(0) $ &$0(0) $ &$1 $ &$ 1$&$0.001 $&$0.001 $ &$1 $&$1 $ &$0.004 $&$0.004  $ &$ 1$&$1  $  \\
			&10&$0(0) $&$0(0)  $ &$1 $ &$1 $ &$0(0) $ &$0(0) $ &$1  $&$1 $&$0.001 $ &$0.001 $ &$1 $&$1 $ &$0.004 $&$0.004 $&$1 $ &$1 $ \\
			&20&$0(0) $&$0(0) $ &$1  $ &$1 $ &$0(0) $ &$0(0) $ &$1 $&$1 $&$0.001 $ &$0.001  $&$1 $ &$1 $&$0.004 $ &$0.004 $ &$1 $&$1  $\\
			&40&$0(0) $&$0(0) $ &$1 $ &$1 $ &$0(0) $ &$0(0) $ &$1 $&$1 $&$0.001 $ &$0.001  $& $1 $&$1 $&$0.004 $ &$0.004  $&$1 $ &$1  $\\ 
			&60&$0(0) $&$0(0) $ &$1  $ &$1 $ &$0(0) $&$0(0) $ &$1 $ &$1 $ &$0.001 $&$0.001 $ &$1 $&$1 $ &$0.004 $&$0.004 $&$1 $ &$1 $\\
			&80&$0(0) $&$0(0) $ &$1 $ &$1 $ &$0(0) $ &$0(0) $ &$1 $&$1 $&$0.001 $ &$0.001  $& $1 $&$1 $ &$0.004 $&$0.004 $& $1 $&$1 $\\
			&100&$0(0) $&$0(0) $ & $1 $&$1 $ &$0(0) $ &$0(0) $ &$1 $&$1 $& $0.001 $&$0.001  $ &$1 $&$1  $ &$0.004 $&$0.004  $&$1 $ &$1 $\\\hline
			Variance shift&1&$0.01(0) $& $0(0) $&$1 $ &$1 $ &$1.73(0.59) $ &$2.10(0.59) $ &$0.93 $&$0.81 $ &$0.07 $&$0.01 $ &$0.93 $&$0.99 $ &$ 0.006$&$ 0.004 $ &$0.99 $&$1 $\\ 
			&5&$0.42(0.07) $&$0.40(0.08) $ &$0.99 $ &$1 $ &$4.42(0.72) $ &$7.30(0.70)  $ &$0.73 $&$0.64 $&$0.001 $ &$0.001 $ &$1 $&$1 $ &$0.004 $&$0.004 $ &$1 $&$1 $\\
			&10&$0.39(0.06) $&$0.22(0.06) $ &$1 $ &$1 $ &$8.69(0.66) $ &$13.12(0.73) $ &$0.55 $ &$0.40 $&$0.001 $&$0.001 $ &$1 $&$1 $&$0.004 $ &$0.004 $&$1 $ &$1  $\\
			&20&$0(0) $&$0(0) $ &$1 $ &$1 $ &$13.43(0.78) $ &$ 18.12(0.69)$ &$0.35  $ &$0.07 $&$0.001 $&$0.001 $ &$1 $&$1 $&$0.004 $ &$0.004 $ &$1 $&$1 $\\
			&40&$0(0) $&$0(0) $ &$1  $ &$1 $ &$18.01(0.68) $ &$19.82(0.68) $ &$0.11  $&$0 $ &$0.001  $ &$0.001 $&$1 $ &$1 $ &$0.004 $&$0.004 $ &$1 $&$1 $\\ 
			&60&$0(0) $&$0(0) $ &$1 $ &$1 $ &$19.19(0.55) $&$19.98(0.94) $ &$0.02 $ &$0  $ &$0.001 $&$0.001 $ &$1 $&$1 $ &$0.004 $&$0.004 $&$1 $ &$1 $\\
			&80&$0(0) $&$0(0) $ &$1  $ &$1 $ &$19.64(0.47) $ &$20(1) $ &$0 $&$0 $ & $0.001 $&$0.001 $ &$1 $&$1 $ &$0.004 $&$0.004 $ &$1 $&$1 $\\
			&100&$0(0) $& $0(0) $&$1  $ &$1  $ &$19.82(0.64) $ &$20(1)  $ &$0  $&$0 $ &$0.001 $ &$0.001 $& $1 $&$1 $ &$0.004 $&$0.004  $ &$1 $&$1  $ \\\hline
			Heavy tail&1&$0.05(0) $&$0(0) $ &$1 $ &$1 $ &$1.65(0.54) $ &$1.70(0.54) $ &$0.96 $ &$0.99 $&$0.03 $ &$0.004 $&$0.96 $ &$0.99 $&$0.01  $ &$0.005 $&$0.98 $ & $ 0.99$\\
			&5&$0.04(0) $&$0.02(0) $ &$1 $ &$1 $ &$2.89(0.71) $ &$ 4.53(0.74) $ &$0.91  $&$0.76 $& $0.001 $&$0.001 $ &$1 $&$1 $ &$0.004 $&$0.004  $&$1 $ &$1  $\\
			&10&$0(0) $&$0(0)  $ &$1 $ &$1 $ &$4.49(0.78) $ &$7.87(0.73) $ &$0.78 $&$0.64 $ &$0.001 $ &$0.001 $&$1 $ &$1 $& $0.004 $&$0.004 $&$1 $ &$1 $\\
			&20&$0(0) $&$0(0) $ &$1 $ &$1 $ &$5.66(0.76) $ &$11.73(0.75) $ &$0.77 $&$0.42  $&$0.001 $ &$0.001 $&$1 $ &$1 $ &$0.004 $&$0.004 $&$1 $ &$1 $\\
			&40&$0(0) $&$0(0) $ &$1 $ &$1 $ &$9.40(0.79) $ &$16.41(0.78) $ &$0.54 $ &$0.20 $ &$0.001 $&$0.001 $ &$1 $&$1 $ &$0.004 $&$0.004 $&$1 $ &$1 $\\ 
			&60&$0(0) $&$0(0) $ &$1 $ &$1 $ &$11.02(0.74) $&$18.06(0.82) $ &$0.52 $ &$0.16  $ &$0.001 $&$0.001 $ &$1 $&$1 $ &$0.004 $&$0.004 $&$1 $ &$1 $\\
			&80&$0(0) $&$0(0) $ &$1 $ &$1 $ &$12.53(0.77) $ &$18.51(0.90) $ &$0.41 $ &$0.09 $ &$0.001 $&$0.001 $ &$1 $&$1 $ &$0.004 $&$0.004 $ &$1 $&$1 $\\
			&100&$0(0) $&$0(0) $ & $1  $&$1  $ &$13.17(0.75) $ &$19.07(0.97) $ &$0.30 $&$0.06$ & $0.001 $&$0.001 $&$1 $ &$1  $& $0.004 $&$0.004  $&$1 $ &$ 1 $\\\hline
			Kurtosis&1&$0(0) $&$0(0) $ &$1 $ &$1 $ &$1.23(0.42) $ &$1.55(0.52) $ &$0.96 $ &$0.95 $&$0.002 $&$0.001 $ &$0.99 $&$1 $ &$0.004  $&$0.004 $ &$1 $&$1 $\\
			&5&$0(0)  $&$0(0)  $ &$1  $ &$1 $ &$1.75(0.59)  $ &$3.54(0.70)  $ &$0.96 $ &$0.88 $&$0.001  $&$0.001 $ &$1  $&$1 $ &$0.004  $&$0.004  $&$1  $ &$1  $\\
			&10&$0(0) $&$0(0) $ &$1 $ &$1 $ &$2.81(0.66) $ &$6.41(0.76) $ &$0.94 $ &$0.75 $ &$0.001 $&$0.001 $ &$1 $&$1  $ &$0.004 $&$0.004  $ &$1 $&$1  $\\
			&20&$0(0)  $&$0(0) $ &$1 $ &$1 $ &$4.63(0.71)  $ &$9.90(0.78) $ &$0.84 $&$0.51 $ &$0.001  $ &$0.001 $ &$1 $&$1  $ &$0.004 $&$0.004 $ &$ 1 $&$ 1 $\\
			&40&$0(0)  $&$0(0)  $ &$1  $ &$1  $ &$5.70(0.73)  $ &$13.43(0.77) $ &$0.74 $ &$0.28 $ &$0.001  $&$0.001 $ &$1 $&$1 $ &$0.004 $&$0.004 $ &$1 $&$1 $\\ 
			&60&$0(0) $&$0(0) $ &$1 $ &$1 $ &$7.06(0.75) $&$16.38(0.81) $ &$0.72 $ &$0.23  $ &$0.001  $&$0.001 $ &$1 $&$1 $ &$0.004 $&$0.004 $ &$1 $&$1 $\\
			&80&$0(0) $&$0(0) $ &$1 $ &$1 $ &$8.11(0.79) $ &$17.50(0.83) $ &$0.71 $ &$0.13 $ &$0.001  $&$0.001  $ &$1  $&$1  $ &$0.004 $&$0.004 $ &$1 $&$1 $\\
			&100&$0(0) $&$0(0) $ &$1 $ &$ 1$ &$8.83(0.78) $ &$18.52(0.89) $ &$ 0.55$ & $ 0.09$&$ 0.001 $&$ 0.001$ &$1 $&$ 1$ &$0.004 $&$0.004 $ &$1 $&$ 1$\\
			\\\bottomrule
		\end{tabular}
		\HatchedCell{start1}{end1}{%
			pattern color=black!70,pattern=north east lines}
		\HatchedCell{start2}{end2}{%
			pattern color=black!70,pattern=north east lines}
		\HatchedCell{start3}{end3}{%
			pattern color=black!70,pattern=north east lines}
		\HatchedCell{start4}{end4}{%
			pattern color=black!70,pattern=north east lines}
		%%%%%%%%%%%%%%%
		\HatchedCell{start11}{end11}{%
			pattern color=black!70,pattern=north east lines}
		\HatchedCell{start22}{end22}{%
			pattern color=black!70,pattern=north east lines}
		\HatchedCell{start33}{end33}{%
			pattern color=black!70,pattern=north east lines}
		\HatchedCell{start44}{end44}{%
			pattern color=black!70,pattern=north east lines}
		%%%%%%%%%%%%%%%%%
		\HatchedCell{start111}{end111}{%
			pattern color=black!70,pattern=north east lines}
		\HatchedCell{start222}{end222}{%
			pattern color=black!70,pattern=north east lines}
		\HatchedCell{start333}{end333}{%
			pattern color=black!70,pattern=north east lines}
		\HatchedCell{start444}{end444}{%
			pattern color=black!70,pattern=north east lines}
		%%%%%%%%%%%%%%%
		\HatchedCell{start1111}{end1111}{%
			pattern color=black!70,pattern=north east lines}
		\HatchedCell{start2222}{end2222}{%
			pattern color=black!70,pattern=north east lines}
		\HatchedCell{start3333}{end3333}{%
			pattern color=black!70,pattern=north east lines}
		\HatchedCell{start4444}{end4444}{%
			pattern color=black!70,pattern=north east lines}
		%%%%%%%%%%%%%%%
		\HatchedCell{start11111}{end11111}{%
			pattern color=black!70,pattern=north east lines}
		\HatchedCell{start22222}{end22222}{%
			pattern color=black!70,pattern=north east lines}
		\HatchedCell{start33333}{end33333}{%
			pattern color=black!70,pattern=north east lines}
		\HatchedCell{start44444}{end44444}{%
			pattern color=black!70,pattern=north east lines}
		%%%%%%%%%%%%%%
		\HatchedCell{start111111}{end111111}{%
			pattern color=black!70,pattern=north east lines}
		\HatchedCell{start222222}{end222222}{%
			pattern color=black!70,pattern=north east lines}
		\HatchedCell{start333333}{end333333}{%
			pattern color=black!70,pattern=north east lines}
		\HatchedCell{start444444}{end444444}{%
			pattern color=black!70,pattern=north east lines}
		%%%%%%%%%%%%%%
		\HatchedCell{start1111111}{end1111111}{%
			pattern color=black!70,pattern=north east lines}
		\HatchedCell{start2222222}{end2222222}{%
			pattern color=black!70,pattern=north east lines}
		\HatchedCell{start3333333}{end3333333}{%
			pattern color=black!70,pattern=north east lines}
		\HatchedCell{start4444444}{end4444444}{%
			pattern color=black!70,pattern=north east lines}
		%%%%%%%%%%%%%%%%%
		\HatchedCell{start11111111}{end11111111}{%
			pattern color=black!70,pattern=north east lines}
		\HatchedCell{start22222222}{end22222222}{%
			pattern color=black!70,pattern=north east lines}
		\HatchedCell{start33333333}{end33333333}{%
			pattern color=black!70,pattern=north east lines}
		\HatchedCell{start44444444}{end44444444}{%
			pattern color=black!70,pattern=north east lines}
	}
\end{table}%

%%%%%%%%%%%%%%%%%%%%%%%%%%%%%%%%%
%\newpage 
\subsection{The Semi-BNP GAN}
Now, we examine the performance of the proposed GAN through additional datasets, the details of which are given below. The generated samples are shown in Figures \ref{extra-dataset}. Generally, the generated images using semi-BNP GAN show better resolution than the FNP GAN. The MMD scores presented in Table \ref{MMDS-r} are also evidence to demonstrate this claim. To further assess the performance of MMD-based GANs, we report the commonly used Fr\'{e}chet inception distance (FID) and the Kernel inception distance (KID) metrics \citep{binkowski2018demystifying}. These metrics are well-suited for evaluating the performance of GANs. The corresponding scores\footnote{The codes to compute the KID and FID are available at \url{https://github.com/mbinkowski/MMD-GAN/blob/master/gan/compute_scores.py}.} are reported in Table \ref{MMDS-r}. Similar to our MMD scores, the smaller values of FID and KID show better performance of the GAN.

\subsubsection{Bone Marrow Biopsy Dataset \citep{tomczak2016improving}:}
The bone marrow biopsy (BMB) dataset is a collection of histopathology of BMB images corresponding to 16 patients with some types of blood cancer and anemia: 10 patients for training, 3 for testing, and 3 for validation. This dataset contains 10,800 images in the size of $28\times28$ pixels, 6,800 of which are considered for the training set. The rest of the images have been divided into two sets of equal size for testing and validation. The whole dataset can be found at \url{https://github.com/jmtomczak/vae_householder_flow/tree/master/datasets/histopathologyGray}.
The results based on 6800 training images are presented in Figure \ref{extra-dataset}-(a-c).
\subsubsection{Labeled Faces in the Wild Dataset
	\citep{huang2008labeled}:}
The labeled faces in the wild dataset (LFD) include 13,000 facial image samples with 1,024 ($32\times32$) dimensions. The dataset is available at \url{https://conradsanderson.id.au/lfwcrop/}.
\subsubsection{Brain Tumor MRI Dataset
	\citep{msoudnickparvar2021}:}
In the last experiment, we consider a more challenging medical dataset including brain MRI images available at \url{https://www.kaggle.com/dsv/2645886}.
This dataset has two groups including training and testing sets.
Both are classified into four classes: glioma, meningioma, no tumor, and pituitary. To train the networks, we consider all 5,712 training images. The images vary in size and have extra margins. We use a pre-processing code\footnote{\url{https://github.com/masoudnick/Brain-Tumor-MRI-Classification/blob/main/Preprocessing.py}} to remove margins and then resize images to $50\times50$ pixels. We also scale the pixel value of prepared images to range 0-1 to make the range of distribution of feature values equal and prevent any errors in the backpropagation computation.

%%%%%%%%%%%%%%%%%%%%%%%%%%%
\begin{figure}[h]
	\centering
	\subfloat[Training data]{\includegraphics[width=.22\linewidth]{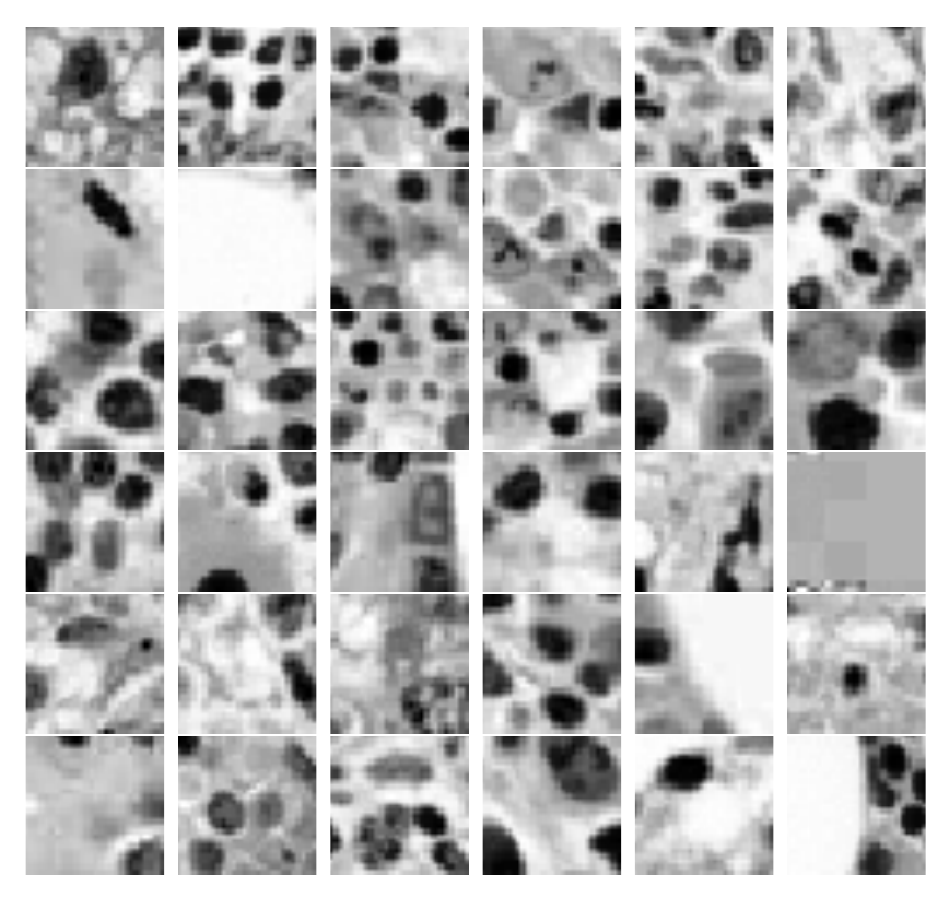}}
	\hspace{1cm}
	\subfloat[Semi-BNP-MMD GAN ]{\includegraphics[width=.22\linewidth]{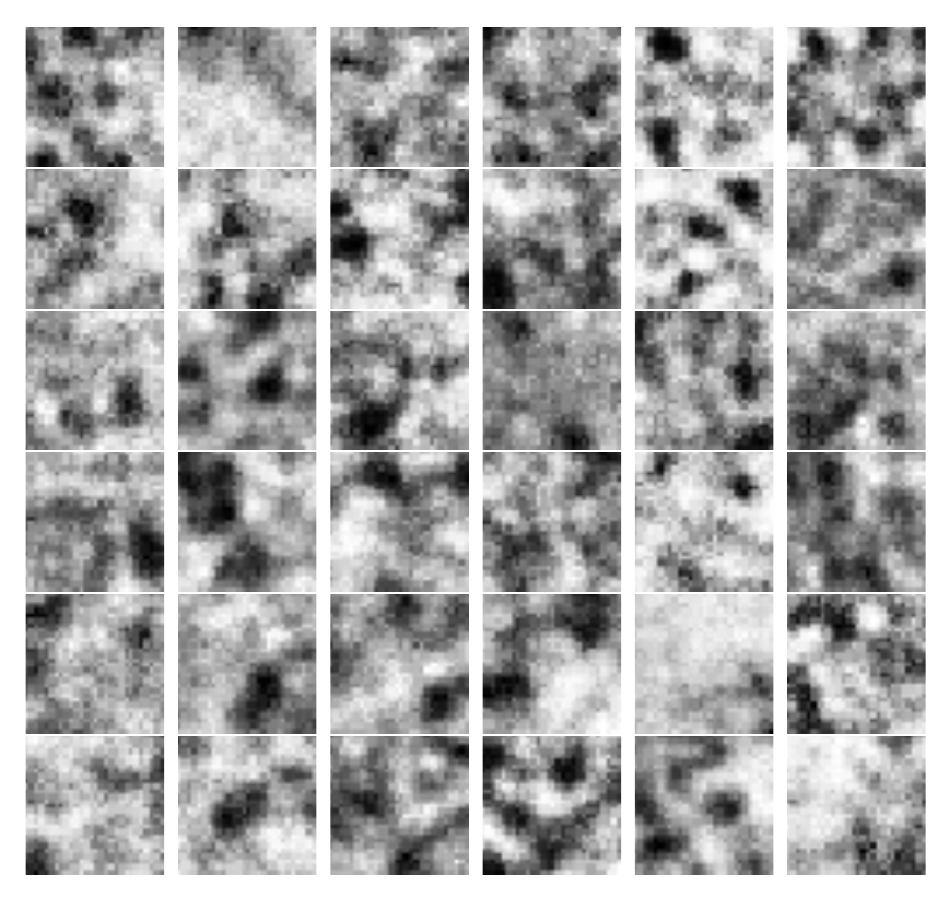}}
	\hspace{1cm}
	\subfloat[MMD-FNP GAN ]{\includegraphics[width=.22\linewidth]{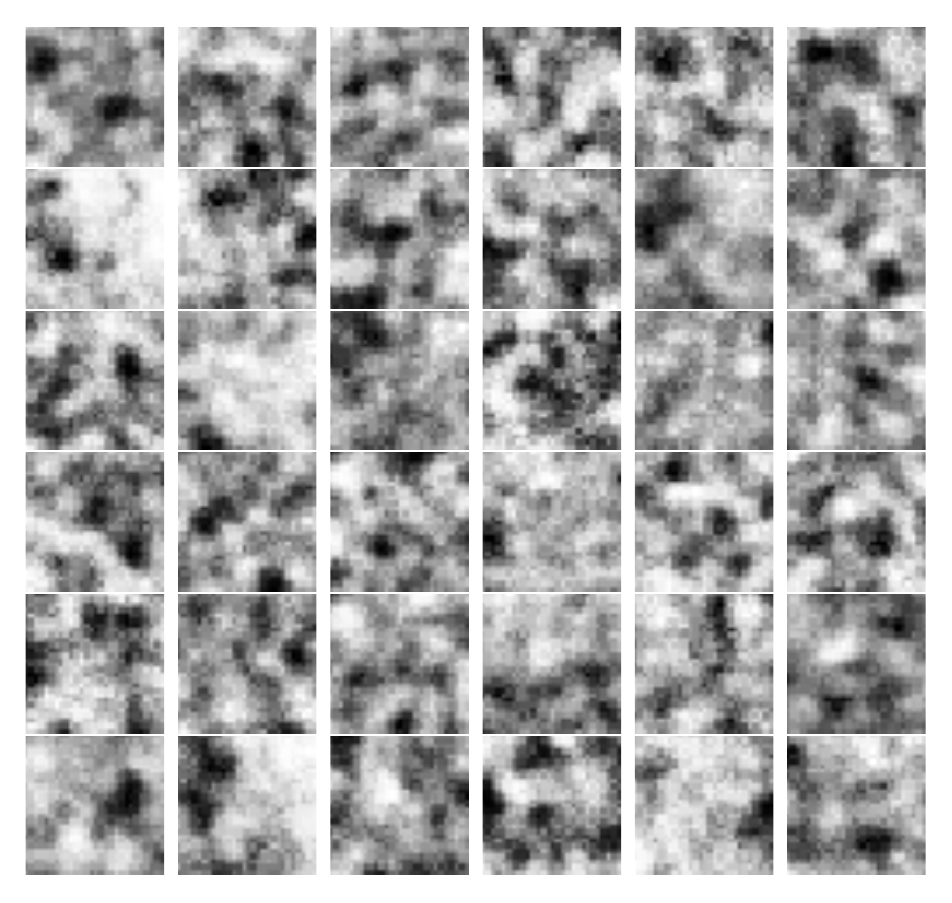}}
	\qquad
	\subfloat[Training data]{\includegraphics[width=.22\linewidth]{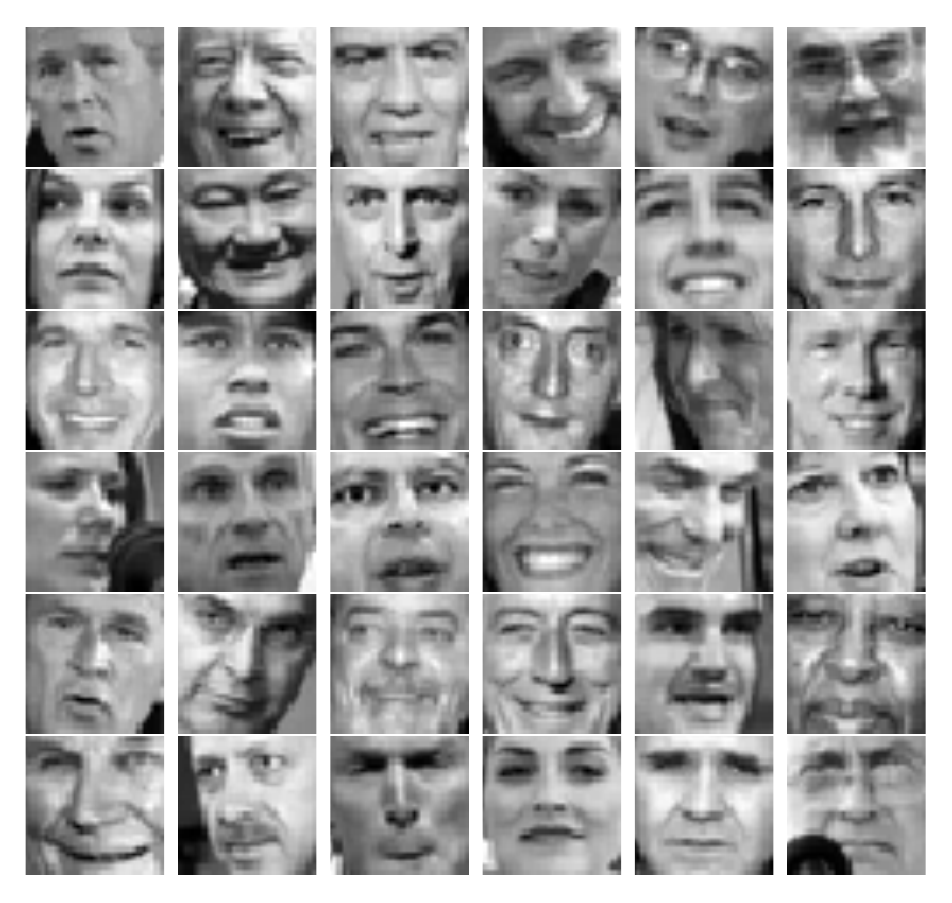}}
	\hspace{01cm}
	\subfloat[Semi-BNP-MMD GAN ]{\includegraphics[width=.22\linewidth]{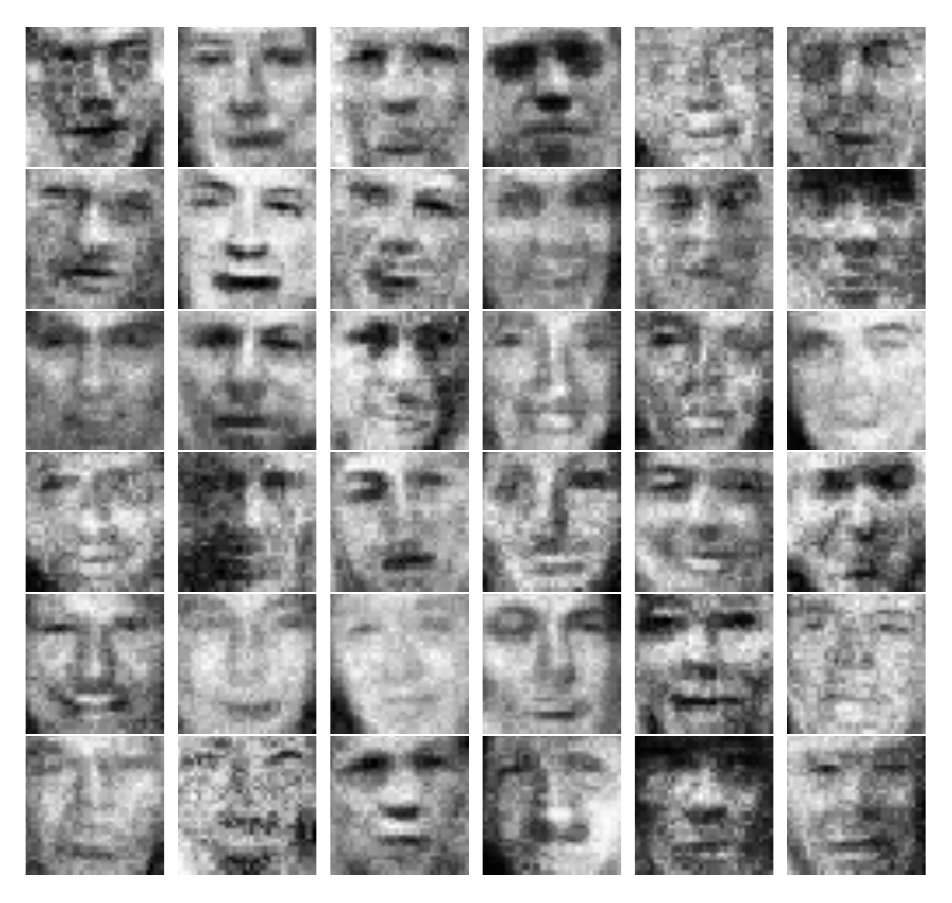}}
	\hspace{1cm}
	\subfloat[MMD-FNP GAN ]{\includegraphics[width=.22\linewidth]{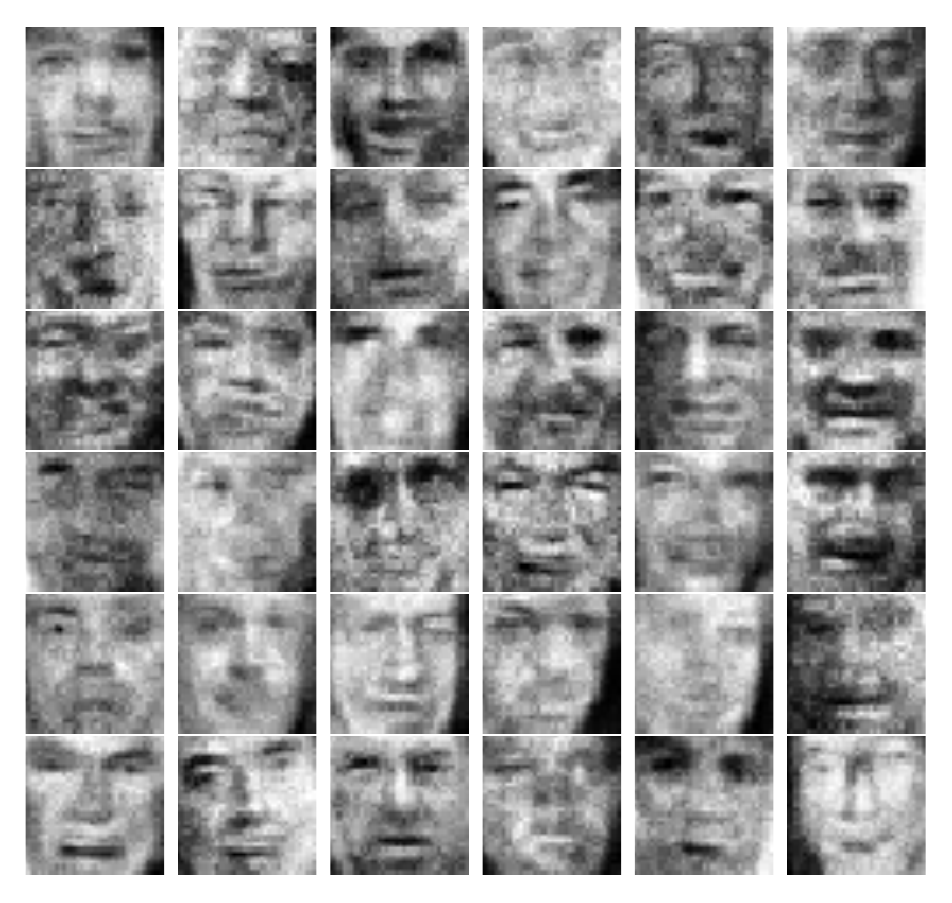}}
	\qquad
	\subfloat[Training data]{\includegraphics[width=.22\linewidth]{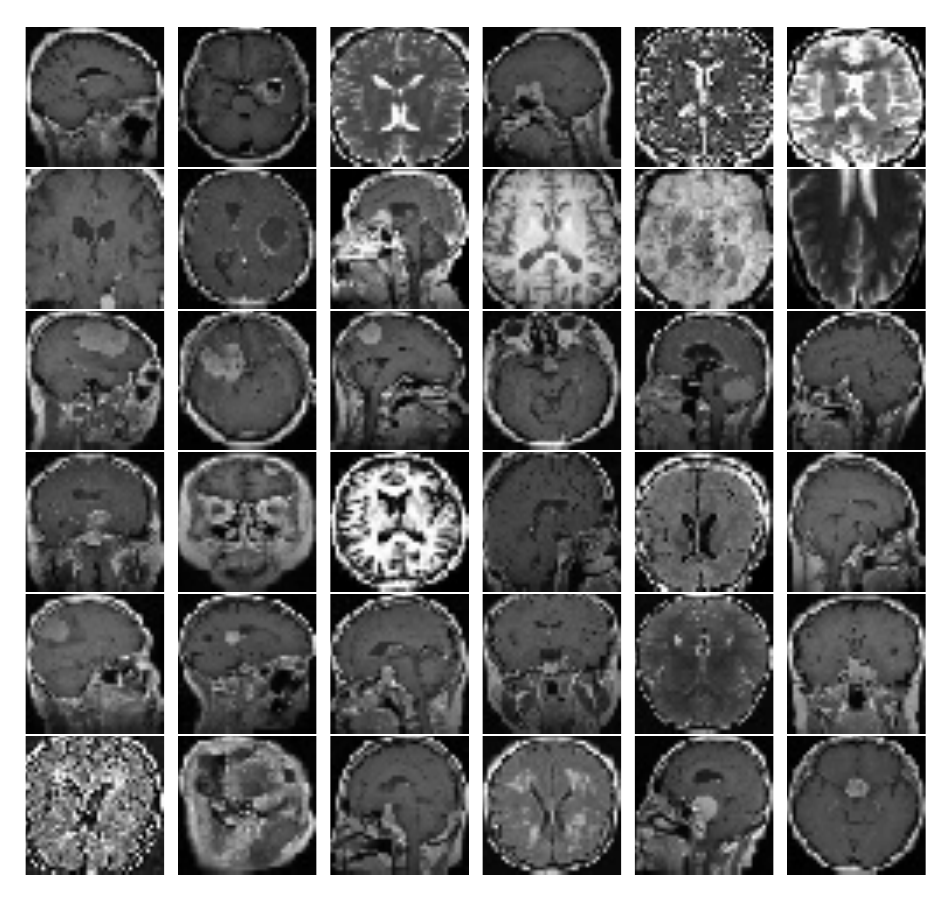}}
	\hspace{01cm}
	\subfloat[Semi-BNP-MMD GAN ]{\includegraphics[width=.22\linewidth]{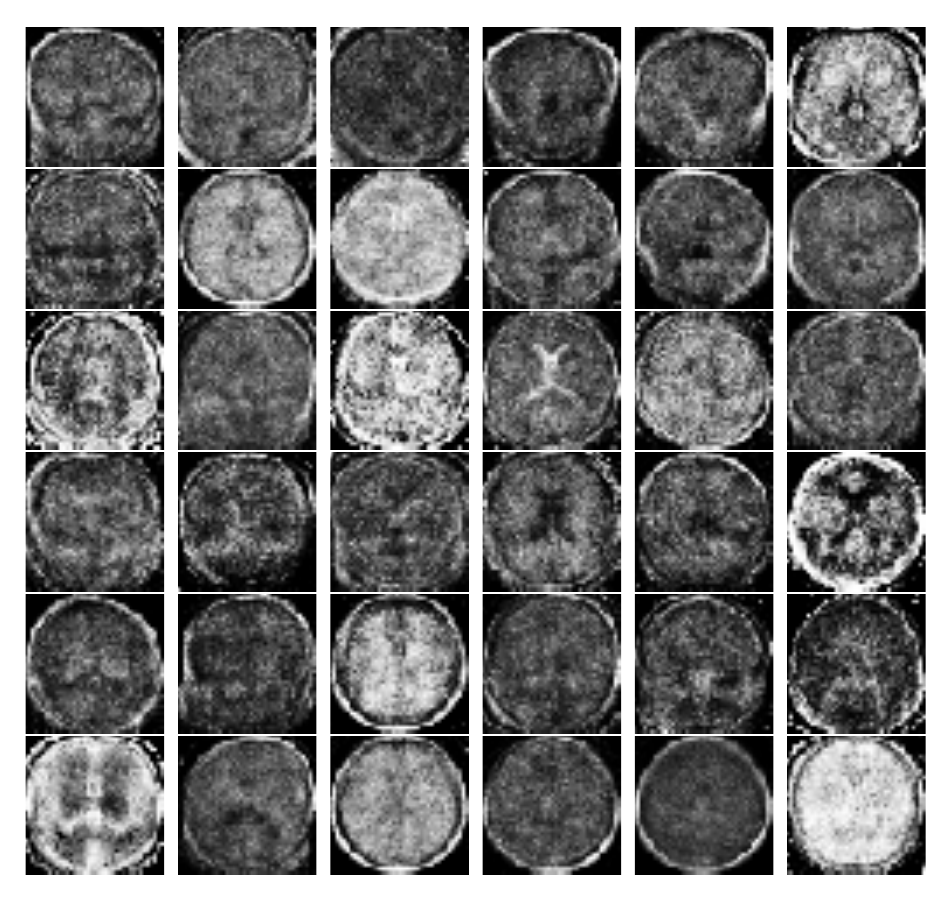}}
	\hspace{1cm}
	\subfloat[MMD-FNP GAN ]{\includegraphics[width=.22\linewidth]{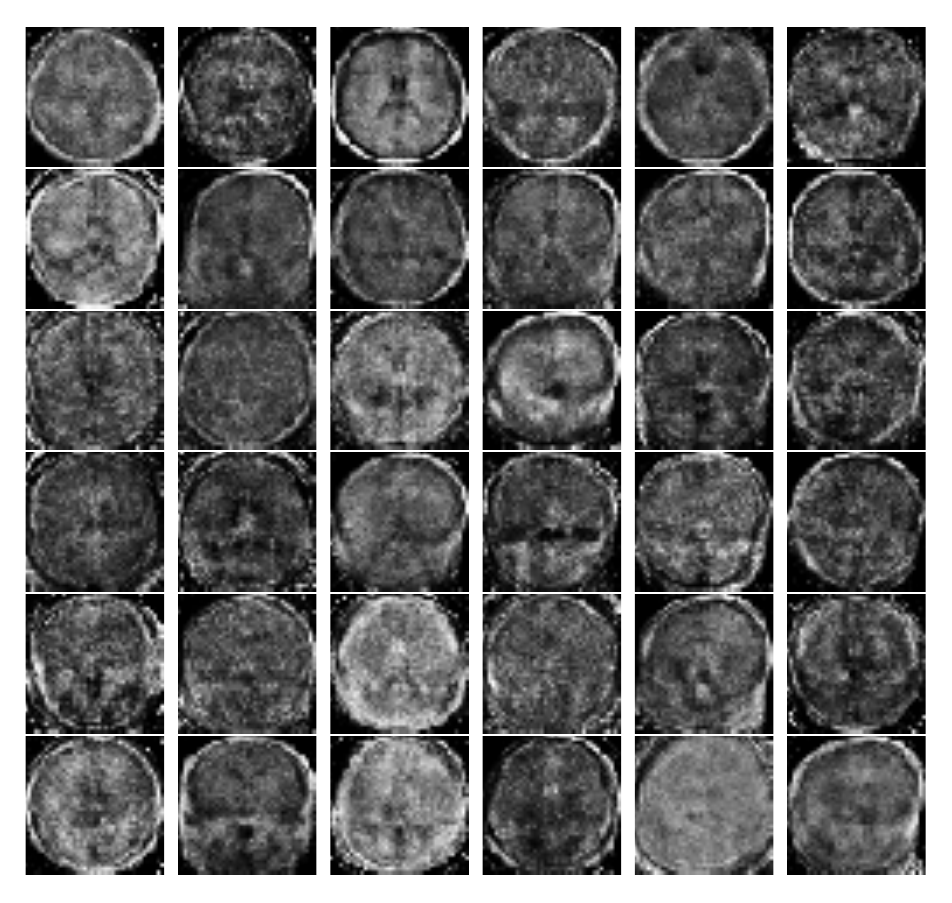}}
	\caption{Generated samples of sizes ($6\times6$) from semi-BNP-MMD and MMD-FNP GAN for the BMB and LFW datasets using a mixture of Gaussian kernels in 40,000 iterations.}\label{extra-dataset}
\end{figure}
\begin{table}[ht] \centering
	\setlength{\tabcolsep}{1.2 mm}
	\caption{The values of MMD, KID, and FID scores for four groups of datasets considering $n_{mb}=1000$ and $r_{mb}=1000$ in \eqref{MMDS}.}\label{MMDS-r}
	\begin{tabular}[c]{lcccccccc}\hline
		\multirow{3}{*}[5pt]{Scores} &  \multicolumn{8}{c}{Dataset} \\\cmidrule(lr){2-9} 
		&\multicolumn{2}{c}{MNIST}&\multicolumn{2}{c}{BMB}&\multicolumn{2}{c}{LFW}&\multicolumn{2}{c}{MRI}\\\cmidrule(lr){2-3}\cmidrule(lr){4-5}\cmidrule(lr){6-7}\cmidrule(lr){8-9}
		&\multicolumn{1}{c}{Semi-BNP}&\multicolumn{1}{c}{FNP}&\multicolumn{1}{c}{Semi-BNP}&\multicolumn{1}{c}{FNP}&\multicolumn{1}{c}{Semi-BNP}&\multicolumn{1}{c}{FNP}&\multicolumn{1}{c}{Semi-BNP}&\multicolumn{1}{c}{FNP}\\\hline
		MMD&$0.0384$& $0.0404$&$0.0285$&$0.0315$ &$0.0281$ &$0.0302$ &$0.2059 $ &$0.2231 $\\
		KID&$0.0034$&$0.0046$&$0.0030$&$0.0036$&$0.0019$&$0.0026$&$0.0260$&$0.0264$\\
		FID&$35.560$&$37.934$&$17.006$&$17.264$&$14.010$&$14.473$&$87.975$&$87.831$\\
		\bottomrule
	\end{tabular}
	
\end{table}%
\section{More Discussion on the Potential Research}
GANs are increasingly used in medical imaging applications which are effective tools for tasks such as medical imaging reconstructions. The synthetic images generated have often been proven to be valuable especially when the original image is noisy or expensive to obtain. 
GANs have also been used for generating images in cross-modality synthesis problems, where we observe magnetic resonance imaging (MRI) for a given patient but want to generate computed tomography (CT) images for that same patient \citep{wolterink2017deep}. This type of generative method for medical imaging can drastically reduce the time and cost of obtaining data if the quality of the synthetic examples is sufficiently high. GANs have also been used in a diagnostic capacity--for example, in detecting brain lesions in images \citep{alex2017generative}. 

Here, the GAN is trained by distinguishing between labeled data of brain images that contain and do not contain lesions. Then, the discriminator of the GAN is used to detect brain lesions on new images. However, GANs are far less commonly used for tasks like diagnosis. According to a survey on medical imaging research in GANs, less than $10\%$ of the top papers surveyed were dedicated towards making diagnoses, whereas the vast majority of papers were dedicated towards generating realistic synthetic examples of medical images for further analysis \citep{yi2019generative}. We believe this is because where the cost of making errors in diagnosis is immediately consequential to people, unlike other AI applications where GANs are largely used.

We plan to extend the current work by mapping the data to a lower dimensional space using an auto-encoder, a dimensionality reduction model helps to reduce the noise in data and tries to optimize the cost function between the real data and fake data in the code space. Then, we will propose a 3D semi-BNP GAN in the code space to improve the ability of the GAN to generate medical datasets. The auto-encoder method should further reduce the chance of mode collapse and the 3D semi-BNP GAN will reduce the blurriness of the generated samples that may be caused by using the auto-encoder. In future work, our model will be able to generate 3D images and, hence, increase the resolution of images, especially for MRI images. We hope that our future work will make an impact in the field of medical imaging.
\section{Notations}
\begin{table}[h]
	\centering
	\begin{tabular}{@{}ll@{}}
		\toprule
		Notation & Definition \\ \midrule
		$N(\cdot,\cdot)$ & Normal distribution \\
		$LN(\cdot,\cdot)$ & Lognormal distribution \\
		$t_3(\cdot,\cdot)$ & $t$-distribution with 3 degrees of freedom \\
		$LG(\cdot,\cdot)$ & Logistic distribution \\
		$B_{d}$ & $d \times d$ matrix with $0.25$ on the main diagonal and $0.2$ off the diagonal \\
		$\mathbf{c}_d$ & $d$-dimensional column vector of $c$'s \\
		$I_d$ & $d \times d$ identical matrix \\ \bottomrule
	\end{tabular}
	\centering
	\begin{tablenotes}
		\tiny
		\item \fontsize{8}{8}\selectfont{ In all distribution notations, the first component represents the mean vector and the second component\\ represents the covariance matrix.}
	\end{tablenotes}
\end{table}

\end{document}